\documentclass{article}
\PassOptionsToPackage{numbers, compress}{natbib}

\usepackage{ragged2e}
 \usepackage[preprint]{neurips_2026}

\usepackage[utf8]{inputenc}
\usepackage[T1]{fontenc}    
\usepackage{hyperref}      
\usepackage{url}           
\usepackage{booktabs}       
\usepackage{amsfonts}      
\usepackage{nicefrac}       
\usepackage{microtype}      
\usepackage{xcolor}

\usepackage{graphicx} 
\usepackage{amsmath}
\usepackage{amssymb}
\usepackage{bm}
\usepackage[ruled,vlined]{algorithm2e}

\NewDocumentCommand{\hhh}{o}{{\mathcal{H}}
    \IfValueT{#1}{^{#1}}
}
\NewDocumentCommand{\xxx}{}{{\mathcal{X}}}

\newcommand{\eqdef}{
    \ensuremath{\stackrel{\mbox{\upshape\tiny def.}}{=}}
}

\NewDocumentCommand{\coR}{o}{{\overline{\rm co}(\mathcal R
    \IfValueT{#1}{^{#1}}
)}}

\usepackage{soul}

\usepackage{xcolor}
\definecolor{aogreen}{rgb}{0.0, 0.5, 0.0}

\newcommand{\pa}{\operatorname{pa}}
\newcommand{\ch}{\operatorname{ch}}
\usepackage{float}
\usepackage{algorithm2e}
\usepackage{graphicx} 
\usepackage{url}
\usepackage{xcolor}

\usepackage{xparse}
\usepackage{natbib}

\usepackage{amsfonts,amsmath,amsthm,amssymb,enumitem}
\newtheorem{theorem}{Theorem}[section]
\newtheorem{lemma}[theorem]{Lemma}
\newtheorem{proposition}[theorem]{Proposition}

\newtheorem{definition}[theorem]{Definition}
\newtheorem{corollary}[theorem]{Corollary}

\usepackage{tikz}
\usetikzlibrary{matrix}

\usepackage{hyperref} 
    \hypersetup{
    	colorlinks = true,
    	linkcolor = blue,
    	anchorcolor = blue,
    	citecolor = blue,
    	filecolor = blue,
    	urlcolor = blue
    }

\usepackage{comment}

\usepackage{marginnote}
\usepackage{xcolor}
\definecolor{darkcyan}{rgb}{0.0, 0.55, 0.55}
\definecolor{MidnightBlue}{RGB}{25,25,112}
\definecolor{MidnightBlueComplementingGreen}{RGB}{25,112,25}
\definecolor{MidnightBlueComplementingPurple}{RGB}{112,25,112}
\definecolor{MidnightBlueComplementingRed}{RGB}{112,25,69}
\definecolor{WowColor}{rgb}{.75,0,.75}
\definecolor{MildlyAlarming}{rgb}{0.85,0.25,0.1}
\definecolor{SubtleColor}{rgb}{0,0,.50}
\definecolor{antiquefuchsia}{rgb}{0.57, 0.36, 0.51}
\definecolor{fashionfuchsia}{rgb}{0.96, 0.0, 0.63}
\definecolor{jade}{rgb}{0.0, 0.66, 0.42}
\definecolor{caribbeangreen}{rgb}{0.0, 0.8, 0.6}
\definecolor{aquamarine}{rgb}{0.5, 0.8, 0.85}
\definecolor{attentioncolor}{RGB}{152,90,81}
\definecolor{burgred}{RGB}{40,3,22}
\definecolor{AKGreen}{RGB}{17,123,92}
\definecolor{Turquoise}{RGB}{64,224,208}

\definecolor{darkjade}{RGB}{0,122,84}
\definecolor{Window1}{RGB}{92,150,31}%
    \definecolor{Window1dark}{RGB}{41,67,13}%
\definecolor{Window2}{RGB}{255,168,28}
    \definecolor{Window2dark}{RGB}{114,75,12}
\definecolor{Window3}{RGB}{255,96,33}
    \definecolor{Window3dark}{RGB}{97,36,12}
\definecolor{InputColor}{RGB}{20,255,177}
    \definecolor{InputColorlight}{RGB}{222,237,229}

\usepackage[colorinlistoftodos]{todonotes}

\NewDocumentCommand{\Noah}{mo}{
    \IfValueF{#2}{
                        {{
                            \textcolor{magenta}{ 
                            \textbf{N:}
                            \textit{{#1}}
                            }
                        }}
        }
    \IfValueT{#2}{
                        \marginnote{{\scriptsize
                            \textcolor{magenta}{ 
                            \textbf{N:}
                            \textit{{#1}}
                            }
                        }}
        }
                    }

\definecolor{blue_nice}{rgb}{0.0, 0.5, 0.69}
\NewDocumentCommand{\AK}{mo}{
    \IfValueF{#2}{
                        {{
                            \textcolor{blue_nice}{ 
                            \textbf{Annie:}
                            \textit{{#1}}
                            }
                        }}
        }
    \IfValueT{#2}{
                        \marginnote{{\scriptsize
                            \textcolor{blue_nice}{ 
                            \textbf{Annie:}
                            \textit{{#1}}
                            }
                        }}
        }
                    }
\NewDocumentCommand{\Gudi}{mo}{
    \IfValueF{#2}{
                        {{
                            \textcolor{caribbeangreen}{ 
                            \textbf{G:}
                            \textit{{#1}}
                            }
                        }}
        }
    \IfValueT{#2}{
                        \marginnote{{\scriptsize
                            \textcolor{caribbeangreen}{ 
                            \textbf{G:}
                            \textit{{#1}}
                            }
                        }}
        }
                    }

\NewDocumentCommand{\bea}{mo}{
    \IfValueF{#2}{
                        {{
                            \textcolor{magenta}{ 
                            \textbf{B:}
                            \textit{{#1}}
                            }
                        }}
        }

    \IfValueT{#2}{
                        \marginnote{{\scriptsize
                            \textcolor{magenta}{ 
                            \textbf{B:}
                            \textit{{#1}}
                            }
                        }}
        }
                    }
\usepackage{tikz}

\usepackage{subcaption}
\usepackage{placeins}
\usepackage{algpseudocode}
\title{Classification Fields: Arbitrarily Fine Recursive Hierarchical Clustering From Few Examples}

\author{
\makebox[\textwidth][c]{%
\begin{tabular}{c}
\textbf{Yicen Li}$^{1,2}$\thanks{Corresponding author. Email: \texttt{li2642@mcmaster.ca}},
\textbf{Ruiyang Hong}$^{1,2}$,
\textbf{Anastasis Kratsios}$^{1,2}$\thanks{Corresponding author. Email: \texttt{kratsioa@mcmaster.ca}}
\\
\textbf{Haitz Sáez de Ocáriz Borde}$^{3}$,
\textbf{Paul D. McNicholas}$^{1,2}$ \\[0.5em]
{\normalfont $^{1}$Department of Mathematics and Statistics, McMaster University, Canada} \\
{\normalfont $^{2}$Vector Institute, Canada} \\
{\normalfont $^{3}$University of Cambridge, United Kingdom}
\end{tabular}%
}
}

\usepackage[most]{tcolorbox}

\newtcolorbox{questionbox}{
  enhanced,
  hbox,
  colback=gray!6,
  colframe=gray!50,
  boxrule=0.4pt,
  arc=1.5mm,
  left=3mm,
  right=3mm,
  top=1.2mm,
  bottom=1.2mm,
  before=\begin{center},
  after=\end{center}
}

\begin{document}

\maketitle

\vspace{-15pt}
\begin{abstract}
Classical clustering methods usually return either a finite partition of the observed data or a finite dendrogram over it. This finite-sample view is inadequate when the hierarchy of interest is a recursive geometric object with fine-scale refinements that continue beyond the levels directly observed. We introduce classification fields: infinite-depth hierarchical cluster structures on $\mathbb{R}^d$ generated by a local parent-to-child refinement rule. A classification field generator maps each parent centre to an ordered, bounded, and separated tuple of child residuals. Together with a root and a scale factor, this rule recursively generates cluster centres, Voronoi cells, and a metric DAG encoding the hierarchy. Given only a finite prefix of such a hierarchy, we learn a classification field predictor that approximates the generator and can be rolled out to unseen depths. 
We prove exponential truncation convergence in the completed cell metric and
ReLU realizability with width \(O(\varepsilon^{-\gamma})\) and depth
\(\widetilde O(\varepsilon^{-3\gamma/2})\), where
\(\gamma=\log K/(-\log s)\), up to finite-window aspect-ratio factors.
The approximation holds at the level of the induced compact metric structures, measured in the completed cell-metric Hausdorff distance.
Experimental validation on matched CFG-generated hierarchies, IFS fractals, and image-induced recursive clustering hierarchies shows that learned predictors preserve ordered child slots, unordered geometry, and hierarchy-level path metrics under recursive rollout. These results support the claim that finite hierarchical observations can reveal local refinement rules capable of generating substantially deeper classification fields.

\end{abstract}
\vspace{-10pt}
\section{Introduction}
Most hierarchical clustering methods operate on a fixed finite sample: given $N$ observations, they return a partition of those observations into at most $k \leq N$ groups, or a finite dendrogram whose leaves are the observed data points~\cite{manghiuc2021hierarchical,laenen2023nearly,braverman2025learningaugmented}. This is appropriate when the sample itself is the object to be clustered. We instead study settings where the observed hierarchy is only a finite prefix of a richer recursive structure. The number of latent fine-scale cluster centers may grow far beyond the number of observed nodes, i.e., $k \gg N$, and refinements may continue at depths not present in the training hierarchy. In this regime, the target is not merely a finite decomposition of the available data, but a multiscale geometric object generated by a reusable local refinement rule. This leads to a central question: if we observe only finite levels of a hierarchy:

\begin{questionbox}
\parbox{0.92\linewidth}{\itshape\justifying
Can we learn a local refinement rule from a finite observed prefix and recursively roll it out to generate substantially deeper, unobserved levels of the hierarchy?
}
\end{questionbox}

We answer this question affirmatively by introducing classification fields: infinite-depth hierarchical cluster structures generated by local parent-to-child refinement rules. Unlike classical hierarchical clustering, which describes a fixed collection of observations, classification fields model clusters themselves as recursively refinable geometric objects. Even recent differentiable or learned approaches to hierarchy typically fit a finite observed tree, embed a fixed structure, or optimize a finite clustering objective~\cite{stewart2023differentiable,manduchi2023tree,wang2026deep}; they do not directly learn a refinement mechanism that can be recursively rolled out beyond the observed hierarchy.

We formalize classification fields as recursively generated metric directed acyclic graphs (DAG). The local rule is a classification field generator (CFG), which maps each parent centre to an ordered $K$-tuple of child residuals. Starting from a root and a scale factor, recursive application of the CFG generates cluster centres at all depths; the associated Voronoi cells give the cluster-level geometric realization, while genealogical addresses encode parent-child structure. This separation between genealogy and geometry induces tree-based, point-based, and cell-based distances, allowing finite truncations and completed infinite-depth fields to be compared within a common metric-geometric framework \cite{borde2023neural}. From this viewpoint, hierarchical clustering becomes a problem of learning a local recursive refinement law, rather than constructing a finite combinatorial summary of a fixed dataset.

This framework also clarifies the learning problem. Given only a finite prefix of a hierarchy, we learn a classification field predictor (CFP) that approximates the underlying CFG and can be applied recursively to generate deeper levels. Under boundedness, separation, and regularity assumptions on the generator, we show that sufficiently deep finite truncations approximate the completed classification field under the completed cell metric. Building on quantitative memorization and approximation results for ReLU networks \cite{kratsios2023small, kim2023minimum,yu2024generalizablity}, we further show that ReLU predictors can realize the finite refinement windows needed for such approximations, yielding a finite neural rollout mechanism for infinite-depth hierarchical completion. Empirically, we evaluate this viewpoint in three settings of increasing mismatch from the theory: matched CFG-generated hierarchies, out-of-family IFS fractal hierarchies, and image-induced recursive clustering hierarchies built from CLIP~\cite{radford2021learningtransferablevisualmodels} embeddings of CIFAR datasets \cite{krizhevsky2009learning}. Across these settings, CFP rollouts preserve ordered child slots, unordered geometry, and hierarchy-level path metrics better than simpler baselines.

Overall, this paper makes three contributions. First, we introduce classification fields as a metric-geometric model for infinite-depth hierarchical clustering: recursively generated DAGs whose nodes carry genealogical addresses, geometric centres, and cluster cells. Second, we develop a metric theory for these objects, including tree, point, and cell based distances and quantitative truncation to completion bounds. Third, we establish neural realizability of recursive refinement by showing that finite refinement windows can be represented by ReLU predictors and rolled out as finite approximations to the limiting field. The experiments validate this classification-field viewpoint on matched recursive hierarchies, fractal hierarchies, and approximate image-induced hierarchies.

\section{Related Work}

\paragraph{Hierarchical clustering.}
Classical hierarchical clustering organizes a fixed finite dataset into a
hierarchy, from agglomerative procedures such as Ward's method
\cite{ward1963hierarchical,johnson1967hierarchical} to modern objective-based
formulations including Dasgupta's cost and related approximation guarantees
\cite{dasgupta2016cost,moseley2023approximation,cohen2019hierarchical,
roy2017hierarchical,manghiuc2021hierarchical,laenen2023nearly,
braverman2025learningaugmented}. Recent work further develops continuous and
differentiable relaxations that make hierarchical clustering amenable to
gradient-based optimization \cite{wang2020objective,chami2020trees,kollovieh2024expected,
stewart2023differentiable,manduchi2023tree}. 
A complementary population-level line studies cluster trees and density
level-set estimation, where the target hierarchy is induced by the connected
components of density upper level sets rather than by a finite dendrogram
\cite{hartigan1981consistency,chaudhuri2010rates,
chaudhuri2014consistent,stuetzle2010generalized,JMLR:v26:24-1052}.
Classification fields address a different object: a hierarchy generated by a
learnable local parent-to-child refinement rule, with recursive rollout of
centres, cells, and path metrics beyond the observed levels.

\paragraph{Deep clustering and learned structured representations.}
Deep clustering methods learn feature representations and cluster assignments, and improve clustering by adapting the representation space to the clustering objective \cite{caron2018deep,xie2016unsupervised,van2020scan,zhou2024comprehensive}. Another related line learns geometric representations of structured objects such as trees and DAGs, enabling discrete hierarchies to be embedded, optimized, or predicted with neural models \cite{ganea2018hyperbolic,borde2024neural,manduchi2023tree,wang2026deep,lin2025joint,borde2023neural}. These works show that neural representations can capture clustering and hierarchical structure, but they typically operate on a fixed dataset, tree, or graph. In contrast, a classification field predictor learns a local refinement map from a finite hierarchy prefix and uses recursive rollout to define deeper truncations of an infinite-depth metric hierarchy. This also differs from Matryoshka representation learning \cite{kusupati2022matryoshka}, which nests information across embedding dimensions or model capacity rather than across recursive hierarchy depth.

\paragraph{Recursive geometric generation.}
Classification fields are also related to recursive geometric generation, especially iterated function systems and self-similar fractals \cite{hutchinson1981fractals,barnsley1986solution}. IFS theory studies how repeated contraction maps generate self-similar attractors, and our fractal experiments use this setting as an out-of-family test of recursive refinement learning \cite{poli2022self}. The target, however, is different: IFS methods usually focus on the limiting set, whereas classification fields retain the explicit level-wise hierarchy, parent-child relations, and induced path metrics. Our theory additionally uses metric comparison tools for finite and infinite hierarchical objects \cite{mémoli2021gromovhausdorffdistanceultrametricspaces,borde2023neural,lin2025joint} together with quantitative ReLU memorization results \cite{kim2023minimum,yu2024generalizablity}. These tools let us formalize finite-depth observations as approximations to a completed infinite-depth classification field.
\section{Infinite Depth Hierarchies}
In this section we formalize classification fields. 
A classification field is the infinite-depth hierarchical object obtained by recursively applying a local parent-to-child refinement rule.
The rule itself is called a classification field generator (CFG).
Given a parent centre, it returns an ordered $K$-tuple of child residuals. Together with a root and a scale factor, repeated application of the CFG generates centres, Voronoi cells, and a genealogical graph. 
Thus, the classification field is the full recursively generated hierarchy, while the CFG is the local rule that generates it.

\begin{figure}[htbp]
    \centering
    \includegraphics[width=\linewidth]{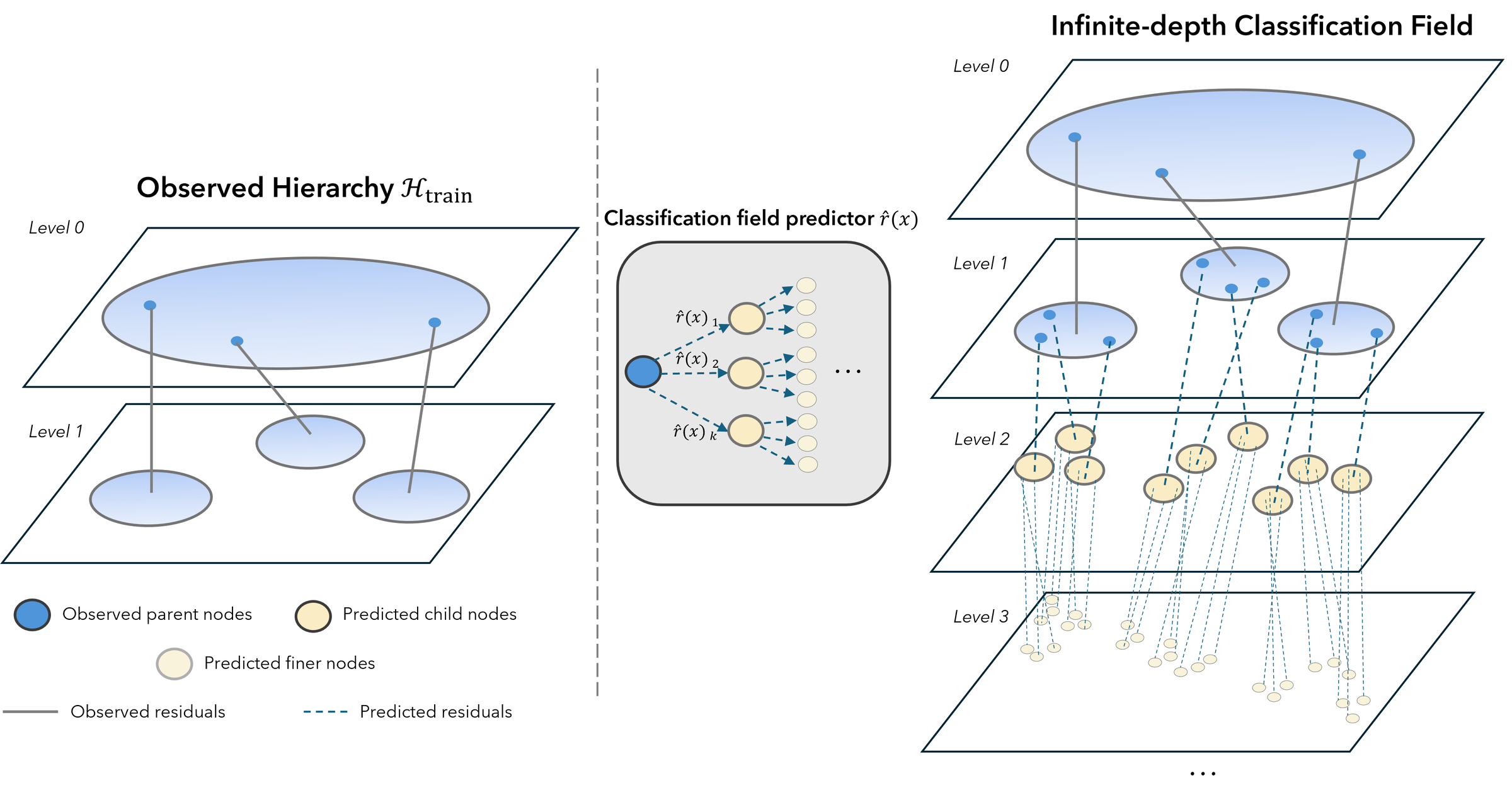}
    \caption{Learning a classification field from finite observations. The observed hierarchy $H_{train}$ contains only the first few refinement levels. A classification field predictor $\hat r$ maps each parent centre to an ordered tuple of child residuals, approximating the local classification field generator. Recursive rollout of $\hat r$ generates deeper unseen levels and yields a finite approximation to the target infinite-depth classification field.}
\label{fig:main}
\label{fig:main}
    \label{fig:placeholder}
\end{figure}

\subsection{Hierarchical Generators}
\label{s:hierarchicalgenerators}
We call a function
$
r:\mathbb{R}^d\to (\mathbb{R}^d)^K
$
a \textit{classification field generator} (CFG) if:
there exist $0<\lambda_{\min}\le \lambda_{\max}\le 1$ and $0<s_{\operatorname{sep}}$ such that: for every $x\in \mathbb{B}_2^d$ and each $k,\tilde{k}\in [K]_+$ with $k\neq \tilde{k}$

\begin{align}
\label{eq:separation_property}
s_{\operatorname{sep}} & \le \|r_k(x) - r_{\tilde{k}}(x)\|_2
,
\\
\lambda_{\min} & \le \|r_k(x)\|_2 \le \lambda_{\max}
.
\label{eq:boundedness_property}
\end{align}
In other words, for each input \(x\in\mathbb B_2^d\), the vectors
\(\{r_k(x)\}_{k=1}^K\) lie in a Euclidean annulus with inner radius
\(\lambda_{\min}\) and outer radius \(\lambda_{\max}\), and form an
\(s_{\operatorname{sep}}\)-separated configuration. This configuration is
not required to be maximal or orthogonal. We note that, there may be exponentially many such points in the sphere, since we are not requiring orthogonality in any manner.
We also note that: we are not requiring any continuity of $r$.

A natural question is how to construct such maps, and the following proposition answers this question.  We use the following idea of separated set in the above annulus; which we call a ``reference packing''. 
\begin{definition}[Reference packing]
\label{def:reference_packing}
Fix $d,K\in\mathbb{N}_+$ and parameters
\[
0<\lambda_-\le \lambda_+<\infty,
\qquad
\varepsilon>0.
\]
A $d\times K$ matrix $
C=(c_1|\cdots|c_K)$ is called a  \emph{$(\lambda_-,\lambda_+,\varepsilon)$-reference packing} if
\begin{align}
\lambda_-
\le
\|c_k\|_2 
\le
\lambda_+
\qquad
&\forall k\in[K],
\\
\|c_i-c_j\|_2
\ge
\varepsilon
\qquad
&\forall i,j\in[K]
\text{ with }i\neq j.
\end{align}
\end{definition}
Given a reference packing,  which as shown in Appendix~\ref{a:generatingpacking} Proposition~\ref{prop:haar_samples_separated}, can be randomly generated rather efficiently, we are able to construct a large family of CFGs.
\begin{proposition}[Parametrization of CFGs]
\label{prop:param_uniformly_separated_annular_configs}
Fix $d,K\in\mathbb{N}_+$ and
$
0<\sigma_{\min}\le \sigma_{\max}<\infty
$, 
$0<\lambda_-\le \lambda_+<\infty$, $
\varepsilon>0$, 
and assume that there is a $(\lambda_-,\lambda_+,\varepsilon)$-reference packing
$
C=(c_1|\cdots|c_K)\in \mathbb{R}^{d\times K}
$.
\hfill\\
For any set of maps
$
A_1,A_2:\mathbb{B}_2^d\to \mathbb{R}^{d\times d}
$
,
$
x\mapsto A_1(x),A_2(x)$, 
with $A_1(x)$ and $A_2(x)$ skew-symmetric for every $x$, and any function
$
\sigma:\mathbb{B}_2^d\to [\sigma_{\min},\sigma_{\max}]^d
$
induce a CFG 
\begin{equation}
\label{eq:r_inducedoski}
\begin{aligned}
    r:\mathbb{B}_2^d & \rightarrow \mathbb{R}^{d\times K}
\\
r(x)& \mapsto e^{A_1(x)}\operatorname{diag}(\sigma(x))e^{A_2(x)}{}^\top \,
C
.
\end{aligned}
\end{equation}
Moreover, 
$s_{\operatorname{sep}} = \sigma_{\min}\varepsilon$, $
\lambda_{\min} = \sigma_{\min}\lambda_-$, and $
\lambda_{\max} = \sigma_{\max}\lambda_+$ in~\eqref{eq:separation_property} and~\eqref{eq:boundedness_property}.
\end{proposition}
\begin{proof}
See Appendix~\ref{s:orth_diag_param_sep_annuli}.
\end{proof}
Under more regularity assumptions on the functions inducing~\eqref{eq:r_inducedoski} we obtain regularity and separation properties of $r$ requiring basic regularity of the functions inducing it. More additional results are provided in Appendix~\ref{s:good_properties}.

\subsection{The Graphical Model}
\label{s:Graphical_Model}
For any radius $r\ge 0$ and any centre $x\in \mathbb{R}^d$ we define the \textit{closed} $\ell^p$-ball of dimension $d$ to be 
$
    \mathbb{B}_p^d(x,r) \eqdef \{z\in \mathbb{R}^d:\|z-x\|_p\le r\}
$.
We call a CFG-root-scale triple admissible if all recursively generated centres lie in the domain on which the CFG bounds hold and all recursively defined cells \(C_\alpha\) are non-empty compact sets.

\paragraph{Vertices}
We fix $1\le p\le \infty$, $K\in \mathbb{N}_+$, $0<\lambda_{\min}\le \lambda_{\max}\le 1$. All CFGs will satisfy~\eqref{eq:separation_property} and~\eqref{eq:boundedness_property}.
We also fix a scale $0<s<1$.
To distinguish the combinatorial genealogy from its geometric realization, we index vertices by their address in the tree. For every level $l\in \mathbb{N}$, we write
$
    \mathcal{A}_l \eqdef [K]_+^l
$ and $
    \mathcal{A}_0 \eqdef \{\varnothing\}
$,
where $\varnothing$ denotes the empty word. 
For \(\alpha=(\alpha_1,\ldots,\alpha_m)\in\mathcal A_m\) and \(k\in[K]_+\), we write \(\alpha k\eqdef(\alpha_1,\ldots,\alpha_m,k)\in\mathcal A_{m+1}\).
Given an initial condition $x_0\in \mathbb{R}^d$, we define the root by
\[
    x_{\varnothing}\eqdef x_0,
    \qquad
    C_{\varnothing}\eqdef \mathbb{B}_p^d(x_0,\lambda_{\max}).
\]
For every subsequent level $l\in \mathbb{N}_+$, every address $\alpha\in \mathcal{A}_{l-1}$, and every $k\in [K]_+$, we define
\begin{equation}
\label{eq:hierarchies}
\begin{aligned}
    x_{\alpha k}
    &\eqdef
    x_{\alpha}+s^l\,r(x_{\alpha})_k,
\\
    C_{\alpha k}
    &\eqdef
    \big\{
        z\in C_{\alpha}
        \,:\,
        \|z-x_{\alpha k}\|_p
        \le
        \min_{\tilde{k}\in [K]_+}\|z-x_{\alpha \tilde{k}}\|_p
    \big\}.
\end{aligned}
\end{equation} 

The level-$l$ vertex set of the genealogical tree is then
$
    \mathcal{V}_l^{\operatorname{tree}}
    \eqdef
    \big\{
        (C_{\alpha},x_{\alpha},\alpha)
        \,:\,
        \alpha\in \mathcal{A}_l
    \big\},
$
and the full genealogical vertex set is
$
    \mathcal{V}^{\operatorname{tree}}
\eqdef
    \bigcup_{l=0}^{\infty}\mathcal{V}_l^{\operatorname{tree}}
$.
In other words, the $l^{th}$ level of the hierarchy consists of the Voronoi cells induced by a fixed perturbation rule (encoded in the CFG $r$), whereby $K$ residuals are added to the previous parent node to construct its $K$ children; this is performed once for each parent at the $(l-1)^{st}$ level to obtain the hierarchy.
If one wishes to identify geometrically coincident nodes, one may pass to the quotient defined by
\[
    (C_{\alpha},x_{\alpha},\alpha)\sim (C_{\beta},x_{\beta},\beta)
    \mbox{ if and only if }
    C_{\alpha}=C_{\beta}
    \ \mbox{ and }\
    x_{\alpha}=x_{\beta}.
\]
The corresponding geometric vertex set is then
$
    \mathcal{V}^{\operatorname{geo}}
    \eqdef
    \mathcal{V}^{\operatorname{tree}}/{\sim}
$. Figure~\ref{fig:lecluster} in the appendix provides an intuitive understanding.
\paragraph{Edges}
The directed edges encode the parent-to-child relation between consecutive levels.
For the genealogical tree, we define
\[
    E^{\operatorname{tree}}
    \eqdef
    \Big\{
        \big(
            (C_{\alpha},x_{\alpha},\alpha),
            (C_{\alpha k},x_{\alpha k},\alpha k)
        \big)
        \in
        \mathcal{V}^{\operatorname{tree}}\times \mathcal{V}^{\operatorname{tree}}
        \,:\,
        \alpha\in \mathcal{A}_{l-1},\ 
        l\in \mathbb{N}_+,\ 
        k\in [K]_+
    \Big\}
    .
\]
Equivalently, there is an edge from every parent vertex to its $K$ children, and only to its $K$ children.

If one identifies geometrically coincident vertices, then the quotient edge set is
\[
    E^{\operatorname{geo}}
    \eqdef
    \Big\{
        ([u],[v])
        \in
        \mathcal{V}^{\operatorname{geo}}\times \mathcal{V}^{\operatorname{geo}}
        \,:\,
        \exists\,u'\in [u],\ \exists\,v'\in [v]
        \mbox{ such that }
        (u',v')\in E^{\operatorname{tree}}
    \Big\}
    .
\]

By construction, every directed edge joins a level-$(l-1)$ vertex to a level-$l$ vertex for some $l\in \mathbb{N}_+$. In particular, levels strictly increase along directed edges. Hence, the genealogical graph is a rooted directed tree. After quotienting geometrically coincident vertices, we work with the induced weighted graph metric. When the quotient does not identify vertices across levels, the quotient remains a DAG.

\paragraph{Edge-Weights and Shortest-Path Metrics}
We consider natural edge-weighting schemes on the genealogical directed graph
$
    G^{\operatorname{tree}}
    \eqdef
    (\mathcal{V}^{\operatorname{tree}},E^{\operatorname{tree}})
$.
In each case, we pass to the underlying undirected graph
$
    \overline{G}^{\operatorname{tree}}
    \eqdef
    (\mathcal{V}^{\operatorname{tree}},\overline{E}^{\operatorname{tree}})
$
obtained by forgetting edge orientations, and endow it with the corresponding shortest-path metric.

For notational convenience, if
$
u=(C_{\alpha},x_{\alpha},\alpha),\,
v=(C_{\beta},x_{\beta},\beta)
\in \mathcal{V}^{\operatorname{tree}},
$
then we write
$
\{u,v\}\in \overline{E}^{\operatorname{tree}}
$
whenever either $(u,v)\in E^{\operatorname{tree}}$ or $(v,u)\in E^{\operatorname{tree}}$.

\paragraph{Scale-Sensitive Cell-based edge weights.}
This weighting scheme measures distances between the Voronoi cells themselves, rather than between their distinguished points. To this end, for non-empty sets $A,B\subseteq \mathbb{R}^d$, we write
\[
    d_{H,p}(A,B)
    \eqdef
    \max\Big\{
        \sup_{a\in A}\inf_{b\in B}\|a-b\|_p,
        \sup_{b\in B}\inf_{a\in A}\|a-b\|_p
    \Big\}
\]
for the Hausdorff distance induced by $\|\cdot\|_p$. We then define the \textit{scale-sensitive} edge-weights by
\begin{equation}
\label{eq:scaled__cluster_weights}
    W_{\operatorname{cell}}
    \big(
        (C_{\alpha},x_{\alpha},\alpha),
        (C_{\alpha k},x_{\alpha k},\alpha k)
    \big)
    \eqdef
    \underbrace{s^{|\alpha|}\lambda_{\max}}_{\text{scale sensitivity}}
    \,
    \underbrace{
        d_{H,p}\big(C_{\alpha},C_{\alpha k}\big)
    }_{\text{cluster comparison}}
.
\end{equation}
We write $G_{\operatorname{cell}}\eqdef (\mathcal{V}^{\operatorname{geo}},d_{\operatorname{cell}})$ for the shortest-path metric on the weighted graph $(\mathcal{V}^{\operatorname{geo}},E^{\operatorname{geo}},W_{\operatorname{cell}})$ with ``scale-sensitive'' edge-weights in~\eqref{eq:scaled__cluster_weights}. 
Two additional weighting schemes and their corresponding properties are provided in Appendix~\ref{s:add_back} and Appendix~\ref{s:add_results}.

\section{Theoretical Guarantees}
\label{s:Main__ss:Theory}

The construction above defines an infinite object, whereas both observation and computation are finite. The following result connects these two regimes. It shows that a sufficiently deep finite rollout approximates the completed classification field, and that the finite refinement window needed for this rollout can be realized by a ReLU classification field predictor. 

\begin{theorem}[Exponentially-Fast Convergence to the Infinite-Precision Classification Field]
\label{thrm:main_result}
Fix \(1\le p<\infty\), \(d\in\mathbb N_+\), \(K>1\), \(0<s<1\), and
\(0<\lambda_{\min}\le \lambda_{\max}\le 1\). 
Let \(x_0\in\mathbb R^d\) and let
\(r:\mathbb R^d\to\mathbb R^{d\times K}\) be a CFG such that the induced
CFG-root-scale triple \((r,x_0,s)\) is admissible. Denote by
\(X_{\operatorname{cell}}\) the completed cell-metric classification field
generated by this triple, and write
\(Y_L\eqdef \mathcal V^{\operatorname{geo}}_{\le L}\subset X_{\operatorname{cell}}\)
for its true depth-\(L\) geometric truncation.

For any \(0<\varepsilon<1\), set 
$L_\varepsilon
\eqdef
\max\left\{
2,
\left\lceil
\frac{
\log\!\left(2\lambda_{\max}^2/((1-s)\varepsilon)\right)
}{
-\log s
}
\right\rceil
\right\}.$

Then there exists a \(\operatorname{ReLU}\)-MLP, called a classification field
predictor (CFP), \(\hat r:\mathbb R^d\to\mathbb R^{d\times K}\), whose
depth-\(L_\varepsilon\) rollout coincides with the true truncation
\(Y_{L_\varepsilon}\). Consequently,
\[
d_H^{\overline d_{\operatorname{cell}}}
\bigl(
Y_{L_\varepsilon},X_{\operatorname{cell}}
\bigr)
\le
\varepsilon .
\]
Moreover, writing
\(\gamma\eqdef \log K/(-\log s)>0\) and
\(B_\varepsilon\eqdef B_{0,L_\varepsilon}\) as in
Proposition~\ref{prop:finite_depth_memorization_rhg}, the CFP can be chosen
with width \(\mathcal O(\varepsilon^{-\gamma})\) and depth $\widetilde{\mathcal O}\!\left(
\varepsilon^{-3\gamma/2}
\left[
1+
\frac{\log_+(B_\varepsilon)}{\log(1/\varepsilon)}
\right]
\right).$ 
In particular, if the finite-window aspect ratios are polynomially bounded,
i.e. \(B_\varepsilon\le \varepsilon^{-\mathcal O(1)}\), then
\(\operatorname{depth}(\hat r)
=\widetilde{\mathcal O}(\varepsilon^{-3\gamma/2})\).
\end{theorem}

The proof is given in Appendix~\ref{s:Guarantees__ss:Gromov_Hausdorff__Tree}.
The theorem is a truncation and neural-realizability guarantee: a sufficiently
deep finite rollout can be identified with an \(\varepsilon\)-accurate truncation of the completed classification field.

\section{Experimental Validation}
\label{s:Experiments}
The experiments test the operational meaning of classification fields: whether a finite observed hierarchy contains enough information to learn a local refinement rule whose recursive rollout preserves deeper geometric and hierarchical structure. 
They are not intended as a benchmark competition for finite hierarchical clustering. Instead, each experiment explores whether a learned CFP can act as a finite computational realization of an underlying, or approximate, classification field.

We consider three settings of increasing mismatch from the idealized theory. 
The first is a matched CFG setting, where the data are generated by a known class of recursive refinement rules and the goal is to test stable rollout beyond the observed hierarchy. 
The second uses IFS-induced fractal hierarchies, where the hierarchy has exact recursive structure but is not drawn from the same parametrized CFG family. 
The third uses image-induced recursive clustering hierarchies built from CLIP embeddings of CIFAR images, where the refinement rule is only approximate and child-slot correspondence must be canonicalized. 
Together, these settings test whether the classification field viewpoint remains useful as we move from exact model generated fields to approximate data-induced fields.
Additional diagnostics on observed depth, scale factor, and child-slot consistency are also reported in Appendix~\ref{A:ablation}.

\paragraph{Experimental setup}
Let
$
\mathcal{H}_{\mathrm{train}}
=
\cup_{l=0}^{L_{\mathrm{train}}}\mathcal{H}_l
$
be the observed prefix of a hierarchy with branching factor $K$. For each
parent node $x\in\mathcal H_l$, we denote its children by
$
\mathrm{chld}(x)_1,\dots,\mathrm{chld}(x)_K \in \mathcal H_{l+1}.
$
We model refinement by a residual rule
$
\mathrm{chld}(x)_k = x + s^{\,l+1} r(x)_k
$ for $k=1,\dots,K,
$
and train a CFP $\hat r$ by minimizing
\begin{equation}
\label{eq:loss}
L\!\left(\hat r \,\middle|\, \mathcal H_{\mathrm{train}}\right)
=
\sum_{l=0}^{L_{\mathrm{train}}-1}
\sum_{x\in \mathcal H_l}
\sum_{k=1}^K
\left\|
x+s^{\,l+1}\hat r(x)_k-\mathrm{chld}(x)_k
\right\|_2^2.
\end{equation}
At test time, $\hat r$ is applied recursively from the root. Every generated child becomes a parent for the next level, and no ground-truth nodes from unseen levels are used during roll-out. Thus the evaluation measures recursive field completion rather than one-step prediction under teacher forcing. All related experiment details are given in Appendix~\ref{s:Experiments_Details}.

\paragraph{Metrics}
We use three metrics corresponding to three aspects of a classification field.
MSE is order-sensitive and tests whether the CFP recovers the canonical child slots.
PI-CD is permutation-invariant and tests whether the predicted children match the target child set geometrically even if slot labels are ignored. 
dpt-distortion compares the shortest-path metrics induced by Euclidean parent-child edge lengths on the predicted and target hierarchy truncations. 
Together, these metrics distinguish slotwise refinement accuracy, unordered geometric fidelity, and preservation of hierarchy-level path structure.
Full definitions are given in Appendix~\ref{A:Metric}.

\paragraph{Extrapolation on CFG-generated hierarchies}
We first evaluate the matched setting in which the target hierarchy is generated by a CFG from the same recursive family used in the theory. This experiment asks whether finite observations of a classification field are sufficient to recover a reusable local refinement rule. Figure~\ref{fig:MSE} shows that the CFP maintains the lowest level-wise error across the unseen rollout horizon. 
The average and learnable-constant baselines degrade quickly, indicating that a single global refinement template cannot represent an input-dependent field. 
The affine predictor is stronger, but still higher above the CFP. 
This gap is important because rollout compounds errors recursively. A predictor that is only locally accurate but does not capture the field structure would drift as generated nodes become future parents. 
The stable CFP curves therefore support the classification field claim that the learned local rule remains coherent under repeated self-composition. 
Appendix~\ref{A:ablation} further studies how this behavior depends on observed depth, scale factor, and child-slot consistency.

\begin{figure}[htbp]
    \centering
    \includegraphics[width=0.5\linewidth]{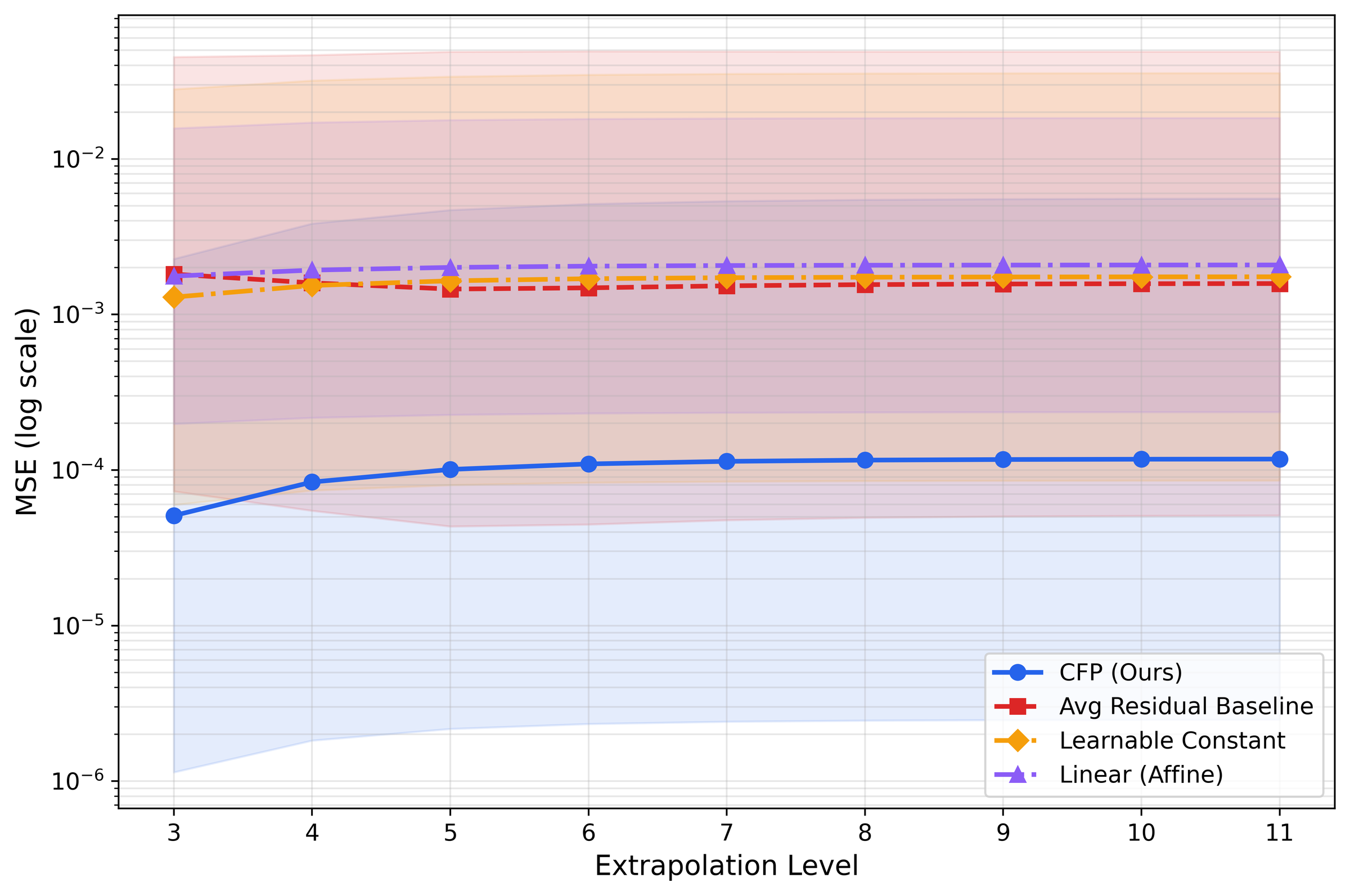}
    \caption{Matched CFG rollout. All methods are trained on levels 0–2 and recursively rolled out for nine unseen levels without teacher forcing. Curves report the geometric mean of level-wise MSE over 9 random CFG trials, with shaded bands corresponding to \(\pm 1\) standard deviation in \(\log_{10}\) MSE. CFP remains stable under repeated self-composition, whereas location-independent templates degrade and the affine rule plateaus at higher error.}
    \label{fig:MSE}
\end{figure}

\paragraph{Generalization to synthetic nonlinear IFS hierarchies}
We next test exact recursive hierarchies outside the matched CFG family.
Iterated function systems define a parent-to-child map \(f_k(x)\), so the update can be written as $x+s^{\,l+1}r(x)_k=f_k(x).$ 
For these benchmarks we set the external scale to \(s=1\) and absorb all contractions into the maps \(f_k\), so that \(r(x)_k=f_k(x)-x\).
Starting from a root \(x_0\in\mathbb R^d\), the hierarchy is generated recursively by applying all maps to every node at the current level.

\begin{figure}[htbp!]
    \centering
    \includegraphics[width=\linewidth]{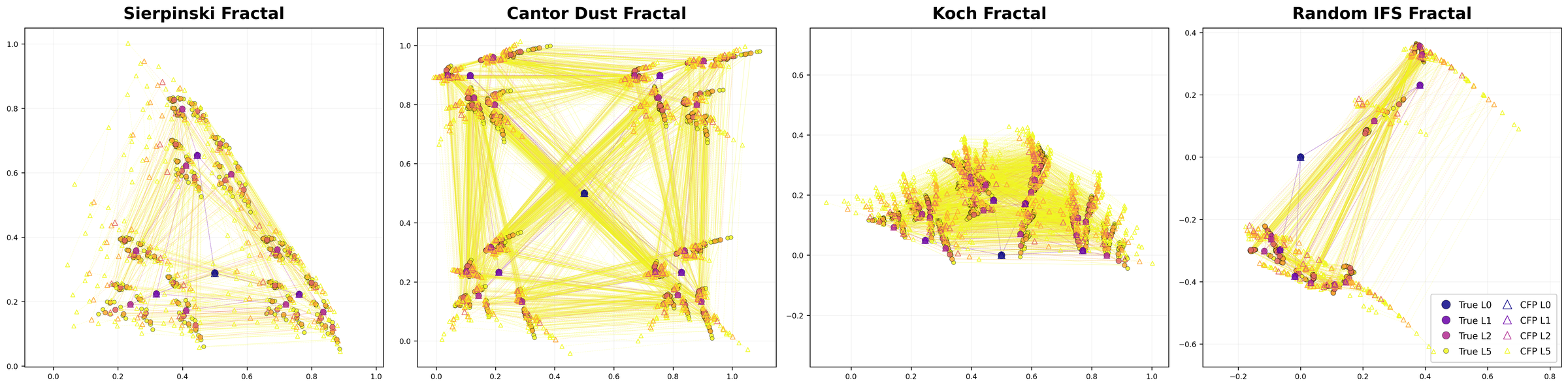}
    \caption{Recursive rollout on nonlinear IFS-induced classification fields. From left to right: Sierpiński triangle, 2D Cantor dust, Koch curve, and random nonlinear IFS. Models are trained only on levels 0–2 and rolled out to level 5. Filled circles and solid edges denote ground truth; hollow triangles and dashed edges denote CFP predictions. These examples test whether learned local refinement rules preserve recursive geometry outside the matched CFG family.}
    \label{fig:nonlinear_benchmark}
\end{figure}

We evaluate four smooth nonlinear IFS benchmarks obtained by perturbing classical fractal IFS families with child-specific sinusoidal warps: nonlinear Sierpi\'nski triangle (\(K=3\)), nonlinear 2D Cantor dust (\(K=4\)), nonlinear Koch curve (\(K=4\)), and a random nonlinear IFS family (\(K=3\)).
The first three use fixed nonlinear maps across trials, while the random nonlinear maps are resampled across trials.
All models are trained on levels \(0\)--\(2\) and evaluated on held-out extrapolation levels \(3\) to \(5\); full map definitions, nonlinear perturbation parameters, linear counterpart experiments, and root choices are given in Appendix~\ref{A:ifs_details}.

These benchmarks should be viewed as out-of-family classification fields: they have a true recursive refinement law, but that law is not necessarily sampled from the parametrized CFG class used in the theory.
Figure~\ref{fig:nonlinear_benchmark} and Table~\ref{tab:nonlinear_fractal_benchmarkwise} show that the CFP achieves the lowest MSE, PI-CD, and dpt-distortion across all hierarchies.
The result is not a claim of optimal IFS identification. Rather, it shows that the classification field formulation is not restricted to the matched generator. 
Learning an input-dependent residual rule can preserve recursive geometry and hierarchy-level path structure even when the underlying field comes from a different recursive mechanism.

\begin{table*}[hbtp!]
\centering
\scriptsize
\setlength{\tabcolsep}{4pt}
\renewcommand{\arraystretch}{0.95}
\caption{Out-of-family nonlinear IFS hierarchies. Metrics are averaged over held-out levels 3 to 5 and five random seeds. Values are reported in units of $10^{-3}$, with standard deviations reported where shown, and lower is better. The macro average is computed uniformly across benchmarks.}
\label{tab:nonlinear_fractal_benchmarkwise}
\scalebox{0.95}{
\begin{tabular}{lccccc}
\toprule
Method & Nonlinear Sierpi\'nski & Nonlinear Cantor Dust & Nonlinear Koch Curve & Rand. Nonlinear IFS & Macro Avg. \\
\midrule
\multicolumn{6}{l}{\textit{MSE} $\downarrow$} \\
CFP (Ours) & \textbf{0.90 $\pm$ 0.19} & \textbf{0.47 $\pm$ 0.11} & \textbf{0.70 $\pm$ 0.21} & \textbf{4.16 $\pm$ 6.95} & \textbf{1.56 $\pm$ 3.55} \\
Affine Residual Predictor & 16.46 $\pm$ 5.46 & 16.15 $\pm$ 9.28 & 9.13 $\pm$ 2.02 & 5.59 $\pm$ 6.81 & 11.84 $\pm$ 7.60 \\
Learnable Const & 75.97 $\pm$ 0.00 & 270.62 $\pm$ 0.00 & 62.93 $\pm$ 0.00 & 169.82 $\pm$ 59.54 & 144.84 $\pm$ 89.93 \\
Avg Residual & 75.98 $\pm$ 0.00 & 270.62 $\pm$ 0.00 & 62.93 $\pm$ 0.00 & 169.82 $\pm$ 59.54 & 144.84 $\pm$ 89.93 \\
\midrule
\multicolumn{6}{l}{\textit{PI-CD} $\downarrow$} \\
CFP (Ours) & \textbf{1.82 $\pm$ 0.32} & \textbf{1.00 $\pm$ 0.18} & \textbf{0.81 $\pm$ 0.16} & \textbf{4.30 $\pm$ 7.01} & \textbf{1.98 $\pm$ 3.53} \\
Affine Residual Predictor & 22.63 $\pm$ 7.65 & 19.40 $\pm$ 8.44 & 5.05 $\pm$ 1.21 & 5.84 $\pm$ 5.41 & 13.23 $\pm$ 9.95 \\
Learnable Const & 106.42 $\pm$ 0.00 & 300.61 $\pm$ 0.00 & 39.14 $\pm$ 0.00 & 162.39 $\pm$ 72.22 & 152.14 $\pm$ 104.10 \\
Avg Residual & 106.42 $\pm$ 0.00 & 300.61 $\pm$ 0.00 & 39.14 $\pm$ 0.00 & 162.39 $\pm$ 72.22 & 152.14 $\pm$ 104.10 \\
\midrule
\multicolumn{6}{l}{\textit{$d_{\mathrm{pt}}$-Distortion} $\downarrow$} \\
CFP (Ours) & \textbf{9.86 $\pm$ 1.48} & \textbf{4.70 $\pm$ 0.52} & \textbf{9.01 $\pm$ 0.53} & \textbf{11.07 $\pm$ 6.31} & \textbf{8.66 $\pm$ 3.88} \\
Affine Residual Predictor & 67.57 $\pm$ 16.28 & 47.34 $\pm$ 15.91 & 57.54 $\pm$ 15.84 & 31.59 $\pm$ 10.29 & 51.01 $\pm$ 19.25 \\
Learnable Const & 84.27 $\pm$ 0.00 & 95.92 $\pm$ 0.00 & 99.86 $\pm$ 0.00 & 125.82 $\pm$ 11.06 & 101.47 $\pm$ 16.38 \\
Avg Residual & 84.27 $\pm$ 0.00 & 95.92 $\pm$ 0.00 & 99.86 $\pm$ 0.00 & 125.82 $\pm$ 11.06 & 101.47 $\pm$ 16.38 \\
\bottomrule
\end{tabular}}
\end{table*}

\paragraph{Image-induced recursive clustering hierarchies}

Finally, we test an approximate-field setting where the reference hierarchy is
induced from real image representations rather than generated by an exact
recursive law. We extract frozen CLIP~\cite{radford2021learningtransferablevisualmodels}
embeddings from CIFAR-10 and CIFAR-100, recursively apply a deep clustering
method to obtain finite ternary centroid hierarchies, and treat these induced
hierarchies as the evaluation reference. The CFP is trained on the observed
upper levels and evaluated against the held-out level-6 centroids of the same
recursively constructed hierarchy. Thus, the ground truth in this experiment is
the deeper level produced by the recursive clustering pipeline, not the CIFAR
class labels. Implementation details are given in Appendix~\ref{A:cifar_details}. 

Note that we do not assume that CIFAR representations contain an exact or indefinitely extendable recursive hierarchy. Rather, this experiment is intended as a non-ideal stress test: the induced hierarchy may be noisy, finite-depth, and only locally self-consistent, and we ask whether useful refinement structure is still learnable in this real-data setting. Also, this setting should be interpreted differently from the CFG and IFS experiments. The target hierarchy is obtained from a finite recursive clustering procedure in representation space, rather than from an exact classification field generated by a known CFG. Consequently, the underlying refinement rule may be noisy, consistent only locally, and only approximately reusable across branches. In addition, child slots do not have a natural alignment and must be canonicalized before applying an order-sensitive loss. Our goal is not to improve image clustering performance directly, but to examine whether the induced centroid hierarchy admits a useful local approximation of recursive refinement.

\begin{table}[hbtp!]
\centering
\scriptsize
\setlength{\tabcolsep}{2.2pt}
\renewcommand{\arraystretch}{0.9}
\caption{Image-induced recursive clustering hierarchies at held-out level 6 over five random seeds. Hierarchies are built by recursively clustering frozen CLIP embeddings, and metrics are computed in the original embedding space after child-slot canonicalization. Values are reported in units of $10^{-3}$ as mean $\pm$ standard deviation, and lower is better.}
\label{tab:cifar_benchmarkwise}
\begin{tabular}{lccc}
\toprule
Method & CIFAR-10 & CIFAR-100 & Macro Avg. \\
\midrule
\multicolumn{4}{l}{\textit{MSE} $\downarrow$} \\
CFP (Ours) & \textbf{2.55 $\pm$ 0.16} & \textbf{3.83 $\pm$ 0.08} & \textbf{3.19 $\pm$ 0.09} \\
Affine Residual Predictor & 6.10 $\pm$ 0.13 & 4.99 $\pm$ 0.10 & 5.55 $\pm$ 0.09 \\
Learnable Const & 6.13 $\pm$ 0.25 & 4.98 $\pm$ 0.23 & 5.56 $\pm$ 0.17 \\
Avg Residual & 6.14 $\pm$ 0.23 & 4.99 $\pm$ 0.22 & 5.56 $\pm$ 0.16 \\
\midrule
\multicolumn{4}{l}{\textit{PI-CD} $\downarrow$} \\
CFP (Ours) & \textbf{822.47 $\pm$ 26.94} & \textbf{1104.90 $\pm$ 24.83} & \textbf{963.69 $\pm$ 13.45} \\
Affine Residual Predictor & 3729.85 $\pm$ 101.96 & 2920.46 $\pm$ 65.18 & 3325.16 $\pm$ 73.01 \\
Learnable Const & 3740.01 $\pm$ 70.42 & 2956.31 $\pm$ 80.87 & 3348.16 $\pm$ 67.59 \\
Avg Residual & 3770.91 $\pm$ 64.72 & 2972.16 $\pm$ 80.16 & 3371.54 $\pm$ 63.61 \\
\midrule
\multicolumn{4}{l}{\textit{$d_{\mathrm{pt}}$-Distortion} $\downarrow$} \\
CFP (Ours) & \textbf{22.68 $\pm$ 0.97} & \textbf{30.15 $\pm$ 1.98} & \textbf{26.41 $\pm$ 0.70} \\
Affine Residual Predictor & 478.95 $\pm$ 32.53 & 410.07 $\pm$ 35.60 & 444.51 $\pm$ 20.79 \\
Learnable Const & 522.23 $\pm$ 29.63 & 500.72 $\pm$ 33.19 & 511.47 $\pm$ 21.11 \\
Avg Residual & 525.55 $\pm$ 28.69 & 501.54 $\pm$ 32.42 & 513.54 $\pm$ 20.60 \\
\bottomrule
\end{tabular}
\end{table}

Table~\ref{tab:cifar_benchmarkwise} shows that the CFP obtains lower errors than the diagnostic baselines on both CIFAR-10 and CIFAR-100 under all three metrics.
The gains are especially large for PI-CD and $d_{\mathrm{pt}}$-distortion, suggesting that the advantage is not merely slotwise coordinate prediction but preservation of the large-scale branching geometry induced by recursive clustering.

Figure~\ref{fig:cifar_hierarchy} provides qualitative examples: CIFAR-10 exhibits a relatively stable branching pattern, while CIFAR-100 is more diffuse and less uniform across branches. Even in this less regular setting, the predicted rollout preserves the dominant hierarchical geometry.

\begin{figure}[htbp]
    \centering
    \begin{subfigure}[t]{0.45\linewidth}
        \centering
        \includegraphics[width=\linewidth]{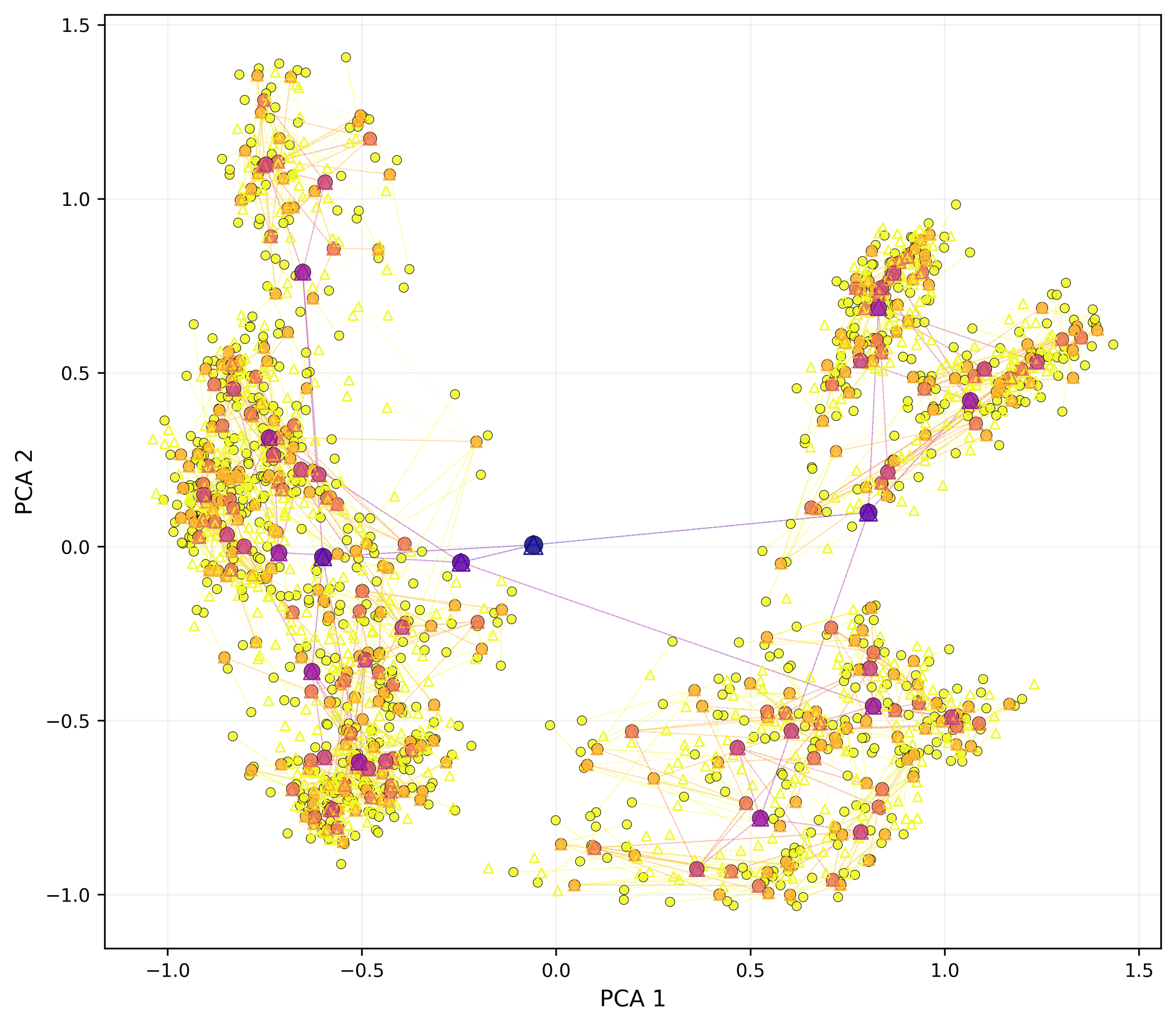}
        \caption{CIFAR-10 hierarchy}
        \label{fig:cifar10_scp_hierarchy}
    \end{subfigure}
    \hfill
    \begin{subfigure}[t]{0.45\linewidth}
        \centering
        \includegraphics[width=\linewidth]{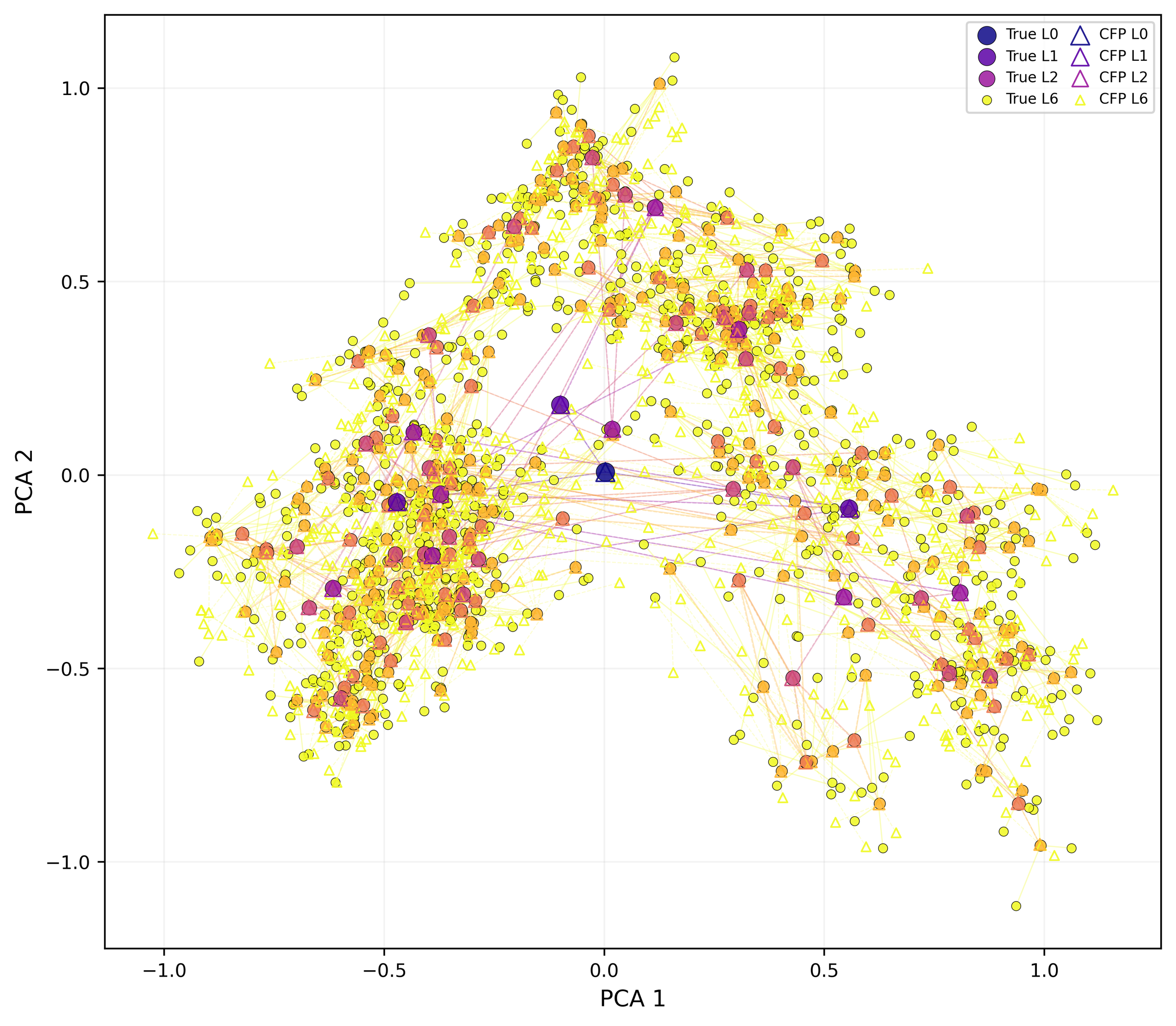}
        \caption{CIFAR-100 hierarchy}
        \label{fig:cifar100_scp_hierarchy}
    \end{subfigure}
    \caption{CIFAR-induced approximate classification fields. CLIP-space centroids are visualized after PCA projection, while all metrics are computed in the original embedding space. CIFAR-10 exhibits a more stable recursive branching pattern, whereas CIFAR-100 is more diffuse across branches. The plots qualitatively illustrate the predicted rollout geometry at the held-out level.}
    \label{fig:cifar_hierarchy}
\end{figure}

Overall, the three experimental settings support the classification field
viewpoint. The matched CFG experiments confirm the mechanism in the setting
closest to the theory, where the target hierarchy is generated by a reusable
local refinement rule. The nonlinear IFS experiments test a stricter form of
model mismatch. The image-induced experiments further stress the framework in an approximate
real-data setting. 
Across these regimes, the CFP improves not only in slotwise prediction error, but also in unordered
geometric matching and hierarchy-level path distortion. These results suggest
that finite hierarchy prefixes can contain learnable refinement mechanisms that
extend to deeper levels, provided the underlying hierarchy has sufficient local
regularity. For limitations see Appendix~\ref{A:limit}.

\section{Conclusion}
We introduced classification fields as a metric-geometric model for recursive
hierarchical clustering beyond finite dendrograms.
The key shift is from learning a hierarchy over finite observed samples to learning a local
parent-to-child refinement rule whose rollout defines progressively finer
cluster structure. 
Our theory formalizes the resulting infinite object as a
recursively generated metric DAG and shows that finite truncations converge to
the completed cell-metric field, while the finite refinement windows needed for
such truncations can be realized by ReLU CFPs. 
Across matched CFGs, nonlinear IFS hierarchies, and image-induced recursive clustering, CFP rollouts preserve slot correspondence, unordered geometry, and hierarchy-level path metrics.
These results support classification fields as a complementary view of hierarchical clustering, while leaving noisy, variable-branching, and weakly aligned hierarchies as important directions for future work.

\section{Acknowledgements and Funding}
\label{s:AckFund}
A.\ Kratsios, Y.\ Li, and R.\ Hong acknowledges financial support from an NSERC Discovery Grant No.\ RGPIN-2023-04482 and No.\ DGECR-2023-00230.  
P.\ McNicholas and Y.\ Li have been funded by the NSERC Discovery Grant No.\ RGPIN-2023-06030, the Canada Research Chairs program No.\ num-
ber CRC-2022-00494, and a Dorothy Killam Fellowship. 
The authors acknowledge that resources used in preparing this research were provided, in part, by the Province of Ontario, the Government of Canada through CIFAR, and companies sponsoring the Vector Institute\footnote{\href{https://vectorinstitute.ai/partnerships/current-partners/}{https://vectorinstitute.ai/partnerships/current-partners/}}.

\bibliographystyle{plainnat}
\bibliography{references}

\appendix
\clearpage
\section{Preliminaries}
\label{s:Prelim}

In this appendix, we review the mathematical background used in the paper.

\subsection{Notation}
\label{s:Prelim__ss:Notation}
Let $d,D\in \mathbb{N}_+$, for any $d\times D$ matrix $A$ and any $1\le p,q\le \infty$ we write $\|A\|_{p:q}\eqdef \|(\|A_{i:}\|_p)_{i=1}^d\|_q$ where $A_{i:}$ denotes the $i^{th}$ row of $A$.  
For any radius $r\ge 0$ any centre $x\in \mathbb{R}^d$ we define the \textit{closed} $\ell^p$-ball of dimension $d$ to be 
$
    \mathbb{B}_p^d(x,r) \eqdef \{z\in \mathbb{R}^d:\|z-x\|_p\le r\}
$.
For any vector $x\in \mathbb{R}^d$ we write 
$\operatorname{diag}(x)\eqdef (x_i I_{i=j})_{i,j=1}^d$.

\subsection{Background}
\label{s:Prelim__ss:Background}

Next, we discuss the background, with a particular focus on graphs and metric geometry.

\subsubsection{Weighted Graphs}
\label{s:Prelim__ss:Background___sss:WeightedGraphs}

In this paper, we consider connected \textit{weighted directed graphs}.  A weighted directed graph is a triple $D=(E,V,W)$ consisting of a (possibly infinite) non-empty set $V\subseteq \mathbb{N}$ of vertices, a set of directed edges $E\subseteq \{(v_1,v_2)\in V^2:\, v_1\neq v_2\}$, and a weight map $W:E\to (0,\infty)$.  
A \textit{directed path} is a finite sequence $(v_n)_{n=1}^N\subseteq V$ satisfying $(v_n,v_{n+1})\in E$ for each $n=1,\dots,N-1$.  
The graph $D=(E,V,W)$ is called a \textit{directed acyclic graph} (DAG) if $E$ has no directed cycles; i.e.\ there is no directed path $(v_n)_{n=1}^N$ with $N\geq 2$ such that $v_1=v_N$.  
A parent of a vertex $v\in V$ is a vertex $w\in V$ for which $(w,v)\in E$, in which case we say that $v$ is a \textit{child} of $w$.  
The set of parents of a vertex $v$ is denoted by $\pa{v}\eqdef \{w\in V:\, (w,v)\in E\}$ and the set of children of $v$ is denoted by $\ch{v}\eqdef \{w\in V:\, (v,w)\in E\}$.  
We say that $D=(E,V,W)$ is \textit{connected} if its underlying undirected graph is connected; i.e.\ for every $u,v\in V$ there exists a finite sequence $(v_n)_{n=1}^N\subseteq V$ with $v_1=u$ and $v_N=v$ such that, for every $n=1,\dots,N-1$, either $(v_n,v_{n+1})\in E$ or $(v_{n+1},v_n)\in E$.

A connected weighted DAG $D=(E,V,W)$ is called a \textit{rooted $K$-ary tree} if there exists a distinguished vertex $v_{\star}\in V$, called the \textit{root}, such that $\pa{v_{\star}}=\emptyset$, every vertex $v\in V\setminus\{v_{\star}\}$ has exactly one parent,and every non-leaf vertex has exactly \(K\) children; i.e.\ \(|\pa{v}|=1\) for all \(v\neq v_{\star}\), and \(|\ch{v}|\in\{0,K\}\) for all \(v\in V\).

\subsubsection{Metric Graphs}
\label{s:Prelim__ss:Background___sss:MetGraphs}

Let
$
    G=(E,V,W)
$
be a connected weighted graph with positive edge weights, that is,
$
    W:E\to (0,\infty)
$
and whose underlying graph $(V,E)$ is connected. Then the associated shortest-path metric
$
    d_G:V\times V\to [0,\infty)
$
is defined by
\[
    d_G(u,v)
    \eqdef
    \inf
    \Bigg\{
        \sum_{i=1}^n
        W\big(\{u_{i-1},u_i\}\big)
        \,:\,
        n\in\mathbb{N},
        \ u_0=u,\ u_n=v,\ 
        \{u_{i-1},u_i\}\in E\ \forall i\in [n]_+
    \Bigg\}
.
\]
That is, $d_G(u,v)$ is the minimum total edge-weight among all finite paths in $G$ joining $u$ and $v$.

\subsubsection{Hausdorff Distance and Voronoi Cells}
\label{s:Prelim__ss:Background___sss:VoronoiCells}

When comparing (non-empty compact) clusters, we will always make reference to the Hausdorff distance on ``hyperspace'' of closed-subsets of a metric space.  See~\cite{alma992483914405151} for details.

\begin{definition}[Hausdorff Metric]
\label{defn:hausdorff_metric}
Let $(K,\rho)$ be a compact metric space.  The Hausdorff \textit{hyperspace} thereon has pointset
\[
\mathbb{H}(K,\rho)
\eqdef
\left\{
C\subseteq K:\ C\neq \emptyset,\ C\ \text{is compact}
\right\}
\]
and is metrized by the \textit{Hausdorff metric} defined for any $A,B\in \mathbb{H}(K,\rho)$ by
\[
d_{\mathbb{H}}(A,B)
\eqdef
\max\left\{
\sup_{a\in A}\inf_{b\in B}\rho(a,b),
\,
\sup_{b\in B}\inf_{a\in A}\rho(a,b)
\right\}.
\]
\end{definition}
The Hausdorff metric will allow us to compare clusters, quantitatively.  
The following result will help us compare \textit{Voronoi} cells, which will form clusters at a given level in our Hierarchical clustering pipeline, and their centres; as well as with the space which they are partitioning.
\begin{proposition}[Comparison: Voronoi Cells and Their Centres]
\label{prop:treebuilder_bounds}
Let $(K,\rho)$ be a compact metric space, and let $(\mathbb{H}(K,\rho),d_{\mathbb{H}})$ denote its Hausdorff hyperspace.  Fix $N\in \mathbb{N}_+$ and let $x_1,\dots,x_N\in K$ be pairwise distinct.  For each $n\in [N]_+$, define the associated Voronoi cell by
\[
C_n \eqdef \left\{x\in K:\ \rho(x,x_n)=\min_{i\in [N]_+}\rho(x,x_i)\right\}.
\]
Then, for every $n,m\in [N]_+$,
\begin{equation}
\label{eq:comparison_inequality}
    \frac{1}{2}\,\rho(x_n,x_m)
\le 
    d_{\mathbb{H}}(C_n,C_m)
\le 
    \operatorname{diam}(K,\rho)
.
\end{equation}
\end{proposition}

\begin{proof}
Since $x_n\in C_n$ for every $n\in [N]_+$, it suffices to prove the lower bound.  Fix $n,m\in [N]_+$ and let $z\in C_m$.  Then
$
\rho(z,x_m)\le \rho(z,x_n)
$. 
By the triangle inequality,
$
\rho(x_n,x_m)\le \rho(x_n,z)+\rho(z,x_m)\le 2\,\rho(x_n,z)
$.  
Therefore,
$
\rho(x_n,z)\ge \frac{1}{2}\,\rho(x_n,x_m)
$, for all $z\in C_m$.
Taking the infimum over $z\in C_m$ yields
$
\operatorname{dist}(x_n,C_m)\ge \frac{1}{2}\,\rho(x_n,x_m)
$.
Since $x_n\in C_n$, we obtain
$
d_{\mathbb{H}}(C_n,C_m)\ge \operatorname{dist}(x_n,C_m)\ge \frac{1}{2}\,\rho(x_n,x_m)
$.
The upper bound is immediate since $C_n,C_m\subseteq K$.
\end{proof}
Although the upper-bound in~\eqref{eq:comparison_inequality} seems trivial, it will be instrumental in allowing us to unambiguously define a distance between infinite-depth hierarchies.

\subsubsection{Gromov-Hausdorff Distance}
\label{s:Prelim__ss:Background___sss:GromovHausdorff}

It will often be convenient for us to compare metric graphs through their
associated compact metric spaces~\citep{borde2023neural}. To this end, we recall the usual
Gromov--Hausdorff space of compact metric spaces. Namely, let
\[
\mathbb{GH}
\eqdef
\left\{
(K,\rho):\ (K,\rho)\ \text{is a compact metric space}
\right\}\big/\cong
\]
denote the collection of compact metric spaces modulo isometry.
For any two compact metric spaces \((K,\rho_K)\) and \((L,\rho_L)\), their
\textit{Gromov--Hausdorff distance} is defined by
\[
d_{\operatorname{GH}}\big((K,\rho_K),(L,\rho_L)\big)
\eqdef
\inf_{(Z,\rho_Z),\,\varphi,\,\psi}
d_{\mathbb{H}}\big(\varphi(K),\psi(L)\big),
\]
where the infimum is taken over all metric spaces \((Z,\rho_Z)\) and all
isometric embeddings
\[
\varphi:(K,\rho_K)\hookrightarrow (Z,\rho_Z),
\qquad
\psi:(L,\rho_L)\hookrightarrow (Z,\rho_Z),
\]
and where \(d_{\mathbb{H}}\) denotes the Hausdorff metric on
\(\mathbb{H}(Z,\rho_Z)\) as in Definition~\ref{defn:hausdorff_metric}.
The metric space \((\mathbb{GH},d_{\operatorname{GH}})\) is called the
\textit{Gromov--Hausdorff metric space}.

\subsection{ReLU Multilayer Perceptrons}
\label{s:Prelim__ss:Background___sss:MLPs}

We now define multilayer perceptrons.   
\begin{definition}[Multilayer Perceptrons with ReLU Activation Function (ReLU MLPs)]
\label{defn:ReLUMLP}
Let $\Delta \in \mathbb{N}_+$ and consider a multi-index $\mathbf{d}\eqdef [d_1,\dots,d_{\Delta+1}]\in \mathbb{N}_+^d$.  
A ReLU MLP is a function $\Phi$ which admits the following iterative representation
\begin{equation}
\label{eq:MLPRepresentation}
    \begin{aligned}
        \Phi(\mathbf{x}) & = \mathbf{W}^{(\Delta)} \mathbf{x}^{(\Delta)}+\mathbf{b}^{(\Delta)}\\
     \mathbf{x}^{(l+1)}& \eqdef \operatorname{ReLU} \bullet\big(\mathbf{W}^{(l)}\,\mathbf{x}^{(l)} + \mathbf{b}^{(l)}\big) 
        \qquad
        \mbox{ for } l=1,\dots,\Delta-1
    \\
    \mathbf{x}^{(1)} & \eqdef  \mathbf{x}.
    \end{aligned}
\end{equation}
where for $l=1,\dots,\Delta$, $\mathbf{W}^{(l)}$ is a $d_{l+1}\times d_l$-matrix and $\mathbf{b}^{(l)}\in \mathbb{R}^{d_{l+1}}$, and $\operatorname{ReLU}\bullet$ denotes componentwise application of the $\operatorname{ReLU}$ function, where $\operatorname{ReLU}\eqdef\max\{0,\cdot\}$.
\end{definition}

\subsection{Additional Background}
\label{s:add_back}
\paragraph{(1) Combinatorial: Tree edge weights.}
The first weighting scheme is purely combinatorial, and does not use the ambient geometry of $\mathbb{R}^d$ nor of the hyperspace $\mathbb{H}(\mathbf{B}_p^d(x_0,\lambda),\|\cdot\|_p)$.
For every edge joining level $l-1$ to level $l$, we assign weight
\begin{equation}
\label{eq:combinatorial_weight}
    W_{\operatorname{abs}}
    \big(
        (C_{\alpha},x_{\alpha},\alpha),
        (C_{\alpha k},x_{\alpha k},\alpha k)
    \big)
    \eqdef
    s^l
\end{equation}
for every $\alpha\in [K]_+^{\,l-1}$ and every $k\in [K]_+$.
That is, the weight of an edge depends only on the level at which the child is created.
We write $(\mathcal{V}^{\operatorname{tree}},d_{\operatorname{tree}})$ for the shortest-path metric on the undirected tree $\overline{G}^{\operatorname{tree}}$ with ``abstract combinatorial'' edge weights in~\eqref{eq:combinatorial_weight}.

\paragraph{(2) Cluster-Centre Based Edge Weights.}
The third weighting scheme measures the ambient geometric distance between the distinguished points of the parent and child cells. Thus, we define
\[
    W_{\operatorname{pt}}
    \big(
        (C_{\alpha},x_{\alpha},\alpha),
        (C_{\alpha k},x_{\alpha k},\alpha k)
    \big)
    \eqdef
    \|x_{\alpha k}-x_{\alpha}\|_p
.
\]
Equivalently, since
$
x_{\alpha k}=x_{\alpha}+s^l r(x_{\alpha})_k
$,
we consider the ``base-point'' edge-weights
\begin{equation}
\label{eq:center_point_weights}
    W_{\operatorname{pt}}
    \big(
        (C_{\alpha},x_{\alpha},\alpha),
        (C_{\alpha k},x_{\alpha k},\alpha k)
    \big)
    =
    s^l\|r(x_{\alpha})_k\|_p
\end{equation}
for every $\alpha\in [K]_+^{\,l-1}$ and every $k\in [K]_+$.

We write $G_{\operatorname{pt}}\eqdef (\mathcal{V}^{\operatorname{geo}},d_{\operatorname{pt}})$ for the shortest-path metric on the weighted graph $(\mathcal{V}^{\operatorname{geo}},E^{\operatorname{geo}},W_{\operatorname{pt}})$ with ``base-point'' edge-weights in~\eqref{eq:center_point_weights}.

\section{Additional Results}
\label{s:add_results}

Here, we provide additional theoretical results.

\subsection{Regularity and separation}
\label{s:good_properties}
\begin{proposition}[Lipschitz regularity of parametrized CFGs]
\label{prop:lipschitz_parametrized_rhg}
Fix $d,K\in\mathbb{N}_+$ and constants
$
0<\sigma_{\min}\le \sigma_{\max}<\infty$, 
$0<\lambda_-\le \lambda_+<\infty$, and $0<\varepsilon$.
Consider a $d\times K$ $(\lambda_-,\lambda_+,\varepsilon)$-reference packing $C=(c_1|\cdots|c_K)$.  
Let
$
A_1,A_2:\mathbb{B}_2^d\to \mathbb{R}^{d\times d}$ and $
\sigma:\mathbb{B}_2^d\to [\sigma_{\min},\sigma_{\max}]^d
$
satisfy the following Lipschitz estimates
\begin{align}
\label{eq:lipschitz_parametrized_rhg_A1}
\|A_1(x)-A_1(\tilde{x})\|_F
&\le
L_{A_1}\|x-\tilde{x}\|_2,
\\
\label{eq:lipschitz_parametrized_rhg_A2}
\|A_2(x)-A_2(\tilde{x})\|_F
&\le
L_{A_2}\|x-\tilde{x}\|_2,
\\
\label{eq:lipschitz_parametrized_rhg_sigma}
\|\sigma(x)-\sigma(\tilde{x})\|_2
&\le
L_\sigma\|x-\tilde{x}\|_2
\end{align}
for some constants $L_{A_1},L_{A_2},L_\sigma\ge 0$ and 
for every $x,\tilde{x}\in \mathbb{B}_2^d$. Assume also that $A_1(x)$ and $A_2(x)$ are skew-symmetric for every $x\in\mathbb{B}_2^d$.
\hfill\\
\noindent
For every $x\in \mathbb{B}_2^d$ define the $d\times K$ matrix
\[
r(x)
\eqdef
e^{A_1(x)}\operatorname{diag}(\sigma(x))e^{A_2(x)}{}^\top C
.
\]
and write $r(x)=(r_1(x)|\cdots|r_K(x))$.
Then, for every $x,\tilde{x}\in \mathbb{B}_2^d$,
\begin{equation}
\label{eq:lipschitz_parametrized_rhg_conclusion}
\max_{k\in[K]_+}
\|r_k(x)-r_k(\tilde{x})\|_2
\le
L_r\,\|x-\tilde{x}\|_2,
\end{equation}
where $
    L_r
\eqdef
    \lambda_+
    \Big(
    \sigma_{\max}L_{A_1}
    +
    L_\sigma
    +
    \sigma_{\max}L_{A_2}
    \Big)
$.
\end{proposition}
\begin{proof}
See Appendix~\ref{s:orth_diag_param_sep_annuli}.
\end{proof}

The key point of the structure of $r$ is that, under the above regularity we may guarantee the separation between cluster centres in the same level.  Let us defined this separation as follows: fix some $x_0\in \mathbb{B}_2^d$
and define recursively
$
x_{\varnothing}\eqdef x_0,
$ and $
x_{\alpha k}
\eqdef
x_\alpha+s^{|\alpha|+1}r_k(x_\alpha)
$
for every address $\alpha\in [K]_+^{\,|\alpha|}$ and every $k\in[K]_+$.
For each $l\in\mathbb{N}$, let
$
\mathcal{X}_l
\eqdef
\{x_\alpha:\ \alpha\in [K]_+^l\},
$
and define the minimum \textbf{intra-level-$l$ separation}, for every $l\in \mathbb{N}_+$, by
\begin{equation}
\label{eq:delta_l__separation}
    \delta_l
\eqdef
    \min_{\substack{x,\tilde{x}\in \mathcal{X}_l\\ x\neq \tilde{x}}}
    \|x-\tilde{x}\|_2
.
\end{equation}


The mild regularity conditions on the CFG in Proposition~\ref{prop:param_uniformly_separated_annular_configs} are enough to guarantee intra-level separation.
\begin{proposition}[{$\Theta(s^l)$-Lower-Bound on Intra-level $l$ Separation}]
\label{prop:key_separation}
If $r$ is a CFG as constructed in Proposition~\ref{prop:param_uniformly_separated_annular_configs} and $A_1$, $A_2$, $\sigma$ satisfy the hypotheses of Proposition~\ref{prop:lipschitz_parametrized_rhg} then for any $l\in \mathbb{N}_+$
\begin{equation}
\label{eq:recursive_delta_lower_bound_parametrized}
\delta_{l+1}
\ge
\max\left\{
0,\,
\min\left\{
s^{l+1}\sigma_{\min}\varepsilon,
\,
\left(
1-s^{l+1}\lambda_+
\big(
\sigma_{\max}L_{A_1}+L_\sigma+\sigma_{\max}L_{A_2}
\big)
\right)\delta_l
-
2s^{l+1}\sigma_{\max}\lambda_+
\right\}
\right\}
\end{equation}
where $
\lambda_{\max}^{\operatorname{(lemma)}}=\sigma_{\max}\lambda_+
$, 
$s_{\operatorname{sep}}^{\operatorname{(lemma)}}=\sigma_{\min}\varepsilon$, and
$
L_r=\lambda_+
\Big(
\sigma_{\max}L_{A_1}
+
L_\sigma
+
\sigma_{\max}L_{A_2}
\Big)$.
\end{proposition}

\paragraph{Benefits of The Tree Metric}
The first main motivation of the tree metric is that it yields a computationally simple distance function, completely available in closed-form.

\begin{proposition}[Closed-Form Expression for the Tree Metric]
\label{prop:closed_form_tree_metric}
Let
$
u=(C_{\alpha},x_{\alpha},\alpha)
$
and
$
v=(C_{\beta},x_{\beta},\beta)
$
be vertices in $\mathcal{V}^{\operatorname{tree}}$, with
$
|\alpha|=m
$
and
$
|\beta|=n
$.
Write $\alpha\wedge\beta$ for the longest common prefix%
\footnote{A prefix of a word $\alpha=(\alpha_1,\dots,\alpha_m)\in [K]_+^m$ is any initial segment of the form $(\alpha_1,\dots,\alpha_j)$ for some $j\in\{0,\dots,m\}$, where $j=0$ corresponds to the empty word $\varnothing$.}~%
of $\alpha$ and $\beta$, and let
$
c\eqdef |\alpha\wedge\beta|
$.
Then
\begin{equation}
\label{eq:closed_form_tree_metric}
        d_{\operatorname{tree}}(u,v)
    =
        \frac{2s^{c+1}-s^{m+1}-s^{n+1}}{1-s}
.
\end{equation}
\end{proposition}

\begin{proof}[{Proof of Proposition~\ref{prop:closed_form_tree_metric}}]
Since $\overline{G}^{\operatorname{tree}}$ is a tree, there is a unique simple path joining $u$ and $v$. This path runs from $u$ up to the vertex indexed by the longest common prefix $\alpha\wedge\beta$, and then from $\alpha\wedge\beta$ down to $v$. Hence its total weight is
\[
    d_{\operatorname{tree}}(u,v)
    =
    \sum_{j=c+1}^{m}s^j
    +
    \sum_{j=c+1}^{n}s^j.
\]
Evaluating these geometric sums yields~\eqref{eq:closed_form_tree_metric}.
\end{proof}
However, the tree metric can be much larger than the other two metrics; cf.\ Proposition~\ref{prop:tree_vs_geo_metrics} for details.

Nevertheless, its size is not exorbitantly large in comparison and its approximate error is largely controlled by the scale parameter.
Since the tree metric can, and will often, have a larger vertex set, we will often have to pass to its quotient; which we formalize via the quotient map
$
    \pi:\mathcal{V}^{\operatorname{tree}}\to \mathcal{V}^{\operatorname{geo}}
$
denote the canonical quotient map. 
Having defined $\pi$, we now make comparisons between these three metrics; we use the following ratio
\begin{equation}
\label{eq:funny_constant}
    \kappa_{p,2}(d)
\eqdef
    \sup_{z\neq 0}\frac{\|z\|_p}{\|z\|_2}
    =
    d^{\max\{\frac1p-\frac12,0\}}
.
\end{equation}
\begin{proposition}[Domination by Tree Metric]
\label{prop:tree_vs_geo_metrics}
For every $u,v\in \mathcal{V}^{\operatorname{tree}}$, we have
\begin{equation}
\label{eq:cell_vs_tree_metric}
    d_{\operatorname{cell}}\big(\pi(u),\pi(v)\big)
    \le
    \frac{2\lambda_{\max}^2}{s}\,
    d_{\operatorname{tree}}(u,v)
    \mbox{ and }
    d_{\operatorname{pt}}\big(\pi(u),\pi(v)\big)
    \le
    \kappa_{p,2}(d)\,\lambda_{\max}\,
    d_{\operatorname{tree}}(u,v).
\end{equation}
If $p\ge 2$ and $
    \kappa_{p,2}(d)\,\lambda_{\max}\le 1
$
then
$
    d_{\operatorname{pt}}\big(\pi(u),\pi(v)\big)
    \le
    d_{\operatorname{tree}}(u,v)
$ for each $ u,v\in \mathcal{V}^{\operatorname{tree}}$.
\end{proposition}

\begin{proof}[{Proof of Proposition~\ref{prop:tree_vs_geo_metrics}}]
Fix an edge
$
e_{\alpha,k}
\eqdef
\Big(
(C_{\alpha},x_{\alpha},\alpha),
(C_{\alpha k},x_{\alpha k},\alpha k)
\Big)
\in E^{\operatorname{tree}},
$
with $|\alpha|=l-1$. By definition,
$
    W_{\operatorname{abs}}(e_{\alpha,k})=s^l
$.
For the point-based weights,
$
    W_{\operatorname{pt}}(e_{\alpha,k})
    =
    s^l\|r(x_{\alpha})_k\|_p
$.
Since
$
    \|r(x_{\alpha})_k\|_2
    \le
    \lambda_{\max}
$,
we obtain by norm comparison that
$
    \|r(x_{\alpha})_k\|_p
    \le
    \kappa_{p,2}(d)\,\|r(x_{\alpha})_k\|_2
    \le
    \kappa_{p,2}(d)\,\lambda_{\max}
$.
Hence
\begin{equation}
\label{eq:pt_vs_tree_metric}
        W_{\operatorname{pt}}(e_{\alpha,k})
        \le
        \kappa_{p,2}(d)\,\lambda_{\max}\,s^l
        =
        \kappa_{p,2}(d)\,\lambda_{\max}\,
        W_{\operatorname{abs}}(e_{\alpha,k}).
\end{equation}
For the cell-based weights,
$
    W_{\operatorname{cell}}(e_{\alpha,k})
    =
    s^{|\alpha|}\lambda_{\max}\,d_{H,p}(C_{\alpha},C_{\alpha k}).
$
Since $C_{\alpha k}\subseteq C_{\alpha}$, one has
$
    d_{H,p}(C_{\alpha},C_{\alpha k})
    \le
    \operatorname{diam}(C_{\alpha},\|\cdot\|_p)
$.
Moreover,
$
    C_{\alpha}\subseteq C_{\varnothing}
    =
    \mathbb{B}_p^d(x_0,\lambda_{\max}),
$
so
\[
    \operatorname{diam}(C_{\alpha},\|\cdot\|_p)
    \le
    \operatorname{diam}(\mathbb{B}_p^d(x_0,\lambda_{\max}),\|\cdot\|_p)
    =
    2\lambda_{\max}.
\]
Therefore, $
    W_{\operatorname{cell}}(e_{\alpha,k})
    \le
    2s^{|\alpha|}\lambda_{\max}^2
    =
    \frac{2\lambda_{\max}^2}{s}\,s^l
    =
    \frac{2\lambda_{\max}^2}{s}\,
    W_{\operatorname{abs}}(e_{\alpha,k})
$.
Now fix $u,v\in \mathcal{V}^{\operatorname{tree}}$, and let
$
    u=u_0,u_1,\dots,u_n=v
$
be the unique simple path in the undirected tree $\overline{G}^{\operatorname{tree}}$. Projecting this path under $\pi$ yields a walk in $\mathcal{V}^{\operatorname{geo}}$ from $\pi(u)$ to $\pi(v)$. Removing repetitions if necessary, we obtain a path in $\mathcal{V}^{\operatorname{geo}}$ whose total weight is no larger than the total weight of the projected walk. Hence,
\[
    d_{\operatorname{pt}}\big(\pi(u),\pi(v)\big)
    \le
    \sum_{i=1}^n W_{\operatorname{pt}}(\{u_{i-1},u_i\}),
\]
and similarly for $d_{\operatorname{cell}}$.
Summing the edgewise bounds above along the unique tree path yields
\eqref{eq:cell_vs_tree_metric} and \eqref{eq:pt_vs_tree_metric}.
\end{proof}

\subsection{Variants}
\label{s:Graphical_Model__ss:Variants}
We consider two variants of the graphs defined above: finite-depth versions, where the clustering process stops at a prescribed finite level, and completed versions, which are possibly uncountable metric completions of the infinite graphs and have better metric-space properties.

\paragraph{Finite Depth Hierarchies}
We consider the above constructions when the recursive generation of levels is stopped at some depth $L\in \mathbb{N}_+$. This is because no infinite recursion can be completed in finite time on a computer; thus, these represent what our deep learning model can at best hope to produce.

\noindent
For every $L\in \mathbb{N}$, define the depth-$L$ genealogical vertex set by
$
    \mathcal{V}^{\operatorname{tree}}_{\leq L}
    \eqdef
    \bigcup_{l=0}^L \mathcal{V}^{\operatorname{tree}}_l
$,
and its quotient by
$
    \mathcal{V}^{\operatorname{geo}}_{\leq L}
    \eqdef
    \pi\big(\mathcal{V}^{\operatorname{tree}}_{\leq L}\big)
    \subseteq
    \mathcal{V}^{\operatorname{geo}}
.
$
We then write
$
    T_L
    \eqdef
    \Big(
        \mathcal{V}^{\operatorname{tree}}_{\leq L},
        d_{\operatorname{tree}}
    \Big)
$
for the corresponding truncated metric tree,
$
    P_L
    \eqdef
    \Big(
        \mathcal{V}^{\operatorname{geo}}_{\leq L},
        d_{\operatorname{pt}}
    \Big)
$
for the corresponding truncated base-point metric graph, and
$
    C_L
    \eqdef
    \Big(
        \mathcal{V}^{\operatorname{geo}}_{\leq L},
        d_{\operatorname{cell}}
    \Big)
$
for the corresponding truncated cell-based metric graph, where $d_{\operatorname{tree}}$ is the restriction of the tree metric to $\mathcal{V}^{\operatorname{tree}}_{\leq L}$, and $d_{\operatorname{pt}}$ and $d_{\operatorname{cell}}$ are the restrictions of the corresponding quotient metrics to $\mathcal{V}^{\operatorname{geo}}_{\leq L}$.

\paragraph{Metric Completions}
\label{s:Graphical_Model__ss:Variants___sss:Complettions}
In what follows, we write
$
    T_{\infty}
    \eqdef
    \Big(
        \overline{\mathcal{V}^{\operatorname{tree}}},
        \overline{d}_{\operatorname{tree}}
    \Big)
$
for the metric completion of the full genealogical tree, and
$
    X_{\operatorname{cell}}
    \eqdef
    \Big(
        \overline{\mathcal{V}^{\operatorname{geo}}},
        \overline{d}_{\operatorname{cell}}
    \Big)
$
for the metric completion of $(\mathcal{V}^{\operatorname{geo}},d_{\operatorname{cell}})$.
\section{{Proof of Theorem~\ref{thrm:main_result}}}
\label{s:Guarantees__ss:Gromov_Hausdorff__Tree}

We now prove our main result.  Our approach is to first derive a series of comparison inequalities between our truncated tree metric geometries and that of the limiting graphs; and their completions.  Next, we quantify the size of a memorizing $\operatorname{ReLU}$-MLP capable of encoding all the parent-children relationships up to level $L$ of the tree.  Finally, putting both these steps together yields our main guarantee.

\subsection{Comparison Inequalities: \hfill\\
Finite-Depth Trees and Infinite DAGs}
\label{s:Guarantees__ss:Gromov_Hausdorff__Tree___ss:Comparisons}

Throughout our analysis it will be convenient to situate the finite trees $T_L$ against the other metric graphs capturing the geometry of infinite clusters.  This section is devoted to upper and lower bounds explaining those relationships precisely.

\subsubsection{Upper-Bounds}
\label{s:Guarantees__ss:Gromov_Hausdorff__Tree___ss:Comparisons__sss:UpperBounds}

Our first upper bound, shows that, amongst other things, if $\lambda_{\max}\ll s$ then the Gromov-Hausdorff distance from $T_L$ to $X_{\operatorname{cell}}$ is of the order of $\mathcal{O}\big(
    \frac{s+\lambda_{\max}^2}{1-s}
\big)$.
For every $L\in\mathbb{N}$, we write
\[
    Y_L
    \eqdef
    \mathcal{V}^{\operatorname{geo}}_{\le L}
    \subseteq
    X_{\operatorname{cell}}
\]
for the depth-$L$ geometric vertex set viewed as a subset of the completion $X_{\operatorname{cell}}$.
\begin{proposition}[Comparison: Finite-Depth Tree and Infinite-Depth Scale-Sensitive Cluster Metrics]
\label{prop:GH_tree_vs_cell}
Assume that $0<s<1$. Then $X_{\operatorname{cell}}$ is compact and, for every $L\in \mathbb{N}$,
\begin{align}
\label{eq:YL_to_Xcell}
    d_H^{\overline{d}_{\operatorname{cell}}}
    (Y_L,X_{\operatorname{cell}})
   &  \le
    \sum_{j=L+1}^{\infty}2\lambda_{\max}^2 s^{j-1}
    =
    \frac{2\lambda_{\max}^2 s^L}{1-s},
\\
\label{eq:tree_vs_Xcell}
    d_{GH}(T_L,X_{\operatorname{cell}})
    & \le
    \max\Big\{1
        ,
        \frac{2\lambda_{\max}^2}{s}
    -1\Big\}
    \frac{s(1-s^L)}{1-s}
    +
    \frac{2\lambda_{\max}^2 s^L}{1-s}
.
\end{align}
\end{proposition}

Before proving our upper bound, we report another useful comparison; namely the gap between $T_L$ and $T_{\infty}$ in Gromov-Hausdorff distance.
\begin{proposition}[Comparison: Finite-Depth Tree and Infinite-Depth Tree Metrics]
\label{prop:GH_tree_to_full_tree}
Assume that $0<s<1$. Then $T_{\infty}$ is compact and, for every $L\in \mathbb{N}$,
\begin{equation}
\label{eq:GH_tree_to_full_tree_bound}
    d_{GH}(T_L,T_{\infty})
    \le
    \sum_{j=L+1}^{\infty}s^j
    =
    \frac{s^{L+1}}{1-s}.
\end{equation}
In particular,
$
    \lim_{L\uparrow\infty}
    d_{GH}(T_L,T_{\infty})
    =
    0
$.
\end{proposition}

\begin{proof}
Fix $L\in \mathbb{N}$ and denote the tail geometric sum 
$
    \tau_L
    \eqdef
    \sum_{j=L+1}^{\infty}s^j
    =
    \frac{s^{L+1}}{1-s}
$.  
We first show that every vertex of the full genealogical tree lies within $d_{\operatorname{tree}}$-distance at most $\tau_L$ from $\mathcal{V}^{\operatorname{tree}}_{\leq L}$.
Indeed, let
$
u=(C_{\alpha},x_{\alpha},\alpha)\in \mathcal{V}^{\operatorname{tree}}
$
with $|\alpha|=m$.
If $m\leq L$, then $u\in \mathcal{V}^{\operatorname{tree}}_{\leq L}$.
If $m>L$, let $\alpha|_L$ denote the length-$L$ prefix of $\alpha$.
Then
$
(C_{\alpha|_L},x_{\alpha|_L},\alpha|_L)\in \mathcal{V}^{\operatorname{tree}}_{\leq L}
$,
and the unique ancestral path from $\alpha|_L$ to $\alpha$ has total weight
\[
    \sum_{j=L+1}^{m}s^j
    \le
    \tau_L.
\]
Hence $
        \sup_{u\in \mathcal{V}^{\operatorname{tree}}}
    \,
        d_{\operatorname{tree}}
        \big(
            u,\mathcal{V}^{\operatorname{tree}}_{\leq L}
        \big)
    \le
        \tau_L
$.
Now, since $\mathcal{V}^{\operatorname{tree}}_{\leq L}\subseteq \overline{\mathcal{V}^{\operatorname{tree}}}$, the reverse directed Hausdorff distance is zero. Therefore
\[
    d_H^{\overline{d}_{\operatorname{tree}}}
    \big(
        \mathcal{V}^{\operatorname{tree}}_{\leq L},
        \overline{\mathcal{V}^{\operatorname{tree}}}
    \big)
    \le
    \tau_L.
\]
As $\mathcal{V}^{\operatorname{tree}}_{\leq L}$ is finite for every $L$, it follows that $\mathcal{V}^{\operatorname{tree}}$ is totally bounded under $d_{\operatorname{tree}}$; hence its completion $T_{\infty}$ is compact.
Finally, since $\mathcal{V}^{\operatorname{tree}}_{\leq L}$ isometrically embeds into $T_{\infty}$, the standard Hausdorff bound inside a common ambient metric space yields
$
        d_{GH}(T_L,T_{\infty})
    \le
        d_H^{\overline{d}_{\operatorname{tree}}}
        \big(
            \mathcal{V}^{\operatorname{tree}}_{\leq L},
            \overline{\mathcal{V}^{\operatorname{tree}}}
        \big)
    \le
        \tau_L
$.
\end{proof}
We now show Proposition~\ref{prop:GH_tree_vs_cell}.
\begin{proof}[{Proof of Proposition~\ref{prop:GH_tree_vs_cell}}]
Abbreviate $ c_{\operatorname{cell}}
    \eqdef
    \frac{2\lambda_{\max}^2}{s}$.
We first prove \eqref{eq:YL_to_Xcell}. Fix a vertex
$
v\in \mathcal{V}^{\operatorname{geo}}
$,
and choose a representative
$
u=(C_{\alpha},x_{\alpha},\alpha)\in \mathcal{V}^{\operatorname{tree}}
$
with $\pi(u)=v$ and $|\alpha|=m$.
If $m\le L$, then $v\in Y_L$.
If $m>L$, let $\alpha|_L$ be the length-$L$ prefix of $\alpha$.
Projecting the ancestral chain from $\alpha|_L$ to $\alpha$ under $\pi$ gives a walk in $\mathcal{V}^{\operatorname{geo}}$ from $\pi(\alpha|_L)\in Y_L$ to $v$.
Hence
$
    d_{\operatorname{cell}}(v,Y_L)
\le
    \sum_{j=L+1}^{m}
    W_{\operatorname{cell}}(e_j)
$,
where $e_j$ denotes the edge joining level $j-1$ to level $j$ along that ancestral chain.
By the bound
$
    W_{\operatorname{cell}}(e_j)
\le
    2\lambda_{\max}^2 s^{j-1}
$,
we obtain
$   
        d_{\operatorname{cell}}(v,Y_L)
    \le
        \sum_{j=L+1}^{m}2\lambda_{\max}^2 s^{j-1}
    \le
        \frac{2\lambda_{\max}^2 s^L}{1-s}
$.
Taking the supremum over $v\in \mathcal{V}^{\operatorname{geo}}$ and then passing to the completion proves \eqref{eq:YL_to_Xcell}. Compactness of $X_{\operatorname{cell}}$ follows as in Proposition~\ref{prop:GH_tree_to_full_tree}.
We show that
\begin{equation}
\label{eq:tree_vs_YL}
    d_{GH}(T_L,Y_L)
    \le
    \max\{1,c_{\operatorname{cell}}-1\}
    \frac{s(1-s^L)}{1-s}.
\end{equation}
Consider the correspondence
\[
    R_L
    \eqdef
    \big\{
        (u,\pi(u))
        :\,
        u\in \mathcal{V}^{\operatorname{tree}}_{\le L}
    \big\}
    \subseteq
    \mathcal{V}^{\operatorname{tree}}_{\le L}\times Y_L.
\]
This is surjective onto both factors.
For any
$
(u,\pi(u)),(v,\pi(v))\in R_L
$,
the comparison estimate established earlier yields
\[
    d_{\operatorname{cell}}(\pi(u),\pi(v))
    \le
    c_{\operatorname{cell}}\,
    d_{\operatorname{tree}}(u,v).
\]
Therefore,
$
    \big|
        d_{\operatorname{tree}}(u,v)
        -
        d_{\operatorname{cell}}(\pi(u),\pi(v))
    \big|
\le
    \max\{1,c_{\operatorname{cell}}-1\}\,
    d_{\operatorname{tree}}(u,v)
$; and upon taking the supremum over $u,v\in \mathcal{V}^{\operatorname{tree}}_{\le L}$ we obtain
\[
    \operatorname{dis}(R_L)
\le
    \max\{1,c_{\operatorname{cell}}-1\}\,
    \operatorname{diam}(T_L).
\]
Now
$
    \operatorname{diam}(T_L)
    \le
    2\sum_{j=1}^{L}s^j
    =
    \frac{2s(1-s^L)}{1-s}
$, 
since the longest simple path joins two depth-$L$ leaves through the root.
Hence
$
    d_{GH}(T_L,Y_L)
\le
    \frac{1}{2}\operatorname{dis}(R_L)
\le
    \max\{1,c_{\operatorname{cell}}-1\}
    \frac{s(1-s^L)}{1-s}
$,
which proves \eqref{eq:tree_vs_YL}.
Finally, combining \eqref{eq:YL_to_Xcell} and \eqref{eq:tree_vs_YL} with the triangle inequality for $d_{GH}$ yields \eqref{eq:tree_vs_Xcell}. 
\end{proof}

\subsubsection{Lower-Bounds}
\label{s:Guarantees__ss:Gromov_Hausdorff__Tree___ss:Comparisons__sss:lowerBounds}
We now illustrate the near tightness of our comparison bounds. 
\begin{proposition}[Gromov--Hausdorff lower bounds against the full geometric models]
\label{prop:GH_lower_bounds_full_models}
Let $\kappa_{p,2}(d)$ be as in \eqref{eq:funny_constant}, and write
$
    G_{\operatorname{pt}}
    \eqdef
    \big(
        \overline{\mathcal{V}^{\operatorname{geo}}},
        \overline{d}_{\operatorname{pt}}
    \big)
$
and
$
    G_{\operatorname{cell}}
    \eqdef
    \big(
        \overline{\mathcal{V}^{\operatorname{geo}}},
        \overline{d}_{\operatorname{cell}}
    \big)
$.
Then, for every $L\in \mathbb{N}_+$,
\begin{align}
\label{eq:GH_lower_point_full}
    d_{GH}(T_L,G_{\operatorname{pt}})
&\ge
    \max\Big\{
        0,
        \frac{s\bigl(1-\kappa_{p,2}(d)\lambda_{\max}\bigr)-s^{L+1}}{1-s}
    \Big\},
\\
\label{eq:GH_lower_cell_full}
    d_{GH}(T_L,G_{\operatorname{cell}})
&\ge
    \max\Big\{
        0,
        \frac{s-2\lambda_{\max}^2-s^{L+1}}{1-s}
    \Big\}.
\end{align}
In particular, if $\kappa_{p,2}(d)\lambda_{\max}<1$, then
$
\liminf_{L\uparrow\infty} d_{GH}(T_L,G_{\operatorname{pt}})
\ge
\frac{s\bigl(1-\kappa_{p,2}(d)\lambda_{\max}\bigr)}{1-s}
$,
and if $2\lambda_{\max}^2<s$, then
$
\liminf_{L\uparrow\infty} d_{GH}(T_L,G_{\operatorname{cell}})
\ge
\frac{s-2\lambda_{\max}^2}{1-s}
$.
\end{proposition}
Our proof of Proposition~\ref{prop:GH_lower_bounds_full_models} builds atop the following lemma.

\begin{lemma}[Gromov--Hausdorff lower bounds]
\label{lem:GH_lower_bounds}
Let $\kappa_{p,2}(d)$ be as in \eqref{eq:funny_constant}. Then, for every $L\in \mathbb{N}_+$,
we have
\begin{align}
\label{eq:GH_lower_point}
    d_{GH}(T_L,P_L)
&\ge
    \max\Big\{
        0,
        \bigl(1-\kappa_{p,2}(d)\lambda_{\max}\bigr)
        \frac{s(1-s^L)}{1-s}
    \Big\},
\\
\label{eq:GH_lower_cell}
    d_{GH}(T_L,C_L)
&\ge
    \max\Big\{
        0,
        \Bigl(1-\frac{2\lambda_{\max}^2}{s}\Bigr)
        \frac{s(1-s^L)}{1-s}
    \Big\}.
\end{align}
In particular, if $\kappa_{p,2}(d)\lambda_{\max}<1$, then $d_{GH}(T_L,P_L)$ is bounded below by a positive constant uniformly in $L$; and if $2\lambda_{\max}^2<s$, then the same holds for $d_{GH}(T_L,C_L)$.
\end{lemma}
\begin{proof}
We first establish the following diameter inequalities
\begin{align}
\label{eq:diam_tree_formula}
    \operatorname{diam}(T_L)
& =
    2\sum_{j=1}^L s^j
=
    \frac{2s(1-s^L)}{1-s}.
\\
\label{eq:diam_point_upper}
    \operatorname{diam}(P_L)
& \le
    \kappa_{p,2}(d)\,\lambda_{\max}\,
    \operatorname{diam}(T_L),
\\
\label{eq:diam_cell_upper}
    \operatorname{diam}(C_L)
& \le
    \frac{2\lambda_{\max}^2}{s}\,
    \operatorname{diam}(T_L)
\end{align}
which will play a helpful role.
Now, the formula \eqref{eq:diam_tree_formula} follows since the largest distance in the depth-$L$ tree is realized by two level-$L$ leaves whose least common ancestor is the root.
By Proposition~\ref{prop:tree_vs_geo_metrics}, cf.~\eqref{eq:pt_vs_tree_metric} and \eqref{eq:cell_vs_tree_metric}, for every $u,v\in \mathcal{V}^{\operatorname{tree}}_{\leq L}$ one has
$
    d_{\operatorname{pt}}(\pi(u),\pi(v))
\le
    \kappa_{p,2}(d)\lambda_{\max}\,d_{\operatorname{tree}}(u,v)
$
and
$
    d_{\operatorname{cell}}(\pi(u),\pi(v))
\le
    \frac{2\lambda_{\max}^2}{s}\,d_{\operatorname{tree}}(u,v)
$.
Taking suprema over all $u,v\in \mathcal{V}^{\operatorname{tree}}_{\leq L}$ yields \eqref{eq:diam_point_upper} and \eqref{eq:diam_cell_upper}.
Finally, for any compact metric spaces $X,Y$ one has
$
d_{GH}(X,Y)\ge \frac12|\operatorname{diam}(X)-\operatorname{diam}(Y)|
$.
Applying this with $X=T_L$ and $Y=P_L$, respectively $Y=C_L$, and then using \eqref{eq:diam_tree_formula}, \eqref{eq:diam_point_upper}, and \eqref{eq:diam_cell_upper}, yields \eqref{eq:GH_lower_point} and \eqref{eq:GH_lower_cell}.
\end{proof}
\begin{proof}[{Proof of Proposition~\ref{prop:GH_lower_bounds_full_models}}]
We first estimate the distance from the depth-$L$ truncations to the full geometric models.
For the base-point metric, let $v\in \mathcal{V}^{\operatorname{geo}}$ and choose a representative
$
u=(C_\alpha,x_\alpha,\alpha)\in \mathcal{V}^{\operatorname{tree}}
$
with $\pi(u)=v$ and $|\alpha|=m$.
If $m\le L$, then $v\in \mathcal{V}^{\operatorname{geo}}_{\le L}$.
If $m>L$, let $\alpha|_L$ denote the length-$L$ prefix of $\alpha$.
Projecting the ancestral chain from $\alpha|_L$ to $\alpha$ under $\pi$ yields a walk from $\pi(\alpha|_L)\in \mathcal{V}^{\operatorname{geo}}_{\le L}$ to $v$.
Hence
\[
    d_{\operatorname{pt}}\bigl(v,\mathcal{V}^{\operatorname{geo}}_{\le L}\bigr)
    \le
    \sum_{j=L+1}^{m} W_{\operatorname{pt}}(e_j),
\]
where $e_j$ denotes the edge joining level $j-1$ to level $j$ along the ancestral chain.
By \eqref{eq:center_point_weights} and \eqref{eq:funny_constant},
\[
    W_{\operatorname{pt}}(e_j)
    =
    s^j\|r(x_{\alpha|_{j-1}})_k\|_p
    \le
    s^j \kappa_{p,2}(d)\|r(x_{\alpha|_{j-1}})_k\|_2
    \le
    s^j \kappa_{p,2}(d)\lambda_{\max}.
\]
Therefore
$
    d_{\operatorname{pt}}\bigl(v,\mathcal{V}^{\operatorname{geo}}_{\le L}\bigr)
\le
    \kappa_{p,2}(d)\lambda_{\max}
    \sum_{j=L+1}^{\infty}s^j
=
    \kappa_{p,2}(d)\lambda_{\max}\,
    \frac{s^{L+1}}{1-s}
$.
Passing to the completion gives
$
    d_H^{\overline{d}_{\operatorname{pt}}}
    \bigl(
        \mathcal{V}^{\operatorname{geo}}_{\le L},
        \overline{\mathcal{V}^{\operatorname{geo}}}
    \bigr)
    \le
    \kappa_{p,2}(d)\lambda_{\max}\,
    \frac{s^{L+1}}{1-s}
$.  
Since $\mathcal{V}^{\operatorname{geo}}_{\le L}\subseteq \overline{\mathcal{V}^{\operatorname{geo}}}$ we have
\begin{equation}
\label{eq:GH_PL_Gpt}
    d_{GH}(P_L,G_{\operatorname{pt}})
\le
    \kappa_{p,2}(d)\lambda_{\max}\,
    \frac{s^{L+1}}{1-s}
.
\end{equation}
For the cell metric, the same argument as in the proof of \eqref{eq:YL_to_Xcell} yields
\begin{equation}
\label{eq:GH_CL_Gcell}
        d_{GH}(C_L,G_{\operatorname{cell}})
    \le
        \frac{2\lambda_{\max}^2 s^L}{1-s}.
\end{equation}
Lemma~\ref{lem:GH_lower_bounds}; together with \eqref{eq:GH_lower_point} and \eqref{eq:GH_PL_Gpt} imply that
\begin{align*}
    d_{GH}(T_L,G_{\operatorname{pt}})
    & \ge
    d_{GH}(T_L,P_L)-d_{GH}(P_L,G_{\operatorname{pt}})
\\
    &
    \ge
    \max\Big\{
        0,
        \bigl(1-\kappa_{p,2}(d)\lambda_{\max}\bigr)
        \frac{s(1-s^L)}{1-s}
    \Big\}
    -
    \kappa_{p,2}(d)\lambda_{\max}\,
    \frac{s^{L+1}}{1-s}.
\end{align*}
Since
$
    \bigl(1-\kappa_{p,2}(d)\lambda_{\max}\bigr)
    s(1-s^L)
    -
    \kappa_{p,2}(d)\lambda_{\max}\,s^{L+1}
    =
    s\bigl(1-\kappa_{p,2}(d)\lambda_{\max}\bigr)-s^{L+1},
$
and since $d_{GH}(T_L,G_{\operatorname{pt}})\ge 0$, this proves \eqref{eq:GH_lower_point_full}.
Likewise, by \eqref{eq:GH_lower_cell} and \eqref{eq:GH_CL_Gcell},
\begin{align*}
    d_{GH}(T_L,G_{\operatorname{cell}})
    & \ge
    d_{GH}(T_L,C_L)-d_{GH}(C_L,G_{\operatorname{cell}})
    \\
    &\ge
    \max\Big\{
        0,
        \Bigl(1-\frac{2\lambda_{\max}^2}{s}\Bigr)
        \frac{s(1-s^L)}{1-s}
    \Big\}
    -
    \frac{2\lambda_{\max}^2 s^L}{1-s}.
\end{align*}
Since $
    \Bigl(1-\frac{2\lambda_{\max}^2}{s}\Bigr)s(1-s^L)-2\lambda_{\max}^2 s^L
=
    s-2\lambda_{\max}^2-s^{L+1}
$, then the non-negativity of $d_{GH}$, we obtain \eqref{eq:GH_lower_cell_full}.
The last claims follow by sending  $L\uparrow\infty$.
\end{proof}

\subsection{Memorization Lemmata}
\label{s:memorization}
We begin by recalling the following memorization lemma, quoted from \citep[Lemma 20]{kratsios2023small}, which is used in the proof of our main result. This lemma also featured as a core technical tool in proving several representation theorems~\cite{borde2024neural,borde2025neural}.
This result extends the quantitative ``class memorization'' theorem of \citep{vardi2022on} to a quantitative vector-valued interpolation result. It may be contrasted with the VC-bounds of \citep{JLMLR_BartlessHAveyLiawMagrabian_2019_VCBoundsReLUffNN}, which imply that any ReLU feedforward network of depth $D$ that memorizes $N$ inputs in $\mathbb{R}^n$ with $C$ classes must have at least $\Omega\big(\frac{D}{\ln(D)}\big)$ parameters.
We rely on a variant of the \textit{aspect ratio} of~\citep{KrauthgameLeeNaor2004}, considered in~\cite{kratsios2023small}, the former of which quantifies the ratio between the total mass and the smallest mass assigned to any point.  For a finite metric space $(\xxx_n,d_n)$, we define its \textit{aspect ratio} as the ratio of its diameter to its minimum non-zero pairwise distance
\[
    \operatorname{aspect}(\xxx_n,d_n)
    \eqdef 
    \frac{
        \max_{x,\tilde{x}\in \xxx_n}\,
        d_n(x,\tilde{x})
    }{
        \min_{\substack{x,\tilde{x}\in \xxx_n\\ x\neq \tilde{x}}}\,
        d_n(x,\tilde{x})
    }
.
\]

\begin{lemma}[Memory Capacity of Deep ReLU Regressors - {\citep[Lemma 20]{kratsios2023small}}]
\label{lem_memory_capacity_deep_ReLU_regressors}
    Let $n,d,N\in \mathbb{N}_+$, let $f:\mathbb{R}^n \rightarrow \mathbb{R}^d$ be some function, and consider distinct $x_1,\dots,x_N\in \mathbb{R}^n$.  
    There exists a deep ReLU network $\mathcal{NN}:\mathbb{R}^n\rightarrow \mathbb{R}^d$ satisfying
    \[
    \mathcal{NN}(x_i) 
    = 
    f(x_i)
    ,
    \]
    for every $i=1,\dots,N$.  
    Furthermore, we have the following quantitative model complexity estimates:
\begin{align}
\label{eq:finite_depth_mem_width}
        \mathrm{width}(\mathcal{NN}) 
        &= 
        n(N - 1) + \max \{ d, 12\}, \\
\label{eq:finite_depth_mem_depth}
        \mathrm{depth}(\mathcal{NN})
        &=
        \mathcal{O}\left(
        N\left\{
            1+
                \sqrt{N\log(N)}
                    \,
                 \left[
                    1
                        +
                   \frac{\log(2)}{\log(N)}\,
                   \left(
                        C_n
                            +
                     \frac{
                         \log\big(
                            N^2\,
                            \operatorname{aspect}(\xxx_N,\|\cdot\|_2)
                          \big)
                     }{
                        \log(2)
                     }
                   \right)_+
                 \right]
            \right\}
        \right), \\
\label{eq:finite_depth_mem_par}
        \mathrm{par}(\mathcal{NN})
        &=
        \mathcal{O}\left(
            N
            \Big(
            \frac{11}{4} 
            \max\{
                n
            ,
                d
            \}	 
            \,
            N^2
            -
            1
            \Big) 
            \,
            \left\{
                d+
                \sqrt{N\log(N)} 
                    \,
                 \left[
                    1 +
                   \frac{\log(2)}{\log(N)} \,
                   \right. \right. \right. \\
\nonumber
                   &\hspace{10mm} \times
                   \left. \left. \left.
                   \left(
                        C_n
                            +
                     \frac{
                         \log\big(
                            N^2\,
                            \operatorname{aspect}(\xxx_N,\|\cdot\|_2)
                          \big)
                     }{
                        \log(2)
                     }
                   \right)_+
             \right]
             \,
                \max\{d,12\}
                \Big[1 + \max\{d,12\}\Big]
        \right\}
    \right)
,
\end{align}
with dimensional constant  $
        C_n
            \eqdef 
    \frac{
                2\log(5 \sqrt{2\pi})
            + 
                \frac{3}{2}
                \log(n)
            -
                \frac1{2}\log(n+1)
        }{
            2\log(2)
        }
>
    0
$
.
\end{lemma}

We apply the above to the parent-child relationship determined by the CFG function inducing the vertices in the combinatorial tree.  Critically, we observe that if \textit{each} of the centres of the Voronoi cells at a given level is exactly interpolated, then the corresponding Voronoi cells at that level are also automatically interpolated.  This point is essential since, as implicitly suggested by Proposition~\ref{prop:treebuilder_bounds}, there is no obvious continuity principle from centres to cells; only the converse direction is evident.  In other words, exact interpolation of cell centres is key, whereas mere approximation of the cell centres need not suffice in general, since the cells themselves may change discontinuously in the Hausdorff metric under arbitrarily small perturbations of their centres.

\begin{proposition}[Finite-window memorization complexity for the CFG]
\label{prop:finite_depth_memorization_rhg}
Fix $x_0\in \mathbb{R}^d$, $0<s<1$, $0<\lambda_{\min}\le \lambda_{\max}\le 1$, $K>1$, and levels
$
0\le L_-<L_+
$.
Let
$
\mathcal{A}_{[L_-,L_+-1]}
    \eqdef
    \bigcup_{l=L_-}^{L_+-1}[K]_+^l
$
be the set of addresses between depths $L_-$ and $L_+-1$, and define the set of \emph{distinct} parent centres
\[
        \mathcal{X}_{L_-,L_+}
    \eqdef
        \{x_\alpha:\alpha\in \mathcal{A}_{[L_-,L_+-1]}\}
.
\]
There exists a ReLU MLP
$
    \hat r_{L_-,L_+}:\mathbb{R}^d\to \mathbb{R}^{d\times K}
$
such that: for each $\alpha\in \mathcal{A}_{[L_-,L_+-1]}$
\[
    \hat r_{L_-,L_+}(x_\alpha)=r(x_\alpha)
.
\]
Moreover%
\footnote{Identifying $\mathbb{R}^{d\times K}\cong \mathbb{R}^{dK}$ by vectorization.}%
, define the ``inter-level separation'' by $
        \delta_{L_-,L_+}
    \eqdef
        \min_{\substack{x,\tilde{x}\in \mathcal{X}_{L_-,L_+}\\ x\neq \tilde{x}}}
        \|x-\tilde{x}\|_2
$
and $
    M_{L_-,L_+}
    \eqdef
    \sum_{l=L_-}^{L_+-1}K^l
$, $
    B_{L_-,L_+}
    \eqdef
    \frac{2\lambda_{\max}s(1-s^{L_+-1})}{(1-s)\delta_{L_-,L_+}}
$
then, we have the following complexity bounds
\begin{align}
\mathrm{width}(\hat r_{L_-,L_+})
& \in
\mathcal{O}(K^{L_+}),
\label{eq:cfg_mem_width_bound}
\\
\mathrm{depth}(\hat r_{L_-,L_+})
& \in
\mathcal{O}\!\left(
    K^{3L_+/2}\sqrt{L_+\log(K)}
    \left[
        1+
        \frac{
            \log_+\!\big(B_{L_-,L_+}\big)
        }{
            L_+\log(K)
        }
    \right]
\right),
\label{eq:cfg_mem_depth_bound}
\\
\mathrm{par}(\hat r_{L_-,L_+})
& \in
\mathcal{O}\!\left(
    K^{7L_+/2+1}\sqrt{L_+\log(K)}
    \left[
        1+
        \frac{
            \log_+\!\big(B_{L_-,L_+}\big)
        }{
            L_+\log(K)
        }
    \right]
\right).
\label{eq:cfg_mem_par_bound}
\end{align}
\end{proposition}

\begin{proof}[{Proof of Proposition~\ref{prop:finite_depth_memorization_rhg}}]
We use the following abbreviations:
$
    N_{L_-,L_+}
    \eqdef
    |\mathcal{X}_{L_-,L_+}|
$, $
    A_{L_-,L_+}
    \eqdef
    \operatorname{aspect}(\mathcal{X}_{L_-,L_+},\|\cdot\|_2)
$. 
Since \citep[Lemma 20]{kratsios2023small} is stated for distinct inputs, we first remove duplicates from the collection of parent centres with depths between $L_-$ and $L_+-1$. Enumerate the distinct set $\mathcal{X}_{L_-,L_+}$ as
$
    \mathcal{X}_{L_-,L_+}=\{z_1,\dots,z_{N_{L_-,L_+}}\}
$.
Identify $\mathbb{R}^{d\times K}$ with $\mathbb{R}^{dK}$ via any fixed linear isomorphism, and define
$
    f_{L_-,L_+}:\mathbb{R}^d\to \mathbb{R}^{dK}
$
on the training set by
\[
    f_{L_-,L_+}(z_i)\eqdef r(z_i),
    \qquad i=1,\dots,N_{L_-,L_+}.
\]
Applying Lemma~20 with input dimension $n=d$, output dimension $dK$, sample size $N_{L_-,L_+}$, and training set $\mathcal{X}_{L_-,L_+}\subseteq \mathbb{R}^d$, yields a ReLU network
$
    \hat r_{L_-,L_+}:\mathbb{R}^d\to \mathbb{R}^{dK}\cong \mathbb{R}^{d\times K}
$
such that
$
    \hat r_{L_-,L_+}(z_i)=r(z_i)
$
for every $i=1,\dots,N_{L_-,L_+}$.
Since every $x_\alpha$ with $L_-\le |\alpha|\le L_+-1$ belongs to $\mathcal{X}_{L_-,L_+}$, this proves
$
    \hat r_{L_-,L_+}(x_\alpha)=r(x_\alpha)
$
for all $\alpha\in \mathcal{A}_{[L_-,L_+-1]}$.

It remains only to estimate $N_{L_-,L_+}$ and $A_{L_-,L_+}$.
Note that there are at most $K^l$ addresses at level $l$, whence
$
    N_{L_-,L_+}
    =
    |\mathcal{X}_{L_-,L_+}|
    \le
    \sum_{l=L_-}^{L_+-1}K^l
$; thus,
\begin{equation}
\label{eq:finite_depth_mem_count}
    N_{L_-,L_+}
    \le
    \sum_{l=L_-}^{L_+-1}K^l
    =
    M_{L_-,L_+}
.
\end{equation}
To bound the aspect ratio, first observe that if $x_\alpha\in \mathcal{X}_{L_-,L_+}$ and $|\alpha|=l\le L_+-1$, then by the recursive definition of the centres,
\[
    x_\alpha-x_0
    =
    \sum_{j=1}^{l}
    s^j r(x_{\alpha|_{j-1}})_{\alpha_j},
\]
where $\alpha|_{j-1}$ denotes the prefix of $\alpha$ of length $j-1$. Hence,
\[
    \|x_\alpha-x_0\|_2
    \le
    \sum_{j=1}^{l}
    s^j
    \|r(x_{\alpha|_{j-1}})_{\alpha_j}\|_2
    \le
    \lambda_{\max}
    \sum_{j=1}^{l}s^j
    \le
    \lambda_{\max}
    \sum_{j=1}^{L_+-1}s^j.
\]
Therefore, for any $x,\tilde{x}\in \mathcal{X}_{L_-,L_+}$,
$
    \|x-\tilde{x}\|_2
\le
    \|x-x_0\|_2+\|\tilde{x}-x_0\|_2
\le
    2\lambda_{\max}
    \sum_{j=1}^{L_+-1}s^j
$.
Thus,
$
    \operatorname{diam}(\mathcal{X}_{L_-,L_+},\|\cdot\|_2)
    \le
    2\lambda_{\max}
    \sum_{j=1}^{L_+-1}s^j
$.
By definition of $\delta_{L_-,L_+}$ and of the aspect ratio,
$
    A_{L_-,L_+}
=
    \frac{
        \operatorname{diam}(\mathcal{X}_{L_-,L_+},\|\cdot\|_2)
    }{
        \delta_{L_-,L_+}
    }
\le
    \frac{
        2\lambda_{\max}
        \sum_{j=1}^{L_+-1}s^j
    }{
        \delta_{L_-,L_+}
    }
$, 
which proves
\begin{equation}
\label{eq:finite_depth_mem_aspect_bound}
    A_{L_-,L_+}
    \le
    \frac{
        2\lambda_{\max}
        \sum_{j=1}^{L_+-1}s^j
    }{
        \delta_{L_-,L_+}
    }
    =
    \frac{
        2\lambda_{\max}s(1-s^{L_+-1})
    }{
        (1-s)\delta_{L_-,L_+}
    }
    =
    B_{L_-,L_+}
.
\end{equation}
Since the right-hand sides of \citep[Lemma 20]{kratsios2023small} are increasing in the sample size and aspect ratio, \eqref{eq:finite_depth_mem_count} and \eqref{eq:finite_depth_mem_aspect_bound} yield
\begin{equation}
\begin{aligned}
    \mathrm{width}(\hat r_{L_-,L_+})
    & \le
    d(M_{L_-,L_+}-1)+\max\{dK,12\}
    =
    \mathcal{O}\bigl(dM_{L_-,L_+}+\max\{dK,12\}\bigr),
\\
    \mathrm{depth}(\hat r_{L_-,L_+})
    & \in
    \mathcal{O}\!\left(
        M_{L_-,L_+}
        \left\{
            1+
            \sqrt{M_{L_-,L_+}\log(M_{L_-,L_+})}
            \left[
                1+
                \frac{\log(2)}{\log(M_{L_-,L_+})}
                \left(
                    C_d+
                    \frac{\log\big(M_{L_-,L_+}^2B_{L_-,L_+}\big)}{\log(2)}
                \right)_+
            \right]
        \right\}
    \right),
\\
    \mathrm{par}(\hat r_{L_-,L_+})
    & \in
    \mathcal{O}\!\left(
        M_{L_-,L_+}
        \Big(
            \tfrac{11}{4}\max\{d,dK\}\,M_{L_-,L_+}^2-1
        \Big)
\right.
\\
\nonumber
& \times 
        \left\{
            dK
            +
            \sqrt{M_{L_-,L_+}\log(M_{L_-,L_+})}
            \left[
                1+
                \frac{\log(2)}{\log(M_{L_-,L_+})}
                \left(
                    C_d+
                    \frac{\log\big(M_{L_-,L_+}^2B_{L_-,L_+}\big)}{\log(2)}
                \right)_+
            \right]
\right.
\\
\nonumber
& \times 
\left.\left.
            \max\{dK,12\}\bigl(1+\max\{dK,12\}\bigr)
        \right\}
    \right).
\end{aligned}
\end{equation}
These are the raw bounds obtained by substituting
\eqref{eq:finite_depth_mem_count} and
\eqref{eq:finite_depth_mem_aspect_bound} into
Lemma~\ref{lem_memory_capacity_deep_ReLU_regressors}.
Finally, if $K>1$, then
$
    M_{L_-,L_+}
    \le
    \sum_{l=0}^{L_+-1}K^l
    =
    \frac{K^{L_+}-1}{K-1}
$,
so, suppressing the dependence on $d$, we obtain
\[
\begin{aligned}
    \mathrm{width}(\hat r_{L_-,L_+})
    & \in
    \mathcal{O}(K^{L_+}),
\\
\hfill
    \mathrm{depth}(\hat r_{L_-,L_+})
    & \in
    \mathcal{O}\!\left(
        K^{3L_+/2}\sqrt{L_+\log(K)}
        \left[
            1+
            \frac{
                \log_+\!\big(B_{L_-,L_+}\big)
            }{
                L_+\log(K)
            }
        \right]
    \right),
\\
\hfill
    \mathrm{par}(\hat r_{L_-,L_+})
    & \in
    \mathcal{O}\!\left(
        K^{7L_+/2+1}\sqrt{L_+\log(K)}
        \left[
            1+
            \frac{
                \log_+\!\big(B_{L_-,L_+}\big)
            }{
                L_+\log(K)
            }
        \right]
    \right).
\end{aligned}
\]
These are precisely
\eqref{eq:cfg_mem_width_bound}--\eqref{eq:cfg_mem_par_bound},
which proves the proposition.
\end{proof}

\begin{proof}[Proof of Theorem~\ref{thrm:main_result}]
By the definition of \(L_\varepsilon\),
\[
\frac{2\lambda_{\max}^2s^{L_\varepsilon}}{1-s}
\le
\varepsilon .
\]
Apply Proposition~\ref{prop:finite_depth_memorization_rhg} with
\(L_-=0\) and \(L_+=L_\varepsilon\). This gives a ReLU MLP
\(\hat r:\mathbb R^d\to\mathbb R^{d\times K}\) such that
\[
\hat r(x_\alpha)=r(x_\alpha),
\qquad
|\alpha|\le L_\varepsilon-1 .
\]
We claim that the depth-\(L_\varepsilon\) rollout generated by \(\hat r\)
coincides with the true CFG truncation. The claim follows by induction over
levels. The root centre and root cell are identical. If
\(\hat x_\alpha=x_\alpha\) and \(\hat C_\alpha=C_\alpha\), then for each
\(k\in[K]_+\),
\[
\hat x_{\alpha k}
=
\hat x_\alpha+s^{|\alpha|+1}\hat r(\hat x_\alpha)_k
=
x_\alpha+s^{|\alpha|+1}r(x_\alpha)_k
=
x_{\alpha k}.
\]
Given the same parent cell and the same child centres, the Voronoi refinement
rule also gives \(\hat C_{\alpha k}=C_{\alpha k}\). Hence the neural rollout
coincides with the true truncation \(Y_{L_\varepsilon}\).

Proposition~\ref{prop:GH_tree_vs_cell} then yields
\[
d_H^{\overline d_{\operatorname{cell}}}
\bigl(
Y_{L_\varepsilon},X_{\operatorname{cell}}
\bigr)
\le
\frac{2\lambda_{\max}^2s^{L_\varepsilon}}{1-s}
\le
\varepsilon .
\]
Finally, since \(L_\varepsilon=\mathcal{O}(\log(1/\varepsilon))\), substituting
\(L_+=L_\varepsilon\) in
Proposition~\ref{prop:finite_depth_memorization_rhg} gives the stated width
and depth bounds.
\end{proof}

We are now ready to prove the more transparent case of our memorization gadget, whereby we may ensure a polynomial scaling of the memorizing network parameters; if we only require that the network memorizes a single ``next'' level and if the CFG satisfies a rudimentary level of regularity.

\begin{corollary}[Single-level memorization with a Polynomial-Sized Network]
\label{cor:single_level_memorization}
Assume the setting of Proposition~\ref{prop:finite_depth_memorization_rhg}.  
Assume moreover that $r$ is a CFG constructed as in
Proposition~\ref{prop:param_uniformly_separated_annular_configs}, and that
the maps $A_1$, $A_2$, and $\sigma$ satisfy the hypotheses of
Proposition~\ref{prop:lipschitz_parametrized_rhg}.

Let
\[
    c_{\operatorname{bd}}
    \eqdef
    \sigma_{\max}\lambda_+ ,
\]
which is the resulting upper bound on the CFG residuals in this parametrized
setting.

For any $l\in\mathbb{N}_+$, there is a ReLU MLP
\[
    \hat r_{l,l+1}:\mathbb{R}^d\to \mathbb{R}^{d\times K}
\]
satisfying: for each $\alpha\in [K]_+^l$,
\[
    \hat r_{l,l+1}(x_\alpha)
=
    r(x_\alpha)
.
\]
In particular, if
\[
    1-s-s^2L_r>0
    \qquad\text{and}\qquad
    \frac{2c_{\operatorname{bd}}s}{1-s-s^2L_r}
    \le
    \sigma_{\min}\varepsilon,
\]
then
\begin{align*}
\mathrm{width}(\hat r_{l,l+1})
& \in
\mathcal{O}(K^l),
\\
\mathrm{depth}(\hat r_{l,l+1})
& \in
\mathcal{O}\!\left(
    K^{3l/2}\sqrt{l\log(K)}
    \left[
        1+
        \frac{\log(1/s)}{\log(K)}
    \right]
\right),
\\
\mathrm{par}(\hat r_{l,l+1})
& \in
\mathcal{O}\!\left(
    K^{7l/2+1}\sqrt{l\log(K)}
    \left[
        1+
        \frac{\log(1/s)}{\log(K)}
    \right]
\right).
\end{align*}
Moreover, for fixed $L_r\ge 0$ and $c_{\operatorname{bd}}>0$, there exists
an $s>0$ small enough so that the conditions
\[
    1-s-s^2L_r>0
    \qquad\text{and}\qquad
    \frac{2c_{\operatorname{bd}}s}{1-s-s^2L_r}
    \le
    \sigma_{\min}\varepsilon
\]
hold.
\end{corollary}

\begin{proof}[{Proof of Corollary~\ref{cor:single_level_memorization}}]
Set
\[
    c_{\operatorname{sep}}
    \eqdef
    \frac{2c_{\operatorname{bd}}s}{1-s-s^2L_r}.
\]
If
\[
    1-s-s^2L_r>0
    \qquad\text{and}\qquad
    \frac{2c_{\operatorname{bd}}s}{1-s-s^2L_r}
    \le
    \sigma_{\min}\varepsilon,
\]
then
\[
    0<c_{\operatorname{sep}}\le \sigma_{\min}\varepsilon .
\]
We now show by induction that, for every $l\in\mathbb{N}_+$,
\begin{equation}
\label{eq:induction_delta_scale}
    \delta_l\ge c_{\operatorname{sep}}s^l .
\end{equation}
By Proposition~\ref{prop:key_separation},
\[
    \delta_1\ge s\,\sigma_{\min}\varepsilon
    \ge
    c_{\operatorname{sep}}s,
\]
so \eqref{eq:induction_delta_scale} holds for $l=1$.

Assume now that
\[
    \delta_l\ge c_{\operatorname{sep}}s^l
\]
for some $l\in\mathbb{N}_+$. Then Proposition~\ref{prop:key_separation}
gives
\begin{align*}
    \delta_{l+1}
    &\ge
    \max\left\{
        0,\,
        \min\left\{
            s^{l+1}\sigma_{\min}\varepsilon,
            \,
            \left(
                1-s^{l+1}L_r
            \right)\delta_l
            -
            2c_{\operatorname{bd}}s^{l+1}
        \right\}
    \right\}
\\
    &\ge
    \max\left\{
        0,\,
        \min\left\{
            s^{l+1}\sigma_{\min}\varepsilon,
            \,
            \left(
                1-s^{l+1}L_r
            \right)c_{\operatorname{sep}}s^l
            -
            2c_{\operatorname{bd}}s^{l+1}
        \right\}
    \right\}
\\
    &=
    \max\left\{
        0,\,
        \min\left\{
            s^{l+1}\sigma_{\min}\varepsilon,
            \,
            s^l
            \big(
                (1-s^{l+1}L_r)c_{\operatorname{sep}}
                -
                2c_{\operatorname{bd}}s
            \big)
        \right\}
    \right\}.
\end{align*}
Since $s^{l+1}\le s^2$ for every $l\ge 1$, we have
\[
    (1-s^{l+1}L_r)c_{\operatorname{sep}}
    -
    2c_{\operatorname{bd}}s
    \ge
    (1-s^2L_r)c_{\operatorname{sep}}
    -
    2c_{\operatorname{bd}}s .
\]
Using the definition of $c_{\operatorname{sep}}$,
\[
    (1-s^2L_r)c_{\operatorname{sep}}
    -
    2c_{\operatorname{bd}}s
    =
    \left(
        1-s^2L_r
        -
        \frac{2c_{\operatorname{bd}}s}{c_{\operatorname{sep}}}
    \right)c_{\operatorname{sep}}
    =
    s\,c_{\operatorname{sep}} .
\]
Therefore
\[
    \delta_{l+1}
    \ge
    \max\left\{
        0,\,
        \min\left\{
            s^{l+1}\sigma_{\min}\varepsilon,
            \,
            s^{l+1}c_{\operatorname{sep}}
        \right\}
    \right\}
    =
    s^{l+1}c_{\operatorname{sep}},
\]
since $c_{\operatorname{sep}}\le \sigma_{\min}\varepsilon$. This proves
\eqref{eq:induction_delta_scale}.

Next, the non-vacuousness claim follows since, as $s\downarrow 0$,
\[
    \lim_{s\downarrow 0} \big(1-s-s^2L_r\big)=1
    \qquad\text{and}\qquad
    \lim_{s\downarrow 0}
    \frac{2c_{\operatorname{bd}}s}{1-s-s^2L_r}
    =
    0 .
\]
Hence, for all sufficiently small $s>0$, one has both
\[
    1-s-s^2L_r>0
    \qquad\text{and}\qquad
    \frac{2c_{\operatorname{bd}}s}{1-s-s^2L_r}
    \le
    \sigma_{\min}\varepsilon .
\]

Set $(L_-,L_+)=(l,l+1)$ in
Proposition~\ref{prop:finite_depth_memorization_rhg}. Then
\[
    \mathcal{A}_{[L_-,L_+-1]}
    =
    [K]_+^l,
\]
so
\[
    \mathcal{X}_{L_-,L_+}
    =
    \{x_\alpha:\ \alpha\in [K]_+^l\}
    =
    \mathcal{X}_l .
\]
Hence
\[
    \delta_{L_-,L_+}=\delta_l,
    \qquad
    M_{L_-,L_+}
    =
    \sum_{j=l}^{l}K^j
    =
    K^l .
\]
Since the parametrized CFG has upper residual bound
$c_{\operatorname{bd}}$, we obtain
\[
    B_{L_-,L_+}
    =
    \frac{2c_{\operatorname{bd}}s(1-s^l)}{(1-s)\delta_l}.
\]
The existence of $\hat r_{l,l+1}$ follows directly from
Proposition~\ref{prop:finite_depth_memorization_rhg}. Moreover,
\eqref{eq:induction_delta_scale} gives
\[
    \delta_l\ge c_{\operatorname{sep}}s^l,
\]
and therefore
\[
B_{l,l+1}
=
\frac{2c_{\operatorname{bd}}s(1-s^l)}{(1-s)\delta_l}
\le
\frac{2c_{\operatorname{bd}}}{(1-s)c_{\operatorname{sep}}}
\,s^{1-l},
\]
since $1-s^l\le 1$. Hence
\[
\log_+\!\big(B_{l,l+1}\big)
\le
\log_+\!\left(
    \frac{2c_{\operatorname{bd}}}{(1-s)c_{\operatorname{sep}}}
\right)
+
(l-1)\log(1/s)
=
\mathcal{O}\bigl(l\log(1/s)+1\bigr).
\]
Substituting this into the preceding bounds from
Proposition~\ref{prop:finite_depth_memorization_rhg} yields
\begin{align*}
\mathrm{depth}(\hat r_{l,l+1})
& \in
\mathcal{O}\!\left(
    K^{3l/2}\sqrt{l\log(K)}
    \left[
        1+
        \frac{\log(1/s)}{\log(K)}
    \right]
\right),
\\
\mathrm{par}(\hat r_{l,l+1})
& \in
\mathcal{O}\!\left(
    K^{7l/2+1}\sqrt{l\log(K)}
    \left[
        1+
        \frac{\log(1/s)}{\log(K)}
    \right]
\right),
\end{align*}
while the width bound remains unchanged.
\end{proof}
\section{Proofs for CFG Parametrization}
\label{s:DetsExperiments}

\subsection{Orthogonal-Diagonal Parametrizations of Uniformly Separated Annular Configurations}
\label{s:orth_diag_param_sep_annuli}

In this subsection, we construct a flexible class of maps
\[
\begin{aligned}
r:\mathbb{R}^d\to (\mathbb{R}^d)^K
\end{aligned}
\]
where we write $r(x)\eqdef (r_1(x)|\cdots|r_K(x))$, 
whose outputs are uniformly annular and uniformly separated.  The key idea is to begin with a fixed reference configuration in $\mathbb{R}^d$, and then deform it by pointwise orthogonal-diagonal-orthogonal linear maps.  The orthogonal factors preserve Euclidean geometry, while the diagonal factor controls the metric distortion explicitly through its singular values.
We begin with the elementary observation that exponentials of skew-symmetric matrices are orthogonal.

\begin{lemma}[Exponentials of skew-symmetric matrices are orthogonal]
\label{lem:exp_skew_orthogonal}
Fix $d\in \mathbb{N}_+$ and let $A\in \mathbb{R}^{d\times d}$ satisfy $
A^\top=-A. $
Then $
e^A\in SO(d).
$
\end{lemma}
\begin{proof}
Since $A^\top=-A$, one has
$
(e^A)^\top=e^{A^\top}=e^{-A}
$.  Therefore $(e^A)^\top e^A=e^{-A}e^A=I_d$. 
Thus $e^A$ is orthogonal.  Clearly $\operatorname{det}(e^A)>0$, so $e^A\in SO(d)$.
\end{proof}

The next result records the precise metric distortion induced by an orthogonal-diagonal-orthogonal factorization.

\begin{lemma}[Singular-value distortion bounds]
\label{lem:sv_distortion_bounds}
Fix $d\in \mathbb{N}_+$ and let
\[
M=V\Sigma W^\top\in \mathbb{R}^{d\times d}
\]
for orthogonal $V,W\in O(d)$
where $
\Sigma=\operatorname{diag}(\sigma_1,\dots,\sigma_d)
$
for numbers satisfying: for each $\forall \ell\in[d]$
$
0<\sigma_{\min}
\le
\sigma_\ell
\le
\sigma_{\max}
<\infty
$.
Then, for every $u\in \mathbb{R}^d$,
\[
\sigma_{\min}|u|
\le
|Mu|
\le
\sigma_{\max}|u|.
\]
Consequently, for every $u,v\in \mathbb{R}^d$,
$
\sigma_{\min}|u-v|
\le
|Mu-Mv|
\le
\sigma_{\max}|u-v|
$.
\end{lemma}

\begin{proof}
Fix $u\in \mathbb{R}^d$. Since $V$ is orthogonal,
$
|Mu|
=
|V\Sigma W^\top u|
=
|\Sigma W^\top u|
$.
Set
$
w\eqdef W^\top u
$.  
Since $W$ is orthogonal, $|w|=|u|$. Therefore
$
|Mu|^2
=
|\Sigma w|^2
=
\sum_{\ell=1}^d \sigma_\ell^2 w_\ell^2
$.
Using
$
\sigma_{\min}^2\le \sigma_\ell^2\le \sigma_{\max}^2
$
for all $
\forall \ell\in[d]$,
we obtain
$
\sigma_{\min}^2\sum_{\ell=1}^d w_\ell^2
\le
|Mu|^2
\le
\sigma_{\max}^2\sum_{\ell=1}^d w_\ell^2
$.
Since $\sum_{\ell=1}^d w_\ell^2=|w|^2=|u|^2$, it follows that
$
\sigma_{\min}^2|u|^2
\le
|Mu|^2
\le
\sigma_{\max}^2|u|^2
$.
Taking square-roots proves the first claim.
Applying the first claim with $u-v$ in place of $u$ gives
$
\sigma_{\min}|u-v|
\le
|M(u-v)|
\le
\sigma_{\max}|u-v|
$.
Since $M(u-v)=Mu-Mv$, the second claim follows.
\end{proof}

The next lemma isolates the corresponding bounds for the columns of a matrix.

\begin{lemma}[Column norm and separation bounds]
\label{lem:column_norm_sep_bounds}
Fix $d\in \mathbb{N}_+$, 
$
M=V\Sigma W^\top\in \mathbb{R}^{d\times d}$, and $
V,W\in O(d)$,
where $
\Sigma=\operatorname{diag}(\sigma_1,\dots,\sigma_d)$
such that: for every $\ell\in[d]_+$
\[
0<\sigma_{\min}
\le
\sigma_\ell
\le
\sigma_{\max}
<\infty
\]
Denote by $m_1,\dots,m_d\in \mathbb{R}^d$ the columns of $M$. Then:
\begin{enumerate}
    \item for every $k\in[d]$,
    \[
    \sigma_{\min}\le |m_k|\le \sigma_{\max};
    \]
    \item for every distinct $i,j\in[d]$,
    \[
    \sqrt{2}\,\sigma_{\min}
    \le
    |m_i-m_j|
    \le
    \sqrt{2}\,\sigma_{\max}.
    \]
\end{enumerate}
\end{lemma}

\begin{proof}
Since $m_k=Me_k$, the first claim follows by applying Lemma~\ref{lem:sv_distortion_bounds} with $u=e_k$ and using $|e_k|=1$.
Similarly, for distinct $i,j\in[d]$,
$
m_i-m_j=M(e_i-e_j)
$.
Applying Lemma~\ref{lem:sv_distortion_bounds} to $u=e_i$ and $v=e_j$, and using
$
|e_i-e_j|=\sqrt{2}
$, 
yields
$
    \sqrt{2}\,\sigma_{\min}
=
    \sigma_{\min}|e_i-e_j|
\le
    |m_i-m_j|
\le
    \sigma_{\max}|e_i-e_j|
=
    \sqrt{2}\,\sigma_{\max}
$.
This proves the claim.
\end{proof}

We next record a convenient sufficient condition for constructing a fixed reference configuration with prescribed pairwise separation.

\begin{definition}[Reference packing]
\label{def:reference_packing}
Fix $d,K\in\mathbb{N}_+$, $\lambda_-,\lambda_+,\varepsilon>0$, and let
\[
C=(c_1|\cdots|c_K)\in \mathbb{R}^{d\times K}.
\]
We say that $C$ is a \emph{$(\lambda_-,\lambda_+,\varepsilon)$-reference packing} if
\begin{align}
\lambda_-
\le
|c_k|
\le
\lambda_+
\qquad
&\forall k\in[K],
\label{eq:reference_packing_norms}
\\
|c_i-c_j|
\ge
\varepsilon
\qquad
&\forall i,j\in[K]
\text{ with }i\neq j.
\label{eq:reference_packing_sep}
\end{align}
\end{definition}

The next proposition is the main structural result: orthogonal-diagonal deformations of a reference packing preserve annular control and pairwise separation up to the extremal singular values.

\begin{proposition}[Orthogonal-diagonal deformation of a reference packing]
\label{prop:orth_diag_deformation_reference_packing}
Fix $d,K\in \mathbb{N}_+$ and constants
$
0<\sigma_{\min}\le \sigma_{\max}<\infty,
\qquad
0<\lambda_-\le \lambda_+<\infty,
\qquad
\varepsilon>0.
$
Let
$
C=(c_1|\cdots|c_K)\in \mathbb{R}^{d\times K}
$
be a $(\lambda_-,\lambda_+,\varepsilon)$-reference packing.

For each $x\in \mathbb{R}^d$, let
$
A_1(x),A_2(x)\in \mathbb{R}^{d\times d}
$
be skew-symmetric matrices, and let
$
\Sigma(x)
=
\operatorname{diag}\big(\sigma(x)\big)
$
satisfy:
$
\sigma_{\min}
\le
\sigma_\ell(x)
\le
\sigma_{\max}
$
for every $\forall \ell\in[d]$.  
Define
$
V(x)\eqdef e^{A_1(x)}$, $
W(x)\eqdef e^{A_2(x)}$,
$M(x)\eqdef V(x)\Sigma(x)W(x)^\top
$,
and, for each $k\in[K]$, set
$
r_k(x)\eqdef M(x)c_k\in \mathbb{R}^d.
$
Then the map
\[
\begin{aligned}
r:\mathbb{R}^d& \to (\mathbb{R}^d)^K,\\
\qquad
x& \mapsto \big(r_1(x),\dots,r_K(x)\big),
\end{aligned}
\]
satisfies the following properties for every $x\in \mathbb{R}^d$:
\begin{enumerate}
    \item for every $k\in[K]$,
    $
    \sigma_{\min}\lambda_-
    \le
    |r_k(x)|
    \le
    \sigma_{\max}\lambda_+;
    $
    \item for every distinct $i,j\in[K]$,
    $
    |r_i(x)-r_j(x)|
    \ge
    \sigma_{\min}\varepsilon.
    $
\end{enumerate}
In particular,
$
r(\mathbb{R}^d)
\subseteq
\Big(
A(\sigma_{\min}\lambda_-,\sigma_{\max}\lambda_+)
\Big)^K,
$
where
$
A(a,b)
\eqdef
\left\{
z\in \mathbb{R}^d:\ a\le |z|\le b
\right\},
$
and the output configurations are uniformly separated:
$
\min_{i\neq j}|r_i(x)-r_j(x)|
\ge
\sigma_{\min}\varepsilon
\qquad
\forall x\in \mathbb{R}^d.
$
\end{proposition}

\begin{proof}
Fix $x\in \mathbb{R}^d$. By Lemma~\ref{lem:exp_skew_orthogonal}, both
$
V(x)=e^{A_1(x)}
\qquad\text{and}\qquad
W(x)=e^{A_2(x)}
$
are orthogonal. Hence $M(x)$ has the orthogonal-diagonal-orthogonal factorization
$
M(x)=V(x)\Sigma(x)W(x)^\top
$
with diagonal factor satisfying
$
\sigma_{\min}
\le
\sigma_\ell(x)
\le
\sigma_{\max}
\qquad
\forall \ell\in[d].
$
Therefore Lemma~\ref{lem:sv_distortion_bounds} applies to $M(x)$.

Fix $k\in[K]$. Since
$
r_k(x)=M(x)c_k,
$
Lemma~\ref{lem:sv_distortion_bounds} yields
$
\sigma_{\min}|c_k|
\le
|r_k(x)|
\le
\sigma_{\max}|c_k|.
$
Using \eqref{eq:reference_packing_norms}, we obtain
$
\sigma_{\min}\lambda_-
\le
|r_k(x)|
\le
\sigma_{\max}\lambda_+.
$
This proves the annular bound.
Next, fix distinct $i,j\in[K]$. Since
$
r_i(x)-r_j(x)=M(x)(c_i-c_j),
$
Lemma~\ref{lem:sv_distortion_bounds} again gives
$
|r_i(x)-r_j(x)|
\ge
\sigma_{\min}|c_i-c_j|.
$
Using \eqref{eq:reference_packing_sep}, we conclude that
$
|r_i(x)-r_j(x)|
\ge
\sigma_{\min}\varepsilon.
$
This proves the separation estimate.
Since $x\in \mathbb{R}^d$ was arbitrary, the proof is complete.
\end{proof}

The previous proposition is the final form needed for our applications.  It shows that one may enforce the desired geometry of the output configurations by choosing a single fixed reference packing and then composing with pointwise orthogonal-diagonal deformations.
We record this consequence explicitly in Proposition~\ref{prop:param_uniformly_separated_annular_configs}.
\begin{proof}[Proof of Proposition~\ref{prop:param_uniformly_separated_annular_configs}]
This is immediate from Proposition~\ref{prop:orth_diag_deformation_reference_packing}.
\end{proof}

\begin{proof}[{Proof of Proposition~\ref{prop:lipschitz_parametrized_rhg}}]
Fix $x,\tilde{x}\in \mathbb{B}_2^d$. Write
\[
V(x)\eqdef e^{A_1(x)},
\qquad
W(x)\eqdef e^{A_2(x)},
\qquad
\Sigma(x)\eqdef \operatorname{diag}(\sigma(x)).
\]
Then $
r(x)=V(x)\Sigma(x)W(x)^\top C.
$
Since $A_1(x)$ and $A_2(x)$ are skew-symmetric, Lemma~\ref{lem:exp_skew_orthogonal} yields
$
V(x),W(x)\in O(d)
$ for each $x\in \mathbb{B}_2^d$.
In particular,
$
\|V(x)\|_{\operatorname{op}}
=
\|W(x)\|_{\operatorname{op}}
=
1
$
We first claim that, for any skew-symmetric matrices $S,T\in \mathbb{R}^{d\times d}$,
\begin{equation}
\label{eq:exp_difference_skew_bound}
\|e^S-e^T\|_{\operatorname{op}}
\le
\|S-T\|_{\operatorname{op}}.
\end{equation}
Indeed, using the standard identity
\[
e^S-e^T
=
\int_0^1
e^{(1-u)S}(S-T)e^{uT}
\,du,
\]
we obtain
$
\|e^S-e^T\|_{\operatorname{op}}
\le
\int_0^1
\|e^{(1-u)S}\|_{\operatorname{op}}
\,
\|S-T\|_{\operatorname{op}}
\,
\|e^{uT}\|_{\operatorname{op}}
\,du
$.  
Since $S$ and $T$ are skew-symmetric, $e^{(1-u)S}$ and $e^{uT}$ are orthogonal, whence each operator norm equals $1$. Thus
\[
\|e^S-e^T\|_{\operatorname{op}}
\le
\int_0^1
\|S-T\|_{\operatorname{op}}
\,du
=
\|S-T\|_{\operatorname{op}},
\]
which proves \eqref{eq:exp_difference_skew_bound}.
We now estimate $r(x)-r(\tilde{x})$. By adding and subtracting intermediate terms,
\begin{align*}
r(x)-r(\tilde{x})
&=
V(x)\Sigma(x)W(x)^\top C
-
V(\tilde{x})\Sigma(\tilde{x})W(\tilde{x})^\top C
\\
&=
\big(V(x)-V(\tilde{x})\big)\Sigma(x)W(x)^\top C
\\
&\qquad
+
V(\tilde{x})\big(\Sigma(x)-\Sigma(\tilde{x})\big)W(x)^\top C
\\
&\qquad
+
V(\tilde{x})\Sigma(\tilde{x})\big(W(x)^\top-W(\tilde{x})^\top\big)C.
\end{align*}
Fix $k\in[K]_+$. Taking the $k$th column and using $\|c_k\|_2\le \lambda_+$ gives
\begin{align*}
\|r_k(x)-r_k(\tilde{x})\|_2
&\le
\|V(x)-V(\tilde{x})\|_{\operatorname{op}}
\|\Sigma(x)\|_{\operatorname{op}}
\|W(x)\|_{\operatorname{op}}
\|c_k\|_2
\\
&\qquad
+
\|V(\tilde{x})\|_{\operatorname{op}}
\|\Sigma(x)-\Sigma(\tilde{x})\|_{\operatorname{op}}
\|W(x)\|_{\operatorname{op}}
\|c_k\|_2
\\
&\qquad
+
\|V(\tilde{x})\|_{\operatorname{op}}
\|\Sigma(\tilde{x})\|_{\operatorname{op}}
\|W(x)-W(\tilde{x})\|_{\operatorname{op}}
\|c_k\|_2.
\end{align*}
Using the simple bounds:
$
\|\Sigma(x)\|_{\operatorname{op}}
\le
\sigma_{\max}
$, 
$
\|\Sigma(\tilde{x})\|_{\operatorname{op}}
\le
\sigma_{\max}
$, and $
\|V(\tilde{x})\|_{\operatorname{op}}
=
\|W(x)\|_{\operatorname{op}}
=
1$, which hold for each $x\in \mathbb{B}_2^d$, we find that
\begin{align*}
\|r_k(x)-r_k(\tilde{x})\|_2
&\le
\lambda_+
\Big(
\sigma_{\max}\|V(x)-V(\tilde{x})\|_{\operatorname{op}}
+
\|\Sigma(x)-\Sigma(\tilde{x})\|_{\operatorname{op}}
+
\sigma_{\max}\|W(x)-W(\tilde{x})\|_{\operatorname{op}}
\Big).
\end{align*}
Now, by \eqref{eq:exp_difference_skew_bound}, we deduce that 
\[
\|V(x)-V(\tilde{x})\|_{\operatorname{op}}
\le
\|A_1(x)-A_1(\tilde{x})\|_{\operatorname{op}}
\le
\|A_1(x)-A_1(\tilde{x})\|_{F}
\le
L_{A_1}\|x-\tilde{x}\|_2,
\]
and similarly
$
    \|W(x)-W(\tilde{x})\|_{\operatorname{op}}
\le
    L_{A_2}\|x-\tilde{x}\|_2
$.
Also,
\[
    \|\Sigma(x)-\Sigma(\tilde{x})\|_{\operatorname{op}}
=
    \|\sigma(x)-\sigma(\tilde{x})\|_{\infty}
\le
    \|\sigma(x)-\sigma(\tilde{x})\|_2
\le
    L_\sigma\|x-\tilde{x}\|_2
.
\]
Therefore
$
    \|r_k(x)-r_k(\tilde{x})\|_2
\le
    \lambda_+
    \Big(
    \sigma_{\max}L_{A_1}
    +
    L_\sigma
    +
    \sigma_{\max}L_{A_2}
    \Big)\|x-\tilde{x}\|_2
$.
Taking the maximum over $k\in[K]_+$ proves \eqref{eq:lipschitz_parametrized_rhg_conclusion}.
The final claim follows from Proposition~\ref{prop:param_uniformly_separated_annular_configs}, which gives the annular and sibling-separation bounds with
$
\lambda_{\max}^{\operatorname{(lemma)}}=\sigma_{\max}\lambda_+
$ and $
s_{\operatorname{sep}}^{\operatorname{(lemma)}}=\sigma_{\min}\varepsilon$.
\end{proof}

\subsubsection{Separation Properties}
\label{s:sep_prop}

We now accumulate some lemmata which will allow us to quantify the separation at a given level of the nodes which are generated by the CFGs considered above.
\begin{lemma}[Recursive lower bound for level-wise centre separation]
\label{lem:recursive_lower_bound_deltaL}
Fix $d,K\in\mathbb{N}_+$, $0<s<1$, and let
$
r:\mathbb{B}_2^d\to \mathbb{R}^{d\times K}
$
be a CFG satisfying
\begin{align}
\label{eq:lem_recursive_deltaL_annulus}
\|r_k(x)\|_2
&\le
\lambda_{\max}
\qquad
\forall x\in \mathbb{B}_2^d,\ \forall k\in[K]_+,
\\
\label{eq:lem_recursive_deltaL_sibling_sep}
\|r_k(x)-r_{\tilde{k}}(x)\|_2
&\ge
s_{\operatorname{sep}}
\qquad
\forall x\in \mathbb{B}_2^d,\ \forall k,\tilde{k}\in[K]_+
\text{ with }k\neq \tilde{k}.
\end{align}
Assume moreover that $r$ is Lipschitz in the sense that there exists $L_r\ge 0$ such that
\begin{equation}
\label{eq:lem_recursive_deltaL_lip}
\|r(x)-r(\tilde{x})\|_{2:\infty}
\le
L_r\|x-\tilde{x}\|_2
\qquad
\forall x,\tilde{x}\in \mathbb{B}_2^d.
\end{equation}
Fix $x_0\in \mathbb{B}_2^d$, 
then $
\delta_1
\ge
s\,s_{\operatorname{sep}}
$, 
and, for every $l\in\mathbb{N}_+$,
\begin{equation}
\label{eq:recursive_delta_lower_bound}
\delta_{l+1}
\ge
\max\left\{
0,\,
\min\left\{
s^{l+1}s_{\operatorname{sep}},
\,
(1-s^{l+1}L_r)\delta_l-2s^{l+1}\lambda_{\max}
\right\}
\right\}.
\end{equation}
In particular, whenever $
(1-s^{l+1}L_r)\delta_l
>
2s^{l+1}\lambda_{\max}
$
one has the positive lower bound
\[
\delta_{l+1}
\ge
\min\left\{
s^{l+1}s_{\operatorname{sep}},
\,
(1-s^{l+1}L_r)\delta_l-2s^{l+1}\lambda_{\max}
\right\}
>0.
\]
\end{lemma}
\begin{proof}
Since
$
\mathcal{X}_1
=
\left\{
x_0+s\,r_k(x_0):\ k\in[K]_+
\right\}
$,
for any distinct $k,\tilde{k}\in[K]_+$,
$
\|(x_0+s\,r_k(x_0))-(x_0+s\,r_{\tilde{k}}(x_0))\|_2
=
s\|r_k(x_0)-r_{\tilde{k}}(x_0)\|_2
\ge
s\,s_{\operatorname{sep}}
$,
by \eqref{eq:lem_recursive_deltaL_sibling_sep}. This proves $\delta_1
\ge
s\,s_{\operatorname{sep}}$.

Fix now $l\in\mathbb{N}_+$ and let
$
u,v\in \mathcal{X}_{l+1}
$
be distinct. Then there exist $x,\tilde{x}\in \mathcal{X}_l$ and $k,\tilde{k}\in[K]_+$ such that
$
u=x+s^{l+1}r_k(x)$, and $
v=\tilde{x}+s^{l+1}r_{\tilde{k}}(\tilde{x})
$. 
We distinguish two cases.
\hfill\\
\noindent
\textit{Case 1: $x=\tilde{x}$.}
Then necessarily $k\neq \tilde{k}$, since $u\neq v$. Hence
\[
\|u-v\|_2
=
s^{l+1}\|r_k(x)-r_{\tilde{k}}(x)\|_2
\ge
s^{l+1}s_{\operatorname{sep}}
\]
by \eqref{eq:lem_recursive_deltaL_sibling_sep}.
\hfill\\
\noindent
\textit{Case 2: $x\neq \tilde{x}$.}
Then, $
\|u-v\|_2
=
\left\|
(x-\tilde{x})
+
s^{l+1}\big(r_k(x)-r_{\tilde{k}}(\tilde{x})\big)
\right\|_2
\ge
\|x-\tilde{x}\|_2
-
s^{l+1}\|r_k(x)-r_{\tilde{k}}(\tilde{x})\|_2
$.
Consequently, we find that
\begin{align*}
\|r_k(x)-r_{\tilde{k}}(\tilde{x})\|_2
&\le
\|r_k(x)-r_k(\tilde{x})\|_2
+
\|r_k(\tilde{x})-r_{\tilde{k}}(\tilde{x})\|_2
\\
&\le
\|r(x)-r(\tilde{x})\|_{2:\infty}
+
\|r_k(\tilde{x})\|_2
+
\|r_{\tilde{k}}(\tilde{x})\|_2
\\
&\le
L_r\|x-\tilde{x}\|_2+2\lambda_{\max},
\end{align*}
by \eqref{eq:lem_recursive_deltaL_lip} and \eqref{eq:lem_recursive_deltaL_annulus}. Therefore
$
\|u-v\|_2
\ge
(1-s^{l+1}L_r)\|x-\tilde{x}\|_2
-
2s^{l+1}\lambda_{\max}
\ge
(1-s^{l+1}L_r)\delta_l
-
2s^{l+1}\lambda_{\max}
$.
Combining both cases yields
$
\|u-v\|_2
\ge
\min\left\{
s^{l+1}s_{\operatorname{sep}},
\,
(1-s^{l+1}L_r)\delta_l-2s^{l+1}\lambda_{\max}
\right\}
$.
Since $u,v\in \mathcal{X}_{l+1}$ were arbitrary distinct points, and since distances are non-negative, we obtain
\[
\delta_{l+1}
\ge
\max\left\{
0,\,
\min\left\{
s^{l+1}s_{\operatorname{sep}},
\,
(1-s^{l+1}L_r)\delta_l-2s^{l+1}\lambda_{\max}
\right\}
\right\},
\]
which is \eqref{eq:recursive_delta_lower_bound}. The final claim is immediate.
\end{proof}

We now complete the main proof of this sub-appendix.
\begin{proof}[{Proof of Proposition~\ref{prop:key_separation}}]
Directly follows upon combining Proposition~\ref{prop:param_uniformly_separated_annular_configs} with Proposition~\ref{prop:lipschitz_parametrized_rhg}, which identify the constants
$
\lambda_{\max}^{\operatorname{(lemma)}}=\sigma_{\max}\lambda_+
$, $
s_{\operatorname{sep}}^{\operatorname{(lemma)}}=\sigma_{\min}\varepsilon
$, $
L_r=\lambda_+
\Big(
\sigma_{\max}L_{A_1}
+
L_\sigma
+
\sigma_{\max}L_{A_2}
\Big)$, and the assumptions of Lemma~\ref{lem:recursive_lower_bound_deltaL} are met.  Substituting these values into \eqref{eq:recursive_delta_lower_bound} yields \eqref{eq:recursive_delta_lower_bound_parametrized}.
\end{proof}

\section{Randomly Generating Separated Points on the Sphere}
\label{app:random_spherical_packings}

In this appendix, we record a simple probabilistic construction of separated point-clouds on the unit sphere via independent Haar samples.  We work throughout with the Euclidean metric inherited from $\mathbb{R}^d$.

\subsection{Haar measure on the sphere}

Fix $d\in\mathbb{N}_+$ with $d\ge 2$, and let
$
S^{d-1}
\eqdef
\left\{
x\in \mathbb{R}^d:\ |x|=1
\right\}
$.
Denote by $\sigma_{d-1}$ the $(d-1)$-dimensional Hausdorff measure on $S^{d-1}$, and define the \emph{Haar probability measure} $\mu_{d-1}$ on $S^{d-1}$ by
\[
\mu_{d-1}(A)
\eqdef
\frac{\sigma_{d-1}(A)}{\sigma_{d-1}(S^{d-1})}
,
\]
for any Borel $A\subseteq S^{d-1}$.
Equivalently, $\mu_{d-1}$ is the unique Borel probability measure on $S^{d-1}$ which is invariant under the action of the orthogonal group $O(d)$:
\[
\mu_{d-1}(QA)=\mu_{d-1}(A)
\qquad
\forall Q\in O(d)
\]
for all Borel $A\subseteq S^{d-1}$.
The standard way to sample from $\mu_{d-1}$ is by Gaussian normalization.

\begin{lemma}[Gaussian representation of Haar measure]
\label{lem:gaussian_representation_haar_sphere}
Let $G\sim N(0,I_d)$ in $\mathbb{R}^d$, and define
\[
U\eqdef \frac{G}{|G|}.
\]
Then $U$ is distributed according to $\mu_{d-1}$.
\end{lemma}

\begin{proof}
Fix $Q\in O(d)$. Since $QG\sim N(0,I_d)$, one has
$
Q\frac{G}{|G|}
=
\frac{QG}{|QG|}
\stackrel{d}{=}
\frac{G}{|G|}
$.  
Thus the law of $U$ is $O(d)$-invariant. Since $U\in S^{d-1}$ almost surely, its law is an $O(d)$-invariant probability measure on $S^{d-1}$. By uniqueness of Haar probability measure on the homogeneous space $S^{d-1}$, this law is $\mu_{d-1}$.
\end{proof}

\subsection{The one-dimensional marginal density}

The next lemma gives the density of a single coordinate of a Haar point on the sphere.

\begin{lemma}[Marginal density of a coordinate]
\label{lem:haar_coordinate_density}
Let $U=(U_1,\dots,U_d)\sim \mu_{d-1}$. Then $U_1$ admits the density
\[
f_d(t)
\eqdef
c_d\,(1-t^2)^{\frac{d-3}{2}}
\mathbf{1}_{[-1,1]}(t),
\qquad
c_d
\eqdef
\frac{\Gamma\!\left(\frac d2\right)}
{\sqrt{\pi}\,\Gamma\!\left(\frac{d-1}{2}\right)}.
\]
For every fixed $u\in S^{d-1}$ and every Borel function $\varphi:[-1,1]\to \mathbb{R}$ with $\varphi(\langle u,\cdot\rangle)\in L^1(\mu_{d-1})$,
\[
\int_{S^{d-1}}\varphi(\langle u,v\rangle)\,\mu_{d-1}(dv)
=
\int_{-1}^1 \varphi(t)\,f_d(t)\,dt.
\]
\end{lemma}

\begin{proof}
By rotational invariance, the law of $\langle u,U\rangle$ does not depend on $u\in S^{d-1}$. It therefore suffices to consider $u=e_1$, in which case $\langle e_1,U\rangle=U_1$. The stated density is the classical spherical-coordinate Jacobian formula, and the normalizing constant is determined by
\[
1
=
\int_{-1}^1 c_d\,(1-t^2)^{\frac{d-3}{2}}\,dt
=
c_d\,\frac{\sqrt{\pi}\,\Gamma\!\left(\frac{d-1}{2}\right)}
{\Gamma\!\left(\frac d2\right)}.
\]
The integral identity then follows by applying the formula to $\varphi(U_1)$ and using rotational invariance.
\end{proof}

\subsection{Small cap probabilities}

We next estimate the Haar mass of a small Euclidean cap.

\begin{lemma}[Small cap estimate]
\label{lem:small_cap_estimate}
Fix $d\ge 2$, let $u\in S^{d-1}$, and let $\varepsilon\in (0,1]$. Then
\[
\mu_{d-1}\big(
\{v\in S^{d-1}:\ |u-v|<\varepsilon\}
\big)
\le
\kappa_d\,\varepsilon^{d-1},
\]
where $\kappa_d
\eqdef
\frac{1}{d-1}\,
\frac{\Gamma\!\left(\frac d2\right)}
{\sqrt{\pi}\,\Gamma\!\left(\frac{d-1}{2}\right)}$.
\end{lemma}
\begin{proof}
Fix $u\in S^{d-1}$. Since
$
|u-v|^2
=
2-2\langle u,v\rangle
$, 
the event $|u-v|<\varepsilon$ is equivalent to
$
\langle u,v\rangle
>
1-\frac{\varepsilon^2}{2}
$.  
Therefore, by Lemma~\ref{lem:haar_coordinate_density},
$
\mu_{d-1}\big(
\{v\in S^{d-1}:\ |u-v|<\varepsilon\}
\big)
=
c_d
\int_{1-\varepsilon^2/2}^{1}
(1-t^2)^{\frac{d-3}{2}}\,dt
$.
For $t\in[-1,1]$,
$
1-t^2=(1-t)(1+t)\le 2(1-t)
$.
Hence
$
(1-t^2)^{\frac{d-3}{2}}
\le
2^{\frac{d-3}{2}}(1-t)^{\frac{d-3}{2}}
$,
and so
\begin{align*}
    \mu_{d-1}\big(
    \{v\in S^{d-1}:\ |u-v|<\varepsilon\}
    \big)
&\le
    c_d\,2^{\frac{d-3}{2}}
    \int_{1-\varepsilon^2/2}^{1}
    (1-t)^{\frac{d-3}{2}}\,dt
\\
&=
    c_d\,2^{\frac{d-3}{2}}
    \cdot
    \frac{2}{d-1}
    \left(\frac{\varepsilon^2}{2}\right)^{\frac{d-1}{2}}
\\
&=
    \frac{c_d}{d-1}\,\varepsilon^{d-1}
.
\qedhere
\end{align*}
\end{proof}

\subsection{A high-probability separation bound}
\label{a:generatingpacking}

We now prove the basic probabilistic statement needed for random packings.

\begin{proposition}[Independent Haar samples are separated with high probability]
\label{prop:haar_samples_separated}
Fix $d\ge 2$, let
$
U_1,\dots,U_m
$
be independent $S^{d-1}$-valued random variables with common law $\mu_{d-1}$, and let $\varepsilon\in(0,1]$. Then
\[
\mathbb{P}\!\left(
\min_{1\le i<j\le m}|U_i-U_j|
\ge
\varepsilon
\right)
\ge
1-\binom{m}{2}\kappa_d\,\varepsilon^{d-1},
\]
where $\kappa_d$ is as in Lemma~\ref{lem:small_cap_estimate}.
\end{proposition}
In particular, if $\varepsilon=\frac{1}{N}$, then
\[
\mathbb{P}\!\left(
\min_{1\le i<j\le m}|U_i-U_j|
\ge
\frac{1}{N}
\right)
\ge
1-\binom{m}{2}\kappa_d\,N^{-(d-1)}.
\]
\begin{proof}[{Proof of Proposition~\ref{prop:haar_samples_separated}}]
For each pair $1\le i<j\le m$, define the bad event
\[
E_{ij}
\eqdef
\left\{
|U_i-U_j|<\varepsilon
\right\}.
\]
Then
$
\left\{
\min_{1\le i<j\le m}|U_i-U_j|<\varepsilon
\right\}
=
\bigcup_{1\le i<j\le m} E_{ij}
$.
Taking a union bound,
$
\mathbb{P}\!\left(
\min_{1\le i<j\le m}|U_i-U_j|<\varepsilon
\right)
\le
\sum_{1\le i<j\le m}\mathbb{P}(E_{ij})
$.
Fix $i<j$. Conditional on $U_i$, the random variable $U_j$ is Haar-distributed and independent of $U_i$, so Lemma~\ref{lem:small_cap_estimate} yields
$
\mathbb{P}(E_{ij}\mid U_i)
=
\mu_{d-1}\big(
\{v\in S^{d-1}:\ |U_i-v|<\varepsilon\}
\big)
\le
\kappa_d\,\varepsilon^{d-1}
$.
Taking expectations gives
$
\mathbb{P}(E_{ij})
\le
\kappa_d\,\varepsilon^{d-1}
$.
Therefore
$
\mathbb{P}\!\left(
\min_{1\le i<j\le m}|U_i-U_j|<\varepsilon
\right)
\le
\binom{m}{2}\kappa_d\,\varepsilon^{d-1}
$.
Taking complements proves the result.
\end{proof}

\begin{corollary}[The $1/N$-separation regime]
\label{cor:one_over_N_separation_regime}
Fix $d\ge 2$, and let $m_N\in \mathbb{N}_+$ for $N\in \mathbb{N}_+$. If $\lim\limits_{N\uparrow \infty}\, 
m_N^2\,N^{-(d-1)}=0$ then 
\begin{equation}
\label{eq:sepclarify}
\lim\limits_{N\uparrow \infty}\,
\mathbb{P}\!\left(
\min_{1\le i<j\le m_N}|U_i-U_j|
\ge
\frac{1}{N}
\right)
= 
1
.
\end{equation}
In particular, if
$
m_N=N\log N,
$
then~\eqref{eq:sepclarify} holds, for every $d\ge 4$.
\end{corollary}
\begin{proof}
The first claim is immediate from Proposition~\ref{prop:haar_samples_separated}. For the second, note that
$
m_N^2\,N^{-(d-1)}
=
N^{3-d}(\log N)^2,
$ which converges to $0$; provided that $d\ge 4$.
\end{proof}

\subsection{Rejection sampling}

We conclude with a practical procedure for generating an $\varepsilon$-separated subset of $S^{d-1}$.

\begin{algorithm}[H]
\caption{Rejection sampling of an $\varepsilon$-separated Haar point-cloud on $S^{d-1}$}
\label{alg:rejection_sampling_sphere}
\KwIn{dimension $d\in\mathbb{N}_+$ with $d\ge 2$, target size $K\in\mathbb{N}_+$, separation radius $\varepsilon>0$}
\KwOut{a set $\mathcal{P}=\{u_1,\dots,u_M\}\subseteq S^{d-1}$ with pairwise distances at least $\varepsilon$, where $M\le K$}
Initialize $\mathcal{P}\leftarrow \emptyset$\;
\While{$|\mathcal{P}|<K$}{
    Sample $G\sim N(0,I_d)$ in $\mathbb{R}^d$\;
    Set $u\leftarrow G/|G|$\;
    \If{$|u-v|\ge \varepsilon$ for every $v\in \mathcal{P}$}{
        Update $\mathcal{P}\leftarrow \mathcal{P}\cup\{u\}$\;
    }
}
\Return $\mathcal{P}$\;
\end{algorithm}

By construction, every accepted point lies on $S^{d-1}$ and every two accepted points are at Euclidean distance at least $\varepsilon$.  Thus Algorithm~\ref{alg:rejection_sampling_sphere} always returns an $\varepsilon$-separated subset of the sphere.

\section{Supplement Details}

\subsection{The Graphical Model}
We attach Figure~\ref{fig:lecluster} to help explain the infinite classification field.
\begin{figure}[htbp]
    \centering
    \includegraphics[width=0.8\linewidth]{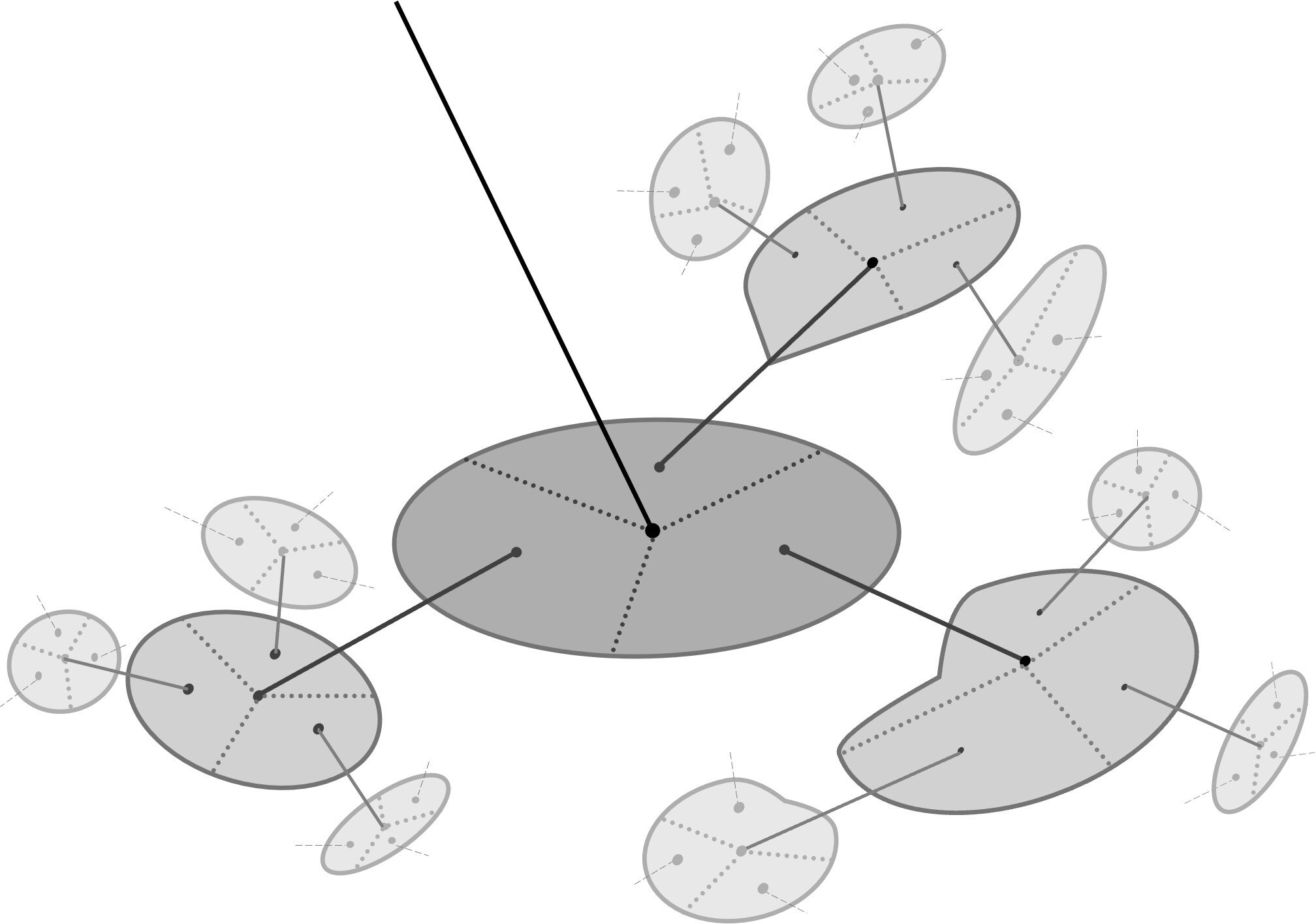}
    \caption{Two views of an infinite classification field. Iterating the CFG first produces a rooted $K$-ary tree. When cells that occupy the same geometric region are identified, this tree becomes the metric DAG used to define our cell and point based distances.}
    \label{fig:lecluster}
\end{figure}

\subsection{Metric definitions}
\label{A:Metric}
Let $\mathcal H_l=\{x^{(l)}_1,\dots,x^{(l)}_{n_l}\}$ denote the ground-truth nodes at level $l$, and let $\widehat{\mathcal H}_l=\{\hat x^{(l)}_1,\dots,\hat x^{(l)}_{n_l}\}$ denote the predicted nodes at the same level. In the synthetic fractal benchmarks, children are generated and stored parent by parent in a fixed map order, so a level-wise pointwise comparison is well defined.

\paragraph{MSE.}
At level \(l\), we report the coordinate-averaged mean squared error
\[
\operatorname{MSE}(\widehat{\mathcal H}_l,\mathcal H_l)
=
\frac{1}{d n_l}
\sum_{i=1}^{n_l}
\|\widehat x_i^{(l)}-x_i^{(l)}\|_2^2 .
\]
This convention matches the implementation, which averages squared error over both
nodes and ambient coordinates.

\paragraph{Permutation-Invariant Chamfer Distance (PI-CD).}
To measure set-level geometric agreement without relying on slot correspondence, we use the symmetric Chamfer distance with squared Euclidean cost:
\[
\mathrm{PI\text{-}CD}(\widehat{\mathcal H}_l,\mathcal H_l)
=
\frac{1}{|\widehat{\mathcal H}_l|}
\sum_{\hat x\in\widehat{\mathcal H}_l}
\min_{y\in\mathcal H_l}\|\hat x-y\|_2^2
+
\frac{1}{|\mathcal H_l|}
\sum_{y\in\mathcal H_l}
\min_{\hat x\in\widehat{\mathcal H}_l}\|y-\hat x\|_2^2.
\]

\paragraph{$d_{\mathrm{pt}}$-distortion.}
For a hierarchy truncated at level $l$, let
\[
\mathcal T_{\le l}
=
\bigcup_{j=0}^{l}\mathcal H_j,
\qquad
\widehat{\mathcal T}_{\le l}
=
\bigcup_{j=0}^{l}\widehat{\mathcal H}_j.
\]
We assign each tree edge weights given by Euclidean parent-child distances. This induces a shortest-path metric on the nodes: for any two nodes $u,v$ in the truncated hierarchy,
\[
d_{\mathrm{pt}}(u,v)
=
d_{\mathrm{root}}(u)+d_{\mathrm{root}}(v)-2\,d_{\mathrm{root}}(\mathrm{lca}(u,v)),
\]
where $d_{\mathrm{root}}(u)$ is the weighted path length from the root to $u$, and $\mathrm{lca}(u,v)$ is the lowest common ancestor of $u$ and $v$.

Let $D^{(l)}_{\mathrm{true}}$ and $D^{(l)}_{\mathrm{pred}}$ denote the corresponding pairwise distance matrices on $\mathcal T_{\le l}$ and $\widehat{\mathcal T}_{\le l}$. We define the normalized $d_{\mathrm{pt}}$-distortion by
\[
\mathrm{Dist}_{d_{\mathrm{pt}}}^{(l)}
=
\frac{
\displaystyle
\frac{1}{\binom{N_l}{2}}
\sum_{1\le i<j\le N_l}
\left|
D^{(l)}_{\mathrm{pred}}(i,j)-D^{(l)}_{\mathrm{true}}(i,j)
\right|
}{
\displaystyle
\max_{1\le i<j\le N_l}
D^{(l)}_{\mathrm{true}}(i,j)
},
\]
where $N_l = |\mathcal T_{\le l}|$. The numerator is the mean absolute discrepancy over all unordered node pairs, and the denominator is the diameter of the ground-truth truncated hierarchy.
The cell metric is used in the theory; in the IFS and image-induced experiments, explicit Voronoi cells are either unavailable or not the object being evaluated, so we report centre-based surrogates: ordered MSE, permutation-invariant chamfer distance, and \(d_{\rm pt}\)-distortion.

\subsection{Experiments Details}
\label{s:Experiments_Details}
Unless stated otherwise, CFG-generated hierarchies use $d=2$, $K=3$,
$s=0.5$, root $x_0=(0,0)$, and observed depth $L_{\mathrm{train}}=2$.
Each method is rolled out for nine unseen levels. We report results over
$9$ independent trials, resampling both the reference child template $C$ and the
ground-truth CFG parameters in each trial. 
The CFG uses
$\nu=3.25$, $\sigma_{\min}=0.1$, $\sigma_{\max}=1.0$, hidden dimension
$500$, and residual-norm check threshold \(\lambda_{\mathrm{chk}}=1.0\). The CFP is a 4-layer MLP with hidden width $64$ trained with Adam for $3000$ epochs at learning rate $10^{-3}$. 
Baselines are chosen to isolate what is gained by learning an input-dependent refinement field: an average residual template, a learnable constant residual template, and an affine residual predictor, as shown in Appendix~\ref{A:Baseline}.
For image-induced hierarchies, the same learning protocol is applied in CLIP embedding space, with details provided in Appendix~\ref{A:cifar_details}. 
All experiments were run on one NVIDIA L40S GPU with 48 GB memory, using Python 3.11.5, PyTorch 2.11.0, and CUDA 12.9. The main experimental runs take approximately one hour in total.

\subsubsection{CFG-generated hierarchies}
\label{A:hierarchiesplots}

For the CFG extrapolation experiments, we generated the training hierarchies shown in Figure~\ref{fig:hierarchy}. Each panel corresponds to one independently generated hierarchy observed up to depth $L_{\mathrm{train}}$ across the $9$ trials.

\begin{figure}[htbp]
    \centering
    \includegraphics[width=0.7\linewidth]{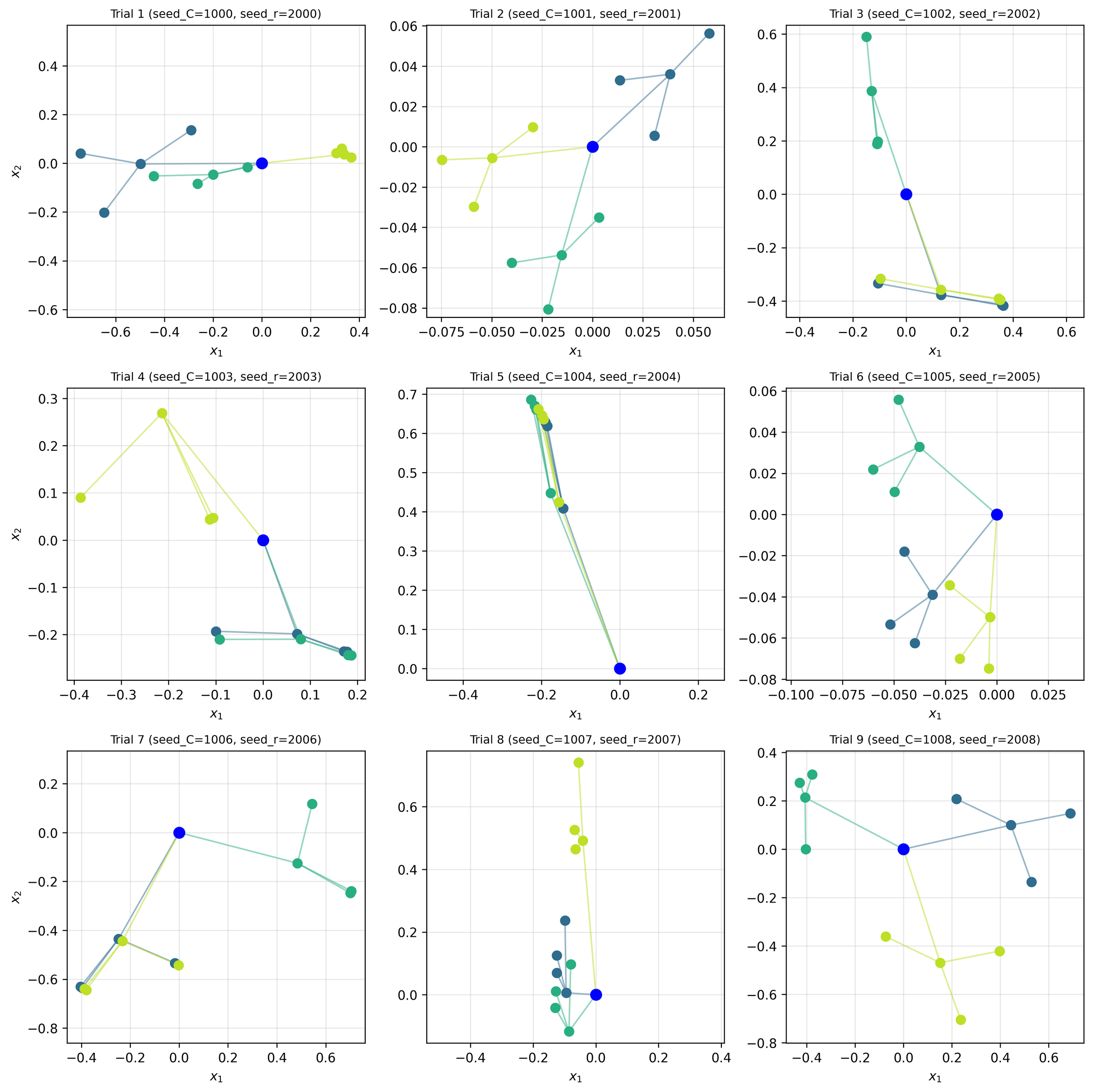}
    \caption{Observed CFG training prefixes across 9 independent trials. Each panel shows the finite hierarchy observed during training before recursive rollout. The diversity across panels reflects resampling of both the reference child template and the ground-truth CFG parameters.}
    \label{fig:hierarchy}
\end{figure}

\begin{figure}[htbp]
    \centering
    \includegraphics[width=0.7\linewidth]{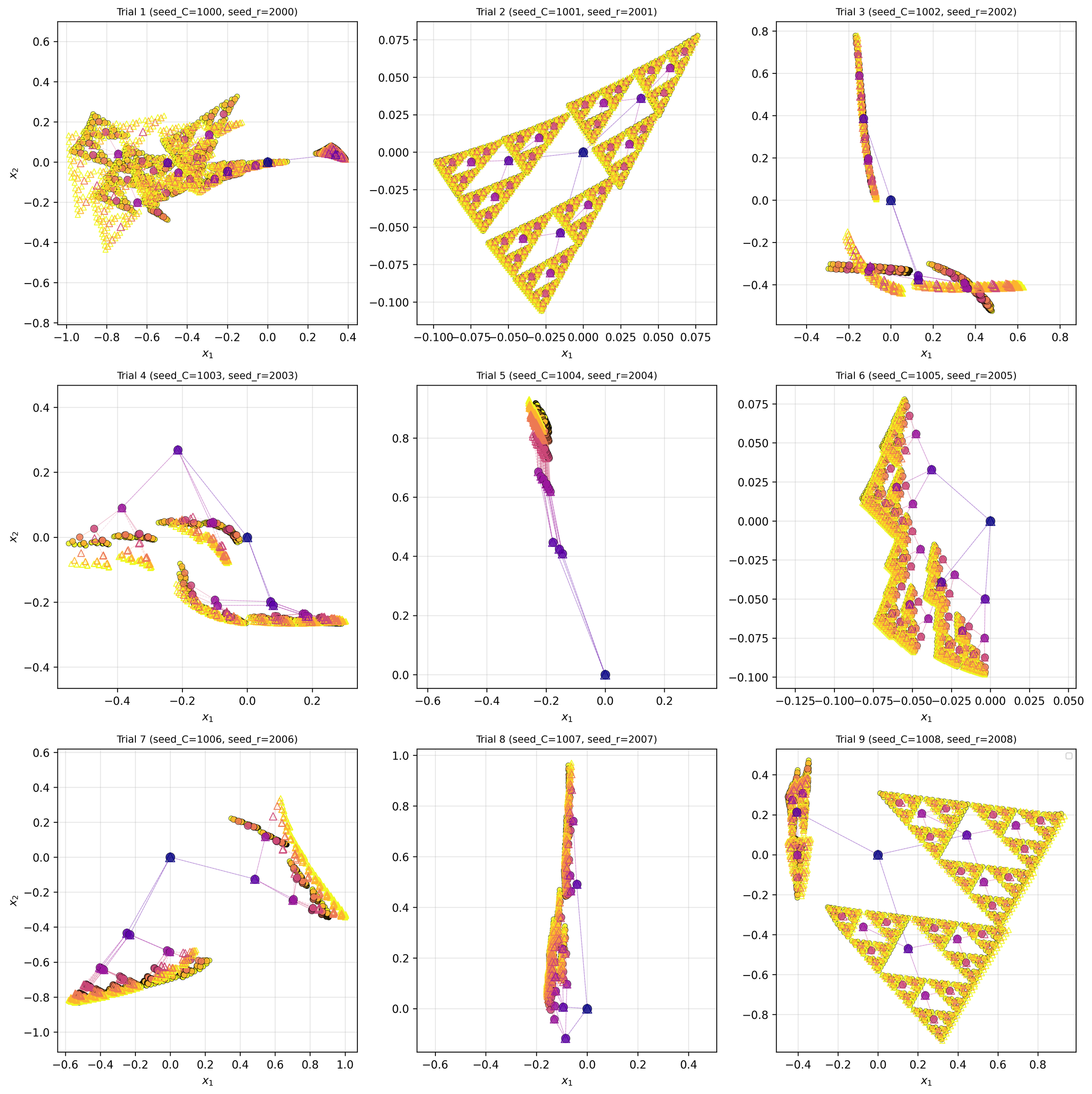}
    \caption{Qualitative CFP rollouts on CFG-generated hierarchies. Each panel shows a CFP trained on the observed prefix and recursively rolled out to unseen depths. The plots illustrate that low extrapolation error corresponds to coherent geometric completion, not only isolated one-step prediction.}
    \label{fig:6hierarchy}
\end{figure}
Figure~\ref{fig:6hierarchy} shows qualitative rollouts of the CFP across the $9$ independent trials. 
The predicted hierarchies continue
to track the global branching structure of the ground truth beyond the
observed prefix, illustrating that the low MSE in Figure~\ref{fig:MSE}
corresponds to coherent geometric completion rather than isolated pointwise
matches.
The following compact notation and Algorithm~\ref{alg:matched_cfg_experiment_compact}
give the exact generation, training, and rollout protocol used for our CFG experiments.
\begin{algorithm}[H]
\small
\DontPrintSemicolon
\caption{Matched CFG-generated hierarchy experiment}
\label{alg:matched_cfg_experiment_compact}
\KwIn{
\(d=2\), \(K=3\), \(s=1/2\), \(x_0=0\), \(L_{\mathrm{train}}=2\),
\(L_{\mathrm{tot}}=11\), \(T=9\);
CFG hyperparameters
\((R_C,\nu,\sigma_{\min},\sigma_{\max},h,\lambda_{\mathrm{chk}})
=(1,3.25,0.1,1.0,500,1.0)\);
CFP width \(64\), hidden depth \(3\), Adam learning rate \(10^{-3}\), epochs \(3000\).
}
\KwOut{
Held-out level-wise errors \(\{e_{\ell,m}^{(t)}\}\). Figure~\ref{fig:MSE}
reports their geometric means with \(\pm 1\) standard deviation bands in
\(\log_{10}\) MSE.
}

Define model classes
\[
\mathcal F_{\mathrm{const}}:\widehat r(x)_k=c_k,\qquad
\mathcal F_{\mathrm{aff}}:\widehat r(x)_k=A_kx+b_k,\qquad
\mathcal F_{\mathrm{CFP}}:\text{ReLU MLP}.
\]

\For{\(t=0,\ldots,T-1\)}{
    Sample \(r_t\leftarrow \operatorname{SampleCFG}(t)\)\;

    Generate the target hierarchy
    \[
        \mathcal H^{(t)}
        \leftarrow
        \operatorname{Rollout}(r_t,x_0,L_{\mathrm{tot}},s).
    \]
    Assert
    \[
        \max_{\ell<L_{\mathrm{tot}}}
        \max_{\alpha\in[K]^\ell}
        \max_{k\in[K]}
        \|r_t(x_\alpha^{(t)})_k\|_2
        \le \lambda_{\mathrm{chk}} .
    \]

    Build observed parent-to-children tuples
    \[
        \mathcal D_t
        \leftarrow
        \Big\{
        \big(x_\alpha^{(t)},(x_{\alpha k}^{(t)})_{k=1}^K,s^{|\alpha|+1}\big)
        :
        |\alpha|<L_{\mathrm{train}}
        \Big\}.
    \]

    Fit the average-residual baseline by
    \[
        \bar r_k
        \leftarrow
        \frac{1}{N_{\mathrm{par}}}
        \sum_{\ell=0}^{L_{\mathrm{train}}-1}
        \sum_{\alpha\in[K]^\ell}
        \frac{x_{\alpha k}^{(t)}-x_\alpha^{(t)}}{s^{\ell+1}},
        \qquad
        N_{\mathrm{par}}=\sum_{\ell=0}^{L_{\mathrm{train}}-1}K^\ell ,
    \]
    and set \(\widehat r_{\mathrm{avg}}(x)_k\equiv \bar r_k\)\;

    \For{\(m\in\{\mathrm{const},\mathrm{aff},\mathrm{CFP}\}\)}{
        Train
        \[
            \widehat r_m
            \leftarrow
            \operatorname{AdamTrain}
            \big(\mathcal F_m,\mathcal L(\cdot;\mathcal D_t)\big).
        \]
    }

    \For{\(m\in\{\mathrm{avg},\mathrm{const},\mathrm{aff},\mathrm{CFP}\}\)}{
        Recursively roll out
        \[
            \widehat{\mathcal H}^{(t)}_m
            \leftarrow
            \operatorname{Rollout}
            (\widehat r_m,x_0,L_{\mathrm{tot}},s).
        \]

        \For{\(\ell=L_{\mathrm{train}}+1,\ldots,L_{\mathrm{tot}}\)}{
            Record
            \[
                e_{\ell,m}^{(t)}
                \leftarrow
                \frac{1}{dK^\ell}
                \sum_{\alpha\in[K]^\ell}
                \big\|
                \widehat x_{\alpha,m}^{(t)}
                -
                x_\alpha^{(t)}
                \big\|_2^2 .
            \]
        }
    }
}

\Return{the raw errors \(e_{\ell,m}^{(t)}\); aggregate them as specified in each
figure or table.}
\end{algorithm}

\paragraph{Compact notation for Algorithm~\ref{alg:matched_cfg_experiment_compact}}
In trial \(t\), \(\operatorname{SampleCFG}(t)\) uses seeds \(1000+t\) and \(2000+t\).
It first samples a reference template \(C_t=(c_1|\cdots|c_K)\) by
\(c_k=R_C g_k/\|g_k\|_2\), \(g_k\sim N(0,I_d)\). It then samples neural
parameters \(\Theta_t\) and defines
\[
\begin{aligned}
\widetilde A_i(x)
&=
\operatorname{mat}_{d\times d}
\!\left(
\nu\,\rho\!\left(B_i\tanh(W_i x+b_i)\right)
\right),\quad
S_i(x)=\frac{\widetilde A_i(x)-\widetilde A_i(x)^\top}{2},\\
\sigma(x)
&=
\sigma_{\min}
+
(\sigma_{\max}-\sigma_{\min})
\rho\!\left(B_\sigma\tanh(W_\sigma x+b_\sigma)\right),\\
r_{\Theta_t}(x)
&=
\exp(S_1(x))\,\operatorname{diag}(\sigma(x))\,\exp(S_2(x))\,C_t .
\end{aligned}
\]
Here \(\rho(z)=(1+\exp(-z))^{-1}\) is applied coordinatewise. All entries of
\(W_i,b_i,B_i,W_\sigma,b_\sigma,B_\sigma\) are sampled independently from
\(N(0,1)\). The matrices \(B_i\) have shape \(d^2\times h\), while
\(B_\sigma\) has shape \(d\times h\). This is the parametrized CFG form in Eq.~(5), with the second orthogonal factor
written in the implementation convention. 
Indeed, since \(S_2(x)\) is skew-symmetric, the displayed implementation equals
Eq.~(5) after identifying \(A_1(x)=S_1(x)\) and \(A_2(x)=-S_2(x)\), so that
\(\exp(A_2(x))^\top=\exp(S_2(x))\).
We also write
\[
\operatorname{Rollout}(g,x_0,L,s)_0=\{x_0\},\qquad
\operatorname{Rollout}(g,x_0,L,s)_\ell
=
\{x+s^\ell g(x)_k:x\in H_{\ell-1},\,k\in[K]\},
\]
with children stored in parent-by-parent and child-slot order. For a training set
\(\mathcal D\), define
\[
\mathcal L(\widehat r;\mathcal D)
=
\sum_{(x,y_{1:K},a)\in\mathcal D}
\sum_{k=1}^K
\|x+a\,\widehat r(x)_k-y_k\|_2^2 .
\]

\subsubsection{Baseline definitions}
\label{A:Baseline}
The average-residual
baseline estimates a location-independent refinement template,
\[
\bar r_k
=
\frac{1}{N_{\mathrm{par}}}
\sum_{l=0}^{L_{\mathrm{train}}-1}
\sum_{x\in \mathcal H_l}
\frac{\mathrm{chld}(x)_k-x}{s^{\,l+1}},
\qquad k=1,\dots,K,
\]
where
$N_{\mathrm{par}}=\sum_{l=0}^{L_{\mathrm{train}}-1}|\mathcal H_l|$.
The learnable-constant baseline uses $\hat r(x)_k=c_k$, and the linear affine
baseline uses $\hat r(x)_k=A_kx+b_k$. All methods are evaluated by the same
recursive rollout rule, $\widehat{\mathrm{chld}}(x)_k=x+s^{\,l+1}\hat r(x)_k.$
The learnable-constant and linear affine baselines are trained with the same
hierarchy loss and recursive rollout protocol as CFP; the average-residual
baseline is fit in closed form by averaging normalized observed residuals.

\subsubsection{IFS benchmark definitions}
\label{A:ifs_details}
All synthetic IFS benchmarks are generated in \(\mathbb R^2\) with rollout scale
\(s=1\), so the hierarchy update reduces to
\[
x+r(x)_k=f_k(x).
\]
Our main out-of-family benchmark suite uses smooth nonlinear deformations of
classical fractal IFS families.  Each nonlinear map is obtained by starting from
an affine contraction \(A_kx+b_k\) and adding a child-specific sinusoidal warp:
\[
f_k(x)
=
A_kx+b_k
+\alpha_k
\begin{bmatrix}
\sin(\langle w^{(1)}_k,x\rangle+\phi_k)\\
\cos(\langle w^{(2)}_k,x\rangle-0.37\phi_k)
\end{bmatrix}.
\]
The implementation also supports a quadratic perturbation, but we set its
coefficient to zero in all reported experiments.  Thus the residual rule
\(r(x)_k=f_k(x)-x\) is genuinely nonlinear while remaining a smooth perturbation
of the corresponding contractive IFS.

For the nonlinear Sierpi\'nski benchmark, the base maps are the standard
three-map Sierpi\'nski IFS with contraction ratio \(1/2\), using vertices
\[
(0,0),\qquad (1,0),\qquad (1/2,\sqrt{3}/2),
\]
and root initialized at their centroid.  For child index \(k=0,1,2\), we set
\[
\alpha_k=0.080,\qquad
w^{(1)}_k=(3.00+0.30k,\;2.10-0.20k),
\]
\[
w^{(2)}_k=(-1.80+0.20k,\;3.00+0.25k),
\qquad
\phi_k=0.80k.
\]
For the nonlinear Cantor dust benchmark, the base maps are
\[
A_k=I/3,\qquad
b_k\in
\left\{
(0,0),\,
(2/3,0),\,
(0,2/3),\,
(2/3,2/3)
\right\},
\]
with root \(x_0=(0.5,0.5)\).  We use the same frequencies and amplitude as
above, with phase \(\phi_k=0.60k\).  For the nonlinear Koch curve benchmark, the
base maps are
\[
A_1=I/3,\quad A_2=R(60^\circ)/3,\quad
A_3=R(-60^\circ)/3,\quad A_4=I/3,
\]
with translations
\[
(0,0),\qquad (1/3,0),\qquad (1/2,\sqrt{3}/6),\qquad (2/3,0),
\]
and root \(x_0=(0.5,0)\).  We again use \(\alpha_k=0.080\) and the same
frequency schedule, with phase \(\phi_k=0.70k\).

Finally, the random nonlinear IFS benchmark samples \(K=3\) maps independently.
For each child map,
\[
A_k=c_kR(\theta_k),\qquad
c_k\sim \mathrm{Unif}(0.20,0.38),\qquad
\theta_k\sim \mathrm{Unif}(0,2\pi),
\]
\[
b_k\sim \mathrm{Unif}([-0.45,0.45]^2).
\]
The two warp directions are sampled by drawing Gaussian vectors, normalizing
them, and scaling each by an independent \(\mathrm{Unif}(2.5,3.5)\) factor.  The
phase and amplitude are sampled as
\[
\phi_k\sim \mathrm{Unif}(0,2\pi),
\qquad
\alpha_k\sim \mathrm{Unif}(0.055,0.085),
\]
with root \(x_0=0\).  For the first three nonlinear benchmarks, the underlying
IFS is fixed across trials.  For the random nonlinear benchmark, a new set of
maps is sampled in each trial.

\paragraph{Contractivity of the nonlinear IFS maps.}
The sinusoidal perturbations are chosen small enough that each nonlinear map
remains contractive. For
\[
f_k(x)
=
A_kx+b_k+\alpha_k
\begin{bmatrix}
\sin(\langle w_k^{(1)},x\rangle+\phi_k)\\
\cos(\langle w_k^{(2)},x\rangle-0.37\phi_k)
\end{bmatrix},
\]
the Jacobian is
\[
Df_k(x)
=
A_k
+
\alpha_k
\begin{bmatrix}
\cos(\langle w_k^{(1)},x\rangle+\phi_k)(w_k^{(1)})^\top\\
-\sin(\langle w_k^{(2)},x\rangle-0.37\phi_k)(w_k^{(2)})^\top
\end{bmatrix}.
\]
Hence, using $\|\cdot\|_2\leq \|\cdot\|_F$,
\[
\|Df_k(x)\|_2
\le
\|A_k\|_2
+
\alpha_k
\sqrt{
\|w_k^{(1)}\|_2^2+\|w_k^{(2)}\|_2^2
}
\;=:\; L_k .
\]
We choose parameters so that $L_k<1$ for every child map. In the fixed
nonlinear Sierpi\'nski benchmark this bound is at most
\[
\frac12
+
0.08
\sqrt{3.6^2+1.7^2+1.4^2+3.5^2}
<0.95 .
\]
For the nonlinear Cantor dust and Koch curve benchmarks, the base contraction
satisfies $\|A_k\|_2\le 1/3$, and the largest frequency vectors in the schedule
give
\[
\frac13
+
0.08
\sqrt{3.9^2+1.5^2+1.2^2+3.75^2}
<0.80 .
\]
For the random nonlinear IFS family, the sampled parameters satisfy
$\|A_k\|_2\le 0.38$, $\alpha_k\le 0.085$, and
$\|w_k^{(1)}\|_2,\|w_k^{(2)}\|_2\le 3.5$, so
\[
0.38+0.085\sqrt{3.5^2+3.5^2}<0.81 .
\]
Therefore each nonlinear child map is globally Lipschitz with constant strictly
smaller than one, and the warped systems remain contractive IFSs.

We additionally prepare affine counterparts of these benchmarks as a diagnostic
linear IFS suite.  These remove the sinusoidal warp and use only the base maps.
The affine Sierpi\'nski benchmark uses the standard three-map IFS with
contraction ratio \(1/2\).  The affine Cantor dust benchmark uses
\[
f_k(x)=x/3+b_k,
\qquad
b_k\in
\left\{
(0,0),\,
(2/3,0),\,
(0,2/3),\,
(2/3,2/3)
\right\},
\]
with root \(x_0=(0.5,0.5)\).  The affine Koch curve benchmark uses
\[
f_1(x)=x/3,\qquad
f_2(x)=R(60^\circ)x/3+(1/3,0),
\]
\[
f_3(x)=R(-60^\circ)x/3+(1/2,\sqrt{3}/6),\qquad
f_4(x)=x/3+(2/3,0),
\]
with root \(x_0=(0.5,0)\).  In the random affine benchmark, we sample
\[
f_k(x)=s_kR(\theta_k)x+b_k,
\]
where
\[
s_k\sim \mathrm{Unif}(0.2,0.45),\qquad
\theta_k\sim \mathrm{Unif}(0,2\pi),\qquad
b_k\sim \mathrm{Unif}([-0.5,0.5]^2),
\]
and use \(x_0=0\).  For the first three affine benchmarks, the IFS is fixed
across trials, while the random affine benchmark resamples the maps in each
trial.

\begin{table*}[t]
\centering
\scriptsize
\setlength{\tabcolsep}{4pt}
\renewcommand{\arraystretch}{0.95}
\caption{
These benchmarks remove the sinusoidal warp from the nonlinear IFS suite and use only affine contraction maps.  We report these results as a linear counterpart to the nonlinear IFS suite.
}
\label{tab:fractal_benchmarkwise}
\begin{tabular}{lccccc}
\toprule
Method & Sierpi\'nski & Cantor Dust & Koch Curve & Rand. Affine IFS & Macro Avg. \\
\midrule
\multicolumn{6}{l}{\textit{MSE} $\downarrow$} \\
CFP (Ours) & \textbf{0.80 $\pm$ 0.28} & \textbf{1.13 $\pm$ 0.33} & \textbf{0.34 $\pm$ 0.02} & \textbf{2.75 $\pm$ 3.29} & \textbf{1.26 $\pm$ 1.78} \\
Affine Residual Predictor & 13.01 $\pm$ 6.12 & 13.93 $\pm$ 8.49 & 6.02 $\pm$ 1.81 & 8.60 $\pm$ 2.88 & 10.39 $\pm$ 6.04 \\
Learnable Const & 67.33 $\pm$ 0.00 & 241.99 $\pm$ 0.00 & 81.05 $\pm$ 0.00 & 133.25 $\pm$ 74.68 & 130.90 $\pm$ 78.36 \\
Avg Residual & 67.33 $\pm$ 0.00 & 241.99 $\pm$ 0.00 & 81.05 $\pm$ 0.00 & 133.25 $\pm$ 74.68 & 130.91 $\pm$ 78.36 \\
\midrule
\multicolumn{6}{l}{\textit{PI-CD} $\downarrow$} \\
CFP (Ours) & \textbf{1.87 $\pm$ 0.50} & \textbf{2.49 $\pm$ 0.67} & \textbf{0.39 $\pm$ 0.03} & \textbf{3.07 $\pm$ 3.49} & \textbf{1.96 $\pm$ 1.94} \\
Affine Residual Predictor & 16.57 $\pm$ 7.53 & 16.16 $\pm$ 8.33 & 4.40 $\pm$ 1.55 & 8.08 $\pm$ 5.38 & 11.31 $\pm$ 7.87 \\
Learnable Const & 102.34 $\pm$ 0.00 & 286.86 $\pm$ 0.00 & 54.34 $\pm$ 0.00 & 113.56 $\pm$ 58.15 & 139.27 $\pm$ 94.21 \\
Avg Residual & 102.34 $\pm$ 0.00 & 286.86 $\pm$ 0.00 & 54.34 $\pm$ 0.00 & 113.56 $\pm$ 58.15 & 139.28 $\pm$ 94.21 \\
\midrule
\multicolumn{6}{l}{\textit{$d_{\mathrm{pt}}$-Distortion} $\downarrow$} \\
CFP (Ours) & \textbf{8.43 $\pm$ 1.34} & \textbf{6.21 $\pm$ 0.68} & \textbf{5.79 $\pm$ 0.48} & \textbf{10.59 $\pm$ 4.57} & \textbf{7.75 $\pm$ 2.97} \\
Affine Residual Predictor & 60.77 $\pm$ 17.14 & 42.22 $\pm$ 16.18 & 42.37 $\pm$ 13.07 & 51.03 $\pm$ 41.67 & 49.10 $\pm$ 24.08 \\
Learnable Const & 72.88 $\pm$ 0.00 & 89.04 $\pm$ 0.00 & 95.13 $\pm$ 0.00 & 131.65 $\pm$ 23.40 & 97.17 $\pm$ 24.53 \\
Avg Residual & 72.88 $\pm$ 0.00 & 89.04 $\pm$ 0.00 & 95.13 $\pm$ 0.00 & 131.65 $\pm$ 23.40 & 97.17 $\pm$ 24.53 \\
\bottomrule
\end{tabular}
\end{table*}

\begin{figure}[htbp]
    \centering
    \includegraphics[width=\linewidth]{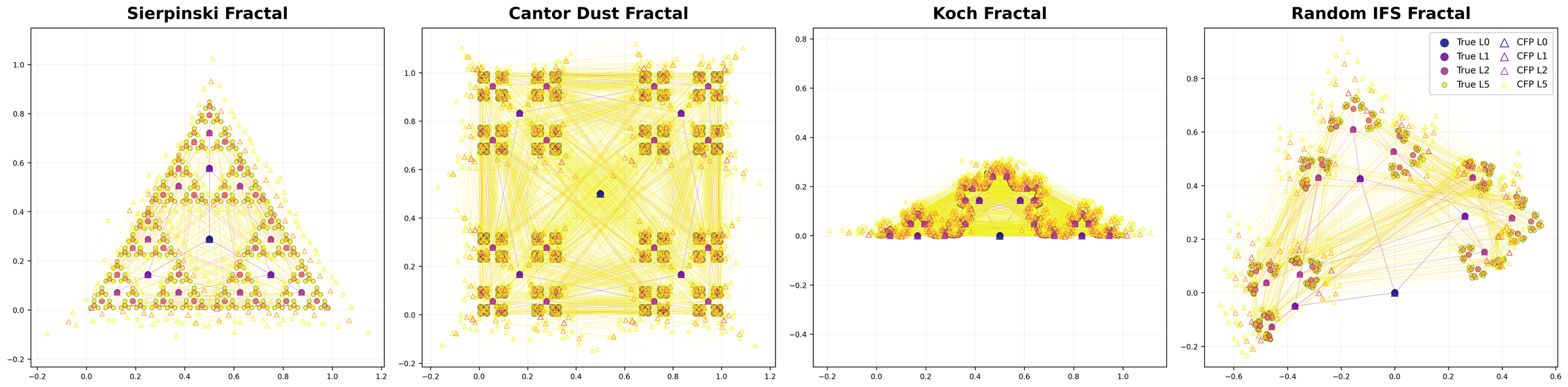}
    \caption{Recursive rollout on linear IFS-induced classification fields. From left to right: Sierpiński triangle, 2D Cantor dust, Koch curve, and random affine IFS. Models are trained only on levels 0–2 and rolled out to level 5. Filled circles and solid edges denote ground truth; hollow triangles and dashed edges denote CFP predictions. These examples test whether learned local refinement rules preserve recursive geometry outside the matched CFG family.}
    \label{fig:benchmark}
\end{figure}

\subsubsection{Implementation details for CIFAR hierarchies}
\label{A:cifar_details}

For the image benchmarks, the hierarchy is induced from representation space rather than generated by a known recursive law. We first extract frozen CLIP ViT-B/32 embeddings \cite{radford2021learningtransferablevisualmodels} from the CIFAR training images and then construct a finite ternary hierarchy recursively. In principle, this hierarchy-construction step is not tied to a single deep clustering algorithm: any method that produces a recursive partition together with cluster representatives could be used in the same pipeline. In this work, we instantiate this step with an existing deep clustering method~\cite{li2025lightsimplifyingimageclustering}, with loss weight $\alpha = 1$, which we apply recursively to obtain a ternary hierarchy.

Because the resulting hierarchy is data-induced, child slots are not automatically aligned across different parents. To reduce this ambiguity, we canonicalize the order of the children at each split using a PCA-based geometric rule before training the residual predictor. Concretely, let $u_1,u_2\in\mathbb{R}^d$ denote the first two principal directions obtained by fitting PCA to the centroids from all observed training levels. For a parent centroid $p$ and one of its children $c_k$, we form the displacement
\[
\delta_k = c_k - p,
\]
and project it onto the PCA plane,
\[
z_k =
\begin{pmatrix}
u_1^\top \delta_k \\
u_2^\top \delta_k
\end{pmatrix}.
\]
We then assign the child a canonical angular coordinate
\[
\theta_k = \operatorname{atan2}\!\big((z_k)_2,(z_k)_1\big),
\]
and order the children of the same parent by increasing $\theta_k$, using $\|\delta_k\|_2$ only as a tie-breaker when needed. This preprocessing step is used only to improve slot consistency across parents; it does not change the underlying hierarchy itself.

After hierarchy construction, we convert the tree into level-wise centroid
arrays and canonicalize the child order. 
For CIFAR, the predictor operates in the CLIP embedding space and is implemented as a bottleneck MLP with width $1024$, depth $7$, and GELU activations. We trained on levels 0 to 5 and report the performance of predicting the unseen level 6.
We use GELU only in this high-dimensional image-induced setting, where the deeper pure MLP was less stable with ReLU. 
This choice does not change the theory.
Our theoretical results are stated for ReLU because it admits explicit,
parameter efficient constructive approximation and interpolation guarantees,
while general finite sample interpolation results also apply to broad classes
of non-polynomial activations~\cite{petersen2024mathematical}. 
This is also consistent with approximation theoretic work showing that alternative
activation choices can reduce approximation error in sufficiently complex settings, although optimization and generalization remain important considerations~\cite{zhang2022deep}. 
We compare against the same linear, learnable constant, and average residual baselines used in the synthetic experiments.

\subsubsection{Ablation studies}
\label{A:ablation}
We next examine three factors that are especially relevant to recursive hierarchy completion.

\textbf{Effect of observed depth.}
Figure~\ref{fig:Ltrain} shows that extrapolation quality improves substantially as the observed training depth $L_{\mathrm{train}}$ increases. When only one refinement level is observed ($L_{\mathrm{train}}=1$), the  CFP behaves similarly to the average-residual baseline, which suggests that a single observed split does not provide enough information to identify a reusable input-dependent refinement rule. Once the model is trained on two or more levels, the gap becomes clear. For $L_{\mathrm{train}}=2$,  CFP achieves test MSE roughly one order of magnitude lower than the baseline, and for $L_{\mathrm{train}}=3$ the gap widens further. The error curves also remain nearly flat across deeper rollout levels, indicating stable recursive extrapolation rather than only better local fitting.

\begin{figure}[htbp]
    \centering
    \includegraphics[width=\linewidth]{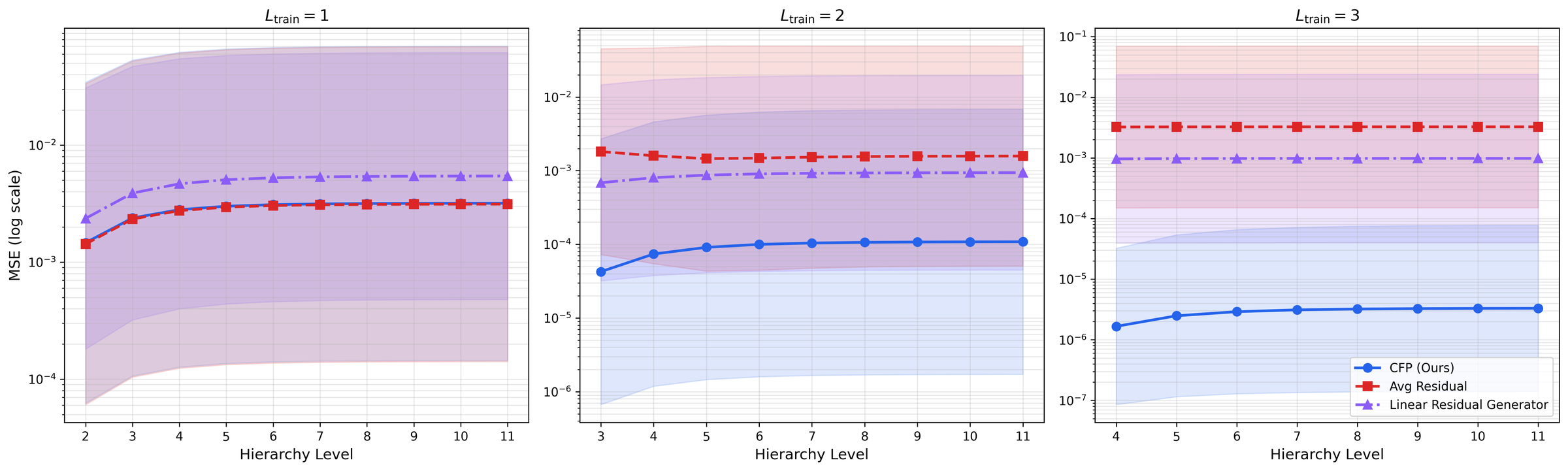}
    \caption{Effect of the observed training depth $L_{\mathrm{train}}$ on extrapolation MSE. As additional levels are observed during training, the CFP achieves substantially lower error on deeper unseen levels.}
    \label{fig:Ltrain}
\end{figure}

\textbf{Effect of scale factor.}
We next vary the scale factor $s\in\{0.3,\,0.5,\,0.7,\,0.85\}$ while fixing $d=2$, $K=3$, and $L_{\mathrm{train}}=2$. Because changing $s$ also changes the overall geometric scale of the hierarchy, we report a normalized MSE obtained by dividing the level-wise prediction error by
\[
\frac{1}{|\mathcal H_l|}\sum_{i=1}^{|\mathcal H_l|}\|x_i^{(l)}-x_0\|_2^2.
\]
This normalization keeps the results comparable across different scale regimes. Figure~\ref{fig:s_ablation} shows that extrapolation becomes harder as $s$ increases, since larger refinement steps make recursive error accumulation more pronounced for all methods. Even in this more difficult setting, the CFP remains consistently more accurate than the average-residual and affine baselines, especially at low and moderate values of $s$.

\begin{figure}[htbp]
    \centering
    \includegraphics[width=\linewidth]{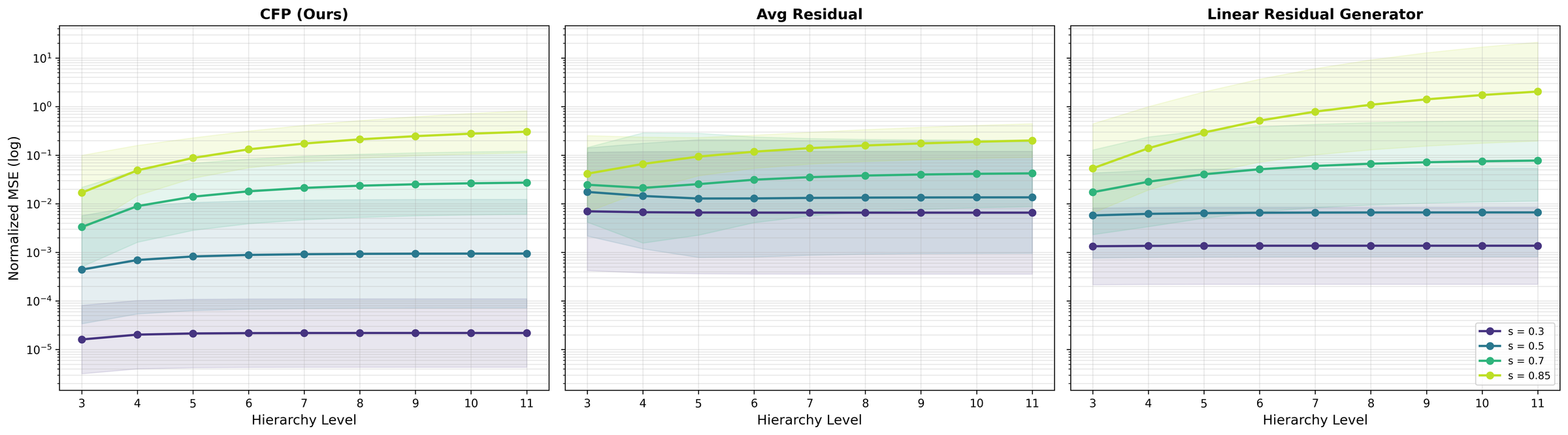}
    \caption{Effect of the scale factor $s$ on extrapolation performance in CFG-generated hierarchies. We fix $d=2$, $K=3$, and $L_{\mathrm{train}}=2$, and report normalized MSE across unseen levels for different values of $s$. Larger values of $s$ make recursive rollout substantially more difficult for all methods, while the CFP remains the most robust across the full range.}
    \label{fig:s_ablation}
\end{figure}

\textbf{Effect of child ordering.}
We also study how sensitive  CFP is to child-slot correspondence. For each trial, we construct three versions of the same ground-truth hierarchy: original ordered, which preserves the generator-induced child ordering; permuted, in which the $K$ children of each parent are independently and randomly permuted together with their descendant subtrees; and permuted $+$ canonicalized, in which the permuted children are reordered by a deterministic geometric rule based on the polar angle of $(\mathrm{child}-\mathrm{parent})$, with Euclidean distance used to break ties. Figure~\ref{fig:childO} shows that the original ordered setting gives the best performance, while random local permutations cause a clear drop. This indicates that consistent child correspondence matters for hierarchy extrapolation.

\begin{figure}[htbp]
    \centering
    \includegraphics[width=\linewidth]{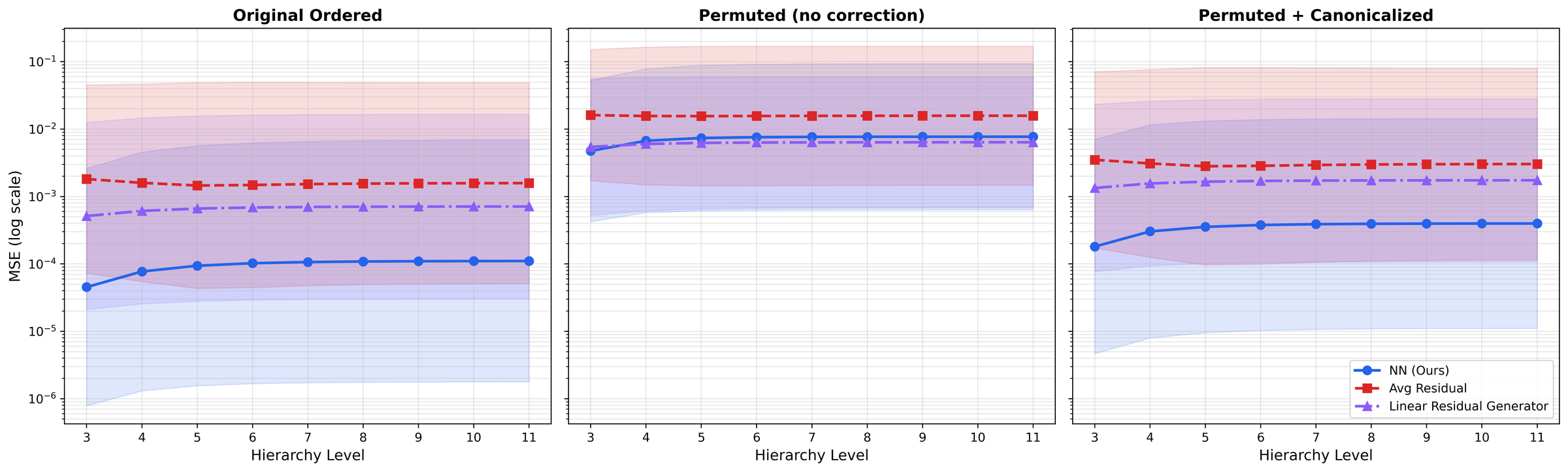}
    \caption{Sensitivity to child-slot correspondence. We compare the generator-induced ordering, independently permuted child slots, and permuted slots followed by geometric canonicalization. Random local permutations degrade rollout accuracy, while canonicalization substantially recovers performance, showing that consistent child slots are important for order-sensitive CFP training.}
    \label{fig:childO}
\end{figure}

\subsubsection{Limitations}
\label{A:limit}
Our framework assumes a fixed branching factor, sufficient separation, and a refinement process that is reusable across branches. Ordered child slots are also required for order-sensitive training. For data-induced hierarchies this requires canonicalization, and weakly aligned slots can degrade rollout. The IFS and CIFAR experiments are therefore best viewed as validation probes outside the exact theorem setting rather than evidence that every hierarchical clustering problem admits a classification field. Extending the framework to noisy, variable-branching, weakly aligned, or non-self-similar hierarchies remains an important direction.


\end{document}